\documentclass[twoside]{article}

\usepackage{amsmath,amsfonts,bm,amsthm}

\def\eqref#1{equation~\ref{#1}}

\def\1{\bm{1}}

\DeclareMathAlphabet{\mathsfit}{\encodingdefault}{\sfdefault}{m}{sl}
\SetMathAlphabet{\mathsfit}{bold}{\encodingdefault}{\sfdefault}{bx}{n}

\DeclareMathOperator*{\argmax}{arg\,max}

\newtheorem{theorem}{Theorem}
\newtheorem{lemma}[theorem]{Lemma}

\newtheorem{definition}[theorem]{Definition}
\newtheorem*{definition*}{Definition}

\usepackage[utf8]{inputenc} %
\usepackage[T1]{fontenc}    %
\usepackage{hyperref}       %
\usepackage{url}            %
\usepackage{booktabs}       %
\usepackage{amsfonts}       %
\usepackage{nicefrac}       %
\usepackage{microtype}      %
\usepackage{xcolor}         %

\usepackage[round]{natbib}
\usepackage{mathtools} %
\usepackage{tikz} %
\usepackage{graphicx}
\usepackage{amsmath}
\usepackage{amssymb}
\usepackage{algorithm}
\usepackage{algpseudocode}
\usepackage{xspace}
\usepackage{tcolorbox}

\newcommand{\division}{\mkern-\medmuskip\rotatebox[origin=c]{45}{\scalebox{0.9}{$-$}}\mkern-\medmuskip}
\newcommand{\hsf}{SHF\xspace}
\newcommand{\hsfs}{SHFs\xspace}
\newcommand{\mf}{MF\xspace}
\newcommand{\mfs}{MFs\xspace}
\usepackage[accepted]{aistats2026}

\begin{document}

\runningauthor{Chen, Dullerud, Niedermayr, Kidd, Senanayake, Koh, Koyejo, Guestrin}

\twocolumn[

\aistatstitle{Modeling Multi-Objective Tradeoffs with Monotonic Utility Functions}

\aistatsauthor{%
  Edward Chen\textsuperscript{1} \And Natalie Dullerud\textsuperscript{1} \And Thomas Niedermayr\textsuperscript{1} \And Elizabeth Kidd\textsuperscript{1}
  \AND
  Ransalu Senanayake\textsuperscript{3} \And Pang Wei Koh\textsuperscript{2} \And Sanmi Koyejo\textsuperscript{1} \And Carlos Guestrin\textsuperscript{1}%
}

\aistatsaddress{\textsuperscript{1}Stanford University, \textsuperscript{2}University of Washington, \textsuperscript{3}Arizona State University}
]

\begin{abstract}

  Countless science and engineering applications in multi-objective optimization (MOO) necessitate that decision-makers (DMs) select a Pareto-optimal (PO) solution which aligns with their preferences. Evaluating individual solutions is often expensive, and the high-dimensional trade-off space makes exhaustive exploration of the full Pareto frontier (PF) infeasible. We introduce a novel, principled two-step process for obtaining a compact set of PO points that aligns with user preferences, which are specified \textit{a priori} as general \textit{monotonic utility functions} (MFs). Our process (1) densely samples the user's region of interest on the PF, then (2) sparsifies the results into a small, diverse set for the DM. We instantiate this framework with \textit{soft-hard functions} (SHFs), an intuitive class of \mfs that operationalizes the common expert heuristic of imposing soft and hard bounds. We provide extensive empirical validation of our framework instantiated with SHFs on diverse domains, including brachytherapy, engineering design, and large language models. For brachytherapy, our approach returns a compact set of points with over 3\% greater SHF-defined utility than the next best approach. Among the other domains, our approach consistently leads in utility, as a final compact set of just 5 points captures over 99\% of the utility offered by the entire dense set.
\end{abstract}

\section{INTRODUCTION}\label{sec:intro}

Optimizing across multiple competing objectives — multi-objective optimization (MOO) — is a fundamental challenge in critical real-world applications, from engineering design and drug discovery to deep learning and healthcare~\citep{fromer_computer-aided_2023, luukkonen_artificial_2023, xie_mars_2021, yu_multi-objective_2000, papadimitriou_multiobjective_2001}. In these scenarios, objectives often conflict, and evaluating potential solutions can be computationally expensive or require expert knowledge. The ultimate goal is typically not to identify the entire set of optimal trade-offs (the Pareto frontier, PF), but rather to efficiently present a \textit{compact and relevant set} of high-quality solutions from which a decision-maker (DM) can select one that best aligns with their latent preferences for the ideal balance among objectives. The complexity of the trade-off space and the cognitive burden of sifting through numerous options make exhaustive PF exploration impractical. DMs often possess valuable prior knowledge to guide the search, yet effectively incorporating these priors to identify a manageable and useful subset of the PF remains a significant hurdle, especially when solution evaluations are costly.

This challenge is particularly acute in high-stakes domains, where practitioners frame objectives in terms of utility rather than simple maximization. A decision-maker's satisfaction with an outcome is often a monotonic and bounded function of the objective value \citep{huber_bayesian_2025, bassett_st_1987, sobrie_uta-poly_2018}. This functional class flexibly models diverse expert priors, from diminishing returns to sharp satisfaction thresholds. In healthcare, for instance, improving a clinical metric from a dangerous to an acceptable level yields high utility, while further improvements may offer only marginal gains \citep{mandal_robust_2019}. We argue that a powerful way to incorporate such expert priors is to have DMs specify their preferences \textit{a priori} via a monotonic utility function (\mf) for each objective. This approach directly translates domain knowledge into the optimization process. As another example, in brachytherapy, clinicians must balance tumor coverage against healthy organ exposure \citep{deufel_pnav_2020}, a task where they naturally reason with \textit{desired targets} ("soft bounds") and \textit{strict limits} ("hard bounds") \citep{viswanathan_american_2012}. This dual-level thinking is a utility structure explicitly modeled by a specific class of \mfs we term \textit{soft-hard functions} (\hsfs).

\textit{Current MOO techniques, while advanced, often do not directly incorporate this intuitive and flexible class of \mfs to guide the search for a compact, decision-relevant set of Pareto-optimal (PO) points}. While it is well-understood that exploring the entire PF can be intractable and unnecessary \citep{paria_flexible_2019, zuluaga_e-pal_2016}, and many methods aim to find representative PF subsets or elicit preferences through mechanisms like pairwise comparisons or utility function modeling \citep{paria_flexible_2019, abdolshah_multi-objective_2019, suzuki_multi-objective_2020, zuluaga_e-pal_2016, while_fast_2012}, there remains a gap in methods that allow DMs to \textit{directly articulate their preferences using these explicit monotonic utilities for each objective}. This direct articulation can constrain the search space effectively and align the presented solutions more closely with the DM's operational thinking, especially for domain experts who already possess such priors. Our work aims to fill this gap. We assume the DM possesses underlying (unknown) trade-off preferences, which are implicitly constrained by their specified \mfs. Given that evaluating each PO point is costly, we employ a two-step process: (1) efficient dense sampling of the PF region consistent with current bounds using Bayesian optimization, and (2) sparsification of these points using robust submodular optimization \textit{to present a small, diverse, and high-utility set to the DM}. We demonstrate that this sparsification step maintains theoretical guarantees on near-optimality with respect to the DM's unknown preferences encoded by the \mfs.

Although a rich body of work exists on incorporating preferences into MOO \citep{malkomes_beyond_2021, abdolshah_multi-objective_2019, zuluaga_e-pal_2016}, many methods do not allow for the \textit{a priori} specification of flexible, nonlinear utility functions. Existing preference elicitation schemes often rely on iterative feedback, such as pairwise comparisons \citep{paria_flexible_2019, suzuki_multi-objective_2020}, which may be sample-inefficient or not fully leverage an expert's pre-existing, well-defined knowledge. In contrast, our framework uniquely focuses on \textit{directly using the DM's specified utility functions as the primary mechanism for guiding exploration and obtaining a decision-relevant set of PO solutions}. Our contributions are:

\noindent
\begin{enumerate}
    \item A two-step process for finding a compact set of PO points that aligns with a DM's priors, specified via general monotonic utility functions (\mfs). The process (1) uses multi-objective Bayesian optimization to densely sample the relevant PF region, then (2) employs robust submodular optimization to sparsify these points into a small, diverse set. This theoretically-grounded approach maintains guarantees on both search coverage and the near-optimality of the final set.
    \item To demonstrate our framework's practical utility, we introduce \textbf{\textit{soft-hard functions}} (SHFs), a concrete class of \mfs that models the common expert heuristic of imposing soft and hard bounds. We provide \textbf{extensive empirical validation} of our framework with \hsfs on diverse domains, including brachytherapy, engineering design, and large language model personalization. 
\end{enumerate}

\begin{figure*}[h]
\begin{center}
\includegraphics[width=1.0\textwidth]{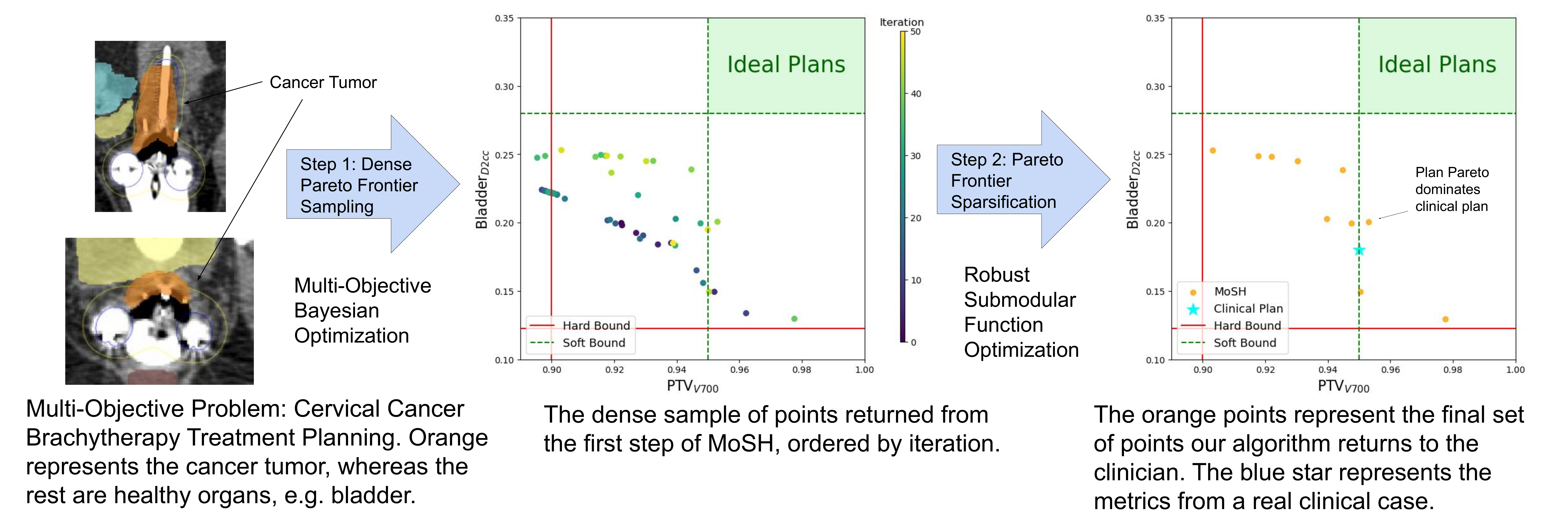}
\end{center}
\caption{
Our two-step pipeline, which we refer to as MoSH, demonstrated on a real cervical cancer brachytherapy clinical case. This figure instantiates our general framework using \textit{soft-hard functions} (\hsfs), a class of \mfs that captures the natural way clinicians reason with "soft" targets and "hard" limits (\S\ref{subsec:soft_hard_utility_functions}). The bounds shown reflect those used in real clinical settings. The MOO task involves balancing tumor radiation dosage against minimizing exposure to nearby healthy organs; two objectives (to maximize) are shown. Since the clinician's precise preferences are unknown, Step 1 densely samples the "Ideal Plans" region defined by the soft bounds. Step 2 then sparsifies this set into a small, diverse collection of high-utility options for the DM. Notably, our final set includes a plan that Pareto dominates the metrics from an actual clinical case, demonstrating the practical value of our approach.
}
\label{overall_pipeline}
\end{figure*}

\section{MULTI-OBJECTIVE OPTIMIZATION WITH MONOTONIC UTILITY FUNCTIONS}\label{background}
In this section, we outline the optimization setting we consider and introduce notation used throughout the paper in \S\ref{subsec:moo_background}. We then introduce our concept of modeling preferences via general monotonic utility functions in \S\ref{subsec:monotonic_utility_functions}. We thus arrive at the problem definition through our mini-max formulation in \S\ref{subsec:problem_definition_minimax}.

\subsection{Multi-Objective Optimization Background}
\label{subsec:moo_background}

As the name suggests, a multi-objective optimization (MOO) problem is an optimization problem that concerns multiple objective functions. Any MOO problem can be written as the joint maximization of $L$ objective functions over some input space $X\subseteq\mathbb{R}^d$,
\begin{equation*}
    \max_{x\in X}(f_1(x),\dots,f_L(x))
\end{equation*}
in which each $f_{\ell}$, $\ell\in[L]$, defines a function $f_{\ell}: X \rightarrow \mathbb{R}$.
Broadly speaking, there does not typically exist a \emph{feasible} solution that marginally optimizes each objective function simultaneously. Therefore, work in MOO generally focuses on Pareto-optimal (PO) solutions.
A feasible solution $x^{\dagger}$ is considered \emph{Pareto-optimal} if no objective can improve without degrading another; in other words, if $x^{\dagger}$ is not \emph{Pareto-dominated} by any other solution (defined in Appendix~\ref{app:definitions_theorems}).

A common approach to multi-objective optimization is to convert the $L$-dimensional objective to a scalar in order to utilize standard optimization methods via a scalarization function. Scalarization functions typically take the form $s_\lambda:\mathbb{R}^L\rightarrow \mathbb{R}$, parameterized by $\pmb{\lambda}$ from some set $\Lambda$ in $L$-dimensional space~\citep{roijers_survey_2013, paria_flexible_2019}. For instance, the general class of linear scalarization functions $s_{\Lambda}(\pmb{y}) := \{\pmb{\lambda}^{\texttt{T}}\pmb{y} \mid \pmb{\lambda}\in\Lambda\}$ constitutes all convex combinations of the objectives in $\Lambda$. The parameters $\pmb{\lambda} \in \Lambda$ can be viewed as weights, or relative preferences, on the  objective functions in the scalarized optimization objective $\max_{x\in X} s_{\pmb{\lambda}}\big([f_1(x), \dots, f_L(x)]\big)$. Then, the advantage of using scalarization functions is that the solution to maximizing $s_{\pmb{\lambda}}\big([f_1(x), \dots, f_{L}(x)]\big)$, for a fixed value of $\pmb{\lambda}$, may lead to a solution along the PF.

\subsection{Monotonic Utility Functions}\label{subsec:monotonic_utility_functions}

Many practical MOO problems, particularly in domains like engineering and healthcare, involve more nuanced desiderata than simply maximizing raw objective values \citep{cui_multi-criteria_2018, grosman_zone_2010, cairoli_model_2019}. For instance, improving an objective from a poor value to an acceptable one may yield high utility, while further improving it from a good value to an excellent one might offer only marginal gains \citep{deufel_pnav_2020}. Our framework aims to operationalize such nuanced priors by allowing a DM to specify their utilities for each objective \textit{a priori} through bounded, monotonic utility functions (\mfs).

An \mf, $u_{f_\ell}: \mathbb{R} \to \mathbb{R} \cup \{-\infty\}$, transforms the raw output of an objective function $f_\ell(x)$ into a numerical representation of its usefulness. We require two fundamental properties for our framework. First, the function must be \textbf{monotonic}. A utility function is monotonically increasing if for any two objective values $v_1$ and $v_2$, an improvement ($v_1 > v_2$) implies an equal or greater utility ($u_{f_\ell}(v_1) \geq u_{f_\ell}(v_2)$). This is a cornerstone of rational choice theory, ensuring that better objective performance is always preferred \citep{huber_bayesian_2025, sobrie_uta-poly_2018}. Second, the utility function must be \textbf{bounded} from above, which models the realistic notion that the utility of any single objective does not grow infinitely \citep{bassett_st_1987}. This boundedness prevents a single objective from completely dominating all others in the optimization process.

By enabling DMs to specify these bounded, monotonic functions directly, our framework empowers them to encode their domain-specific knowledge into the optimization process from the outset. As these are the only strict requirements, effective utility functions can also model other notable priors, such as strongly penalizing infeasible solutions and capturing diminishing returns through concavity. A concrete and intuitive instantiation of a function possessing these properties, the \textit{soft-hard function} (SHF), is detailed in \S\ref{subsec:soft_hard_utility_functions} and used in our empirical validations.

\subsection{Problem Definition}
\label{subsec:problem_definition_minimax}
Given a selected class of scalarization functions $s_{\pmb{\lambda}}$ parameterized by $\pmb{\lambda}\in\Lambda$ and a set of MFs as defined above, we wish to elucidate a set of useful points along the PF in the optimization problem
\begin{equation*}
    \max_{x\in X} s_{\pmb{\lambda}}(u_f(x))
\end{equation*}
where $u_f:= [u_{f_1}, \dots, u_{f_L}]$. Since we only want to select points from the PF, we follow \citet{roijers_survey_2013} and \citet{paria_flexible_2019} in assuming that $s_{\pmb{\lambda}}(u_f(x))$ is monotonically increasing in all coordinates $u_{f_{\ell}}(x)$, which leads to the solution to $\argmax_{x \in X} s_{\pmb{\lambda}}(u_f(x))$ residing on the PF, for some $\pmb{\lambda}$. We assume the DM has some hidden set of preferences, $\pmb{\lambda}^{*}$, for which the solution to the optimization problem $\max_{x\in X} s_{\pmb{\lambda}^{*}}(u_f(x))$ represents the ideal solution on the PF. Therefore, we wish to return a set of Pareto optimal points, $C$, which contains the unknown $c^* = \argmax_{x \in X} s_{\pmb{\lambda}^*}
(u_f(x))$, i.e. the ideal solution that aligns with the DM preferences. 
Since $\pmb{\lambda}^*$ is unknown to us, we want a set $C$ which is robust against any potential value of $\pmb{\lambda}^*$, or weight (preference) from the DM. Furthermore, $\lvert C\rvert \leq k$ for some small integer value $k$ to avoid overwhelming the DM with too many choices, which has been shown to decrease choice quality \citep{diehl_when_2005}. 

As a result, we formulate our general problem of selecting a set of points from the regions defined by the \mfs as:

\begin{equation}\label{submodular-formulation}
    \max_{C \subseteq X, |C| \leq k} \min_{\pmb{\lambda} \in \Lambda}\Bigg[\dfrac{\max_{x \in C} s_{\pmb{\lambda}}(u_f(x))}{\max_{x \in X} s_{\pmb{\lambda}} (u_f(x))}\Bigg]
\end{equation}

We refer to the right term, $\max_{x \in C} s_{\pmb{\lambda}}(u_f(x)) \division \max_{x \in X} s_{\pmb{\lambda}}(u_f(x))$, as the \mf utility ratio. Intuitively, the \mf utility ratio is maximized when the points in $C$ are Pareto optimal and span the high utility regions of the PF, as defined by the \mfs. We use an instantiation of the \mf utility ratio as an evaluation metric for some of our experiments in Section \ref{sec:experimental_results}. Given the unknown $\pmb{\lambda}^*$ and the expensive cost of evaluating points, we solve Equation~\ref{submodular-formulation} via a principled two-step process, allowing us to address the distinct challenges of continuous exploration and discrete set selection with methods that each carry strong theoretical guarantees: (1) We first densely sample the PF to obtain a high-quality discrete representation that provides robust coverage against uncertainty in $\pmb{\lambda}^*$. (2) Then, we sparsify this dense set to find a small, diverse, and near-optimal set of points for the decision-maker. This approach ensures our final solution is supported by guarantees on both the initial search and the final set, a property difficult to achieve with a single, monolithic method.

\section{STEP 1: DENSE PARETO FRONTIER SAMPLING WITH BAYESIAN OPTIMIZATION}\label{dense-sampling}
In this section, we consider the goal of obtaining a dense set of Pareto optimal points. As described earlier, since the DM's preferences, $\pmb{\lambda}^*$, are unknown to us, we wish to obtain a set $D$ which is diverse, high-coverage, and is modeled after the DM-defined \mfs. As is typical in various science and engineering applications, we assume access to some noisy and expensive black-box function -- often modeled with a Gaussian process (GP) \citep{williams_gaussian_1995} as the surrogate function. To achieve that goal, we extend our formulation (\ref{submodular-formulation}) into a Bayesian setting, assuming a prior $p(\pmb{\lambda})$ with support $\Lambda$ imposed on the set of Pareto optimal values and using the notion of random scalarizations \citep{paria_flexible_2019}. 
In this continuous and stochastic setting, we assume that each of the $\ell \in [L]$ objectives are sampled from known GP priors with a common domain, and produce noisy observations, e.g. $y_\ell = f_\ell(x) + \epsilon_\ell$, where $\epsilon \sim N(0, \sigma_\ell^2)$, $\forall \ell \in [L]$. 
We optimize over a set of scalarizations weighted by the prior $p(\pmb{\lambda})$.  

Our overall aim is still to obtain a set of points $C$ on the PF which are robust against the worst-case potential $\pmb{\lambda}^*$, within the user-defined \mf. For computational feasibility, however, we convert the worst-case into an average-case maximization (analyzed further in \S\ref{sec:baseline-methods}). Taking into consideration the aforementioned set of assumptions, we modify the formulation (\ref{submodular-formulation}) to be the following:

\begin{equation}\label{bayesian-formulation}
    \max_{D \subseteq X, |D| \leq k_D} \mathbb{E}_{\pmb{\lambda} \sim p(\pmb{\lambda})}\Bigg[\dfrac{\max_{x \in D} s_{\pmb{\lambda}}(u_f(x))}{\max_{x \in X} s_{\pmb{\lambda}}(u_f(x))}\Bigg]
\end{equation}

\textbf{Random Scalarizations.}
We now describe our sampling-based algorithm to optimize Equation \ref{bayesian-formulation}. Similar to \citet{paria_flexible_2019}, we use the notion of random scalarizations to sample a $\pmb{\lambda}$ from $p(\pmb{\lambda})$ at each iteration which is then used to compute a multi-objective acquisition function based on $s_{\pmb{\lambda}}$ and the \mf. The maximizer of the multi-objective acquisition function is then chosen as the next sample input value to be evaluated with the expensive black-box function. For the acquisition function, we use acq($u, \lambda_t, x$) $= s_{\lambda_t} (u_{\varphi(x)})$ where $\varphi(x) = \mu_t(x) + \sqrt{\beta_t} \sigma_t({x})$ and $\beta_t = \sqrt{0.125 \times \log(2\times t + 1)}$ (Appendix \ref{step1-additional-details}). For scalarization, we use the augmented Chebyshev scalarization, allowing for non-convex PF sampling \citep{wierzbicki_mathematical_1982}. We note, however, that \textit{our algorithm is agnostic to the scalarization and acquisition functions}. The full algorithm is described in Algorithm \ref{algorithm:mosh}. Finally, we provide formal guarantees proving the lower bound of the \mf utility ratio, which goes to one as $T \rightarrow \infty$, and thus providing the dense set $D$ (in Appendix \ref{app:definitions_theorems}).

\begin{algorithm}
\caption{Step 1: Dense Pareto Frontier Sampling}
\begin{algorithmic}[1]

    \State Initialize \mfs $\forall$ $\ell \in [L]$
    \State Initialize $D^{(0)} = \emptyset$, $GP_{\ell}^{(0)} = GP(0, \kappa)$ $\forall$ $\ell \in [L]$
    \For{t = 1 $\rightarrow$ T}
        \State Obtain $\lambda_t \sim p(\lambda)$
        \State $x_t = \argmax_{x \in X} \text{acq}(u, \lambda_t, x)$ 
        \State Obtain $y = f(x_t)$
        \State $D^{(t)} = D^{(t-1)} \cup \{(x_t, y)\}$
        \State $GP_\ell^{(t)} = \text{post}(GP_\ell^{t-1} | (x_t, y))$ $\forall$ $\ell \in [L]$
    \EndFor
    \State Return $D^{(T)}$

\end{algorithmic}
\label{algorithm:mosh}
\end{algorithm}

\section{STEP 2: PARETO FRONTIER SPARSIFICATION}\label{sec:step-2-method}
We now assume there already exists a dense set of points on the PF, sampled from step 1 with \mfs. As DM validation of such a dense set would be costly, the goal for step 2 is now to sparsify that set of points to then present to the DM a more navigable set which still maintains as much utility as the dense set. We do so by leveraging the notion of diminishing returns in utility for each additional point the DM validates. This notion is encapsulated by the property of submodularity, which further allows us to design optimization algorithms with strong theoretical guarantees. We use the definition of submodularity first developed in \citet{nemhauser_analysis_1978} (see Appendix~\ref{app:definitions_theorems} for definition).

As a result, we adapt Equation \ref{submodular-formulation} and formulate the sparsification problem as a robust submodular observation selection (RSOS) problem \citep{krause_robust_2008}: 

\begin{equation}\label{submodular-formulation2}
    \max_{C \subseteq D, |C| \leq k} \min_{\pmb{\lambda} \in \Lambda}\underbrace{\Bigg[\dfrac{\max_{x \in C} s_{\pmb{\lambda}}(u_f(x))}{\max_{x \in D} s_{\pmb{\lambda}}(u_f(x))}\Bigg]}_{F_{\pmb{\lambda}}}
\end{equation}

where $D$ represents the PO dense set obtained from step 1 (Section \ref{dense-sampling}, Algorithm \ref{algorithm:mosh}) and $C$ is the sparse set of \mf-defined PO points returned to the DM. Equation \ref{submodular-formulation2} is not submodular, but can be viewed as a maximin over submodular functions $F_{\pmb{\lambda}}$ (Lemma \ref{utility-ratio-submodular-lemma}), studied by \citet{krause_robust_2008}.

\begin{lemma}
\label{utility-ratio-submodular-lemma}
For some DM preference value $\pmb{\lambda}$, the set function, which takes as input some \mf-defined set $C$: $\max_{x \in C} s_{\pmb{\lambda}}(u_f(x)) \division \max_{x \in D} s_{\pmb{\lambda}}(u_f(x))$ is normalized ($F_{\pmb{\lambda}}(\emptyset)=0)$, monotonic (for all $A \subseteq C \subseteq D, F_{\pmb{\lambda}}(A) \leq F_{\pmb{\lambda}}(C)$), and submodular. (See proof in Appendix~\ref{app:definitions_theorems}).
\end{lemma}

\textbf{Algorithm.}
If Equation \ref{submodular-formulation2} were submodular, the simple greedy algorithm would provide a near optimal solution \citep{nemhauser_analysis_1978} (theorem in Appendix \ref{appendix-step-2-details}). Since it is not, the greedy algorithm performs arbitrarily worse than the optimal solution when solving Equation \ref{submodular-formulation2}, or more generally, problems formulated as RSOS \citep{krause_robust_2008} - often defined as such:
\begin{equation}
    \max_{C \subseteq X, |C| \leq k} \min_{i} F_i(C)
\end{equation}\label{rsos}
where the goal is to find a set $C$ of observations which is robust against the worst possible objective, $\min_i F_i$, from a set of submodular objectives. As a result, we solve Equation \ref{submodular-formulation2} using the Submodular Saturation algorithm, or SATURATE, which provides us with strong theoretical guarantees on the optimality of set $C$ \citep{krause_robust_2008}.

\begin{theorem}\label{krause-theorem}~\citep{krause_robust_2008} For any integer k, SATURATE finds a solution $C_S$ such that $\min_i F_i(C_S) \geq \max_{|C| \leq k} \min_i F_i(C)$ and $|C_S| \leq \psi k$, for $\psi=1+\log(\max_{x \in D} \sum_i F_i(\{x\}))$. The total number of submodular function evaluations is $\mathcal{O}(|D|^2 m \log(m \min_i F_i(D)))$, where $m = |\Lambda|$.
\end{theorem}

At a high level, SATURATE first defines a relaxed version of the original RSOS problem, which contains a superset of feasible solutions, and is guaranteed to find solutions to that relaxed version which are at least as informative as the optimal solution, only at a slightly higher cost. Within our context, this enables us to select a subset of points $C$, from the dense set obtained from Step 1, $D$, which achieves the optimal coverage of $\pmb{\lambda} \in \Lambda$ albeit with a slightly greater number of points. The complete algorithm is shown in Algorithm \ref{algorithm:saturate}.

\begin{algorithm}[H]
    \centering
    \caption{Step 2: Pareto Frontier Sparsification}
    \label{algorithm:saturate}
    \begin{algorithmic}[1]
    \Procedure{Sparsify}{$F_1,\dots,F_{|\Lambda|}, k, \psi$}
        \State $q_{min} \gets 0$; $q_{max} \gets \min_i F_i(D)$; $C_{best} \gets 0$
        
        \While{$(q_{max}-q_{min}) \geq 1 / |\Lambda|$}
            \State $q \gets (q_{min}+q_{max})/2$
            \State Define $\Bar{F}_q(C) = \frac{1}{|\Lambda|} \sum_i \min \{F_i(C), q\}$
            \State $\hat{C} \gets \text{GPC}(\Bar{F}_q, q)$ \Comment{Appendix: Algorithm \ref{algorithm:gpc}}
            
            \If{$|\hat{C}| > \psi k$}
                \State $q_{max} \gets q$
            \Else
                \State $q_{min} \gets q$; $C_{best} \gets \hat{C}$
            \EndIf %
        \EndWhile %
        
        \State \Return $C_{best}$
    \EndProcedure %
    \end{algorithmic}
\end{algorithm}

\section{EXPERIMENTAL RESULTS}
\label{sec:experimental_results}

To empirically validate our framework, we instantiate it with a class of \mfs designed for expert-driven, constrained domains. In this section, we introduce these functions and outline the details for our experiments.

\subsection{Concrete \mf: Soft-Hard Functions}
\label{subsec:soft_hard_utility_functions}
To instantiate our general framework for empirical validation, we introduce a concrete and intuitive class of \mfs called \textit{soft-hard functions} (SHFs). These functions are designed to operationalize the dual-level preferences common in expert-driven domains like brachytherapy and diabetes care, where practitioners reason with both strict requirements ("hard bounds") and desirable targets ("soft bounds") \citep{deufel_pnav_2020, grosman_zone_2010, cairoli_model_2019}. Traditional methods that only enforce hard constraints fail to capture this nuanced structure. SHFs explicitly model this prevalent DM heuristic by defining utility over these regions.

By construction, an SHF satisfies four key characteristics. First, objective values $f_{\ell}(x)$ below the hard bound $\alpha_{\ell,H}$ yield $-\infty$ utility, strongly penalizing infeasible solutions. Second, for $f_{\ell}(x) \geq \alpha_{\ell,S}$ (the soft bound), the utility function becomes concave, modeling diminishing returns once a satisfactory level is met. Third, beyond a saturation point $\alpha_{\ell,\tau}$ ($\geq \alpha_{\ell,S}$), utility becomes constant, preventing a single objective from dominating others. Finally, the function is monotonically increasing in $f_{\ell}(x)$, ensuring better objective values are preferred. 

We note that SHFs maintain generality in the landscape of optimization problems, as the function can be understood as an instantiation of fairly standard components of an optimization problem: feasibility constraints on objectives (hard constraints), and constraints with slack variables (soft constraints). While optimization problems frequently do not include constraints on \emph{every} objective function, our experiments and implementation consider the most constrained setting (soft and hard bounds on every objective). Nonetheless, flexibility is maintained in the general case for SHFs if the defined soft or hard bounds are effectively vacuous (i.e. no soft or hard constraint is enforced via the transformation) -- for instance, $\alpha_{\ell, H} = -\infty$,  $\alpha_{\ell, S} = \alpha_{\ell, \tau} = \infty$.

For simplicity and clear interpretability, we present a specific piecewise-linear functional form for \hsfs in Figure \ref{piecewise_util_func}. Additional exposition and the precise mathematical form are described in Appendix \ref{sec:appendix-shfs}. When our general two-step framework is instantiated with these \hsfs, we refer to the overall method as \textbf{MoSH}. The first step, dense PF sampling, is called \textbf{MoSH-Dense}, and the second step, PF sparsification, is called \textbf{MoSH-Sparse}. Accordingly, we refer to the specialized MF utility ratio as the \textbf{SHF utility ratio}.

\begin{figure}[h]
\begin{center}
\includegraphics[width=0.25\textwidth]{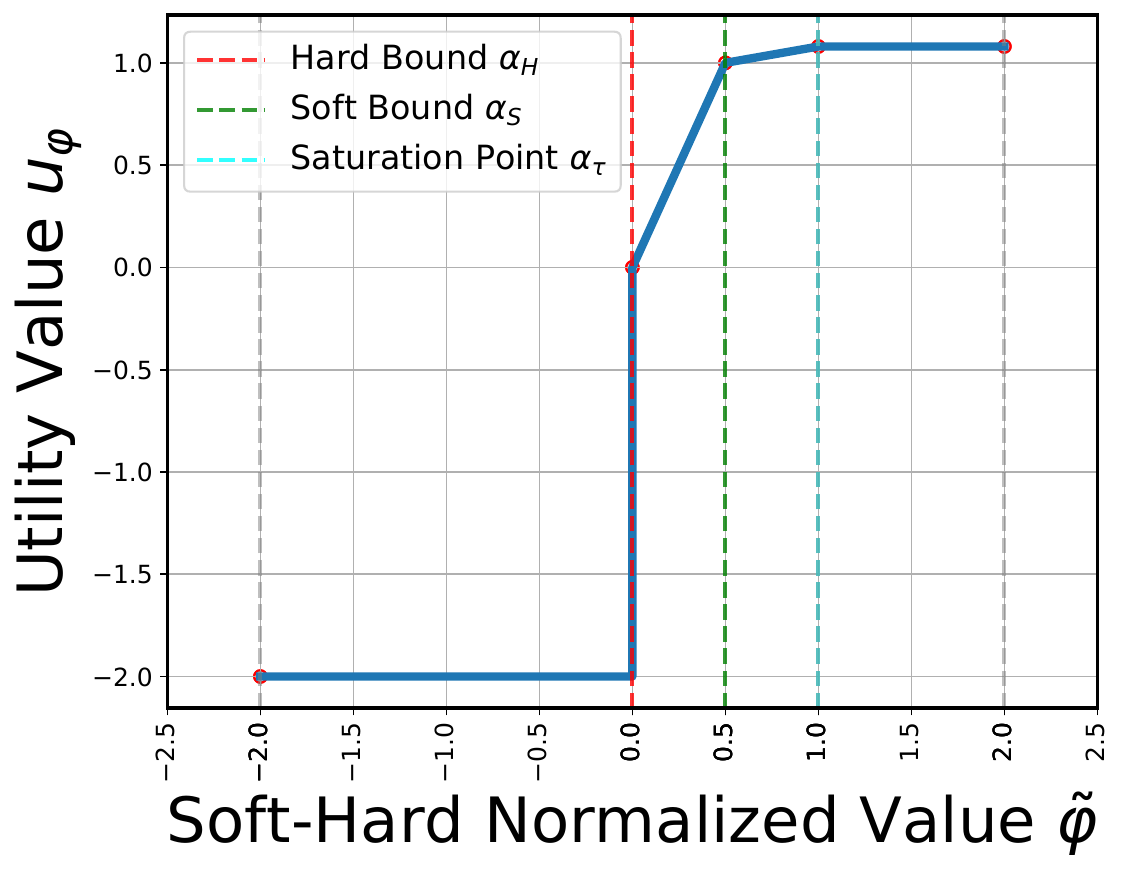}
\end{center}
\caption{
Example of a normalized soft-hard function (SHF) $u$. Dashed lines indicate hard, soft, and saturation point regions. Utility below the hard bound is $-\inf$ (shown as -2 for illustration).}
\label{piecewise_util_func}
\end{figure}

\subsection{Baseline Methodologies}\label{sec:baseline-methods}
\textbf{Step 1: Dense Pareto Frontier Sampling:}
We experiment with both synthetic and real-world applications and compare against other similar Bayesian multi-objective optimization approaches: Expected hypervolume improvement (EHVI) \citep{emmerich_computation_2008}, ParEGO \citep{knowles_parego_2006}, Multi-objective Bayesian optimization Using Random Scalarizations (MOBO-RS) \citep{paria_flexible_2019}, and random sampling. We compare against MOBO-RS with variations on the scalarization function, Chebyshev and linear, and acquisition function, UCB and Thompson sampling. Additional experimental results comparing against preference-based evolutionary multi-objective (PBEMO) algorithms may be found in Appendix \ref{additional-experimental-results}. \textit{We also use \hsfs to conduct empirical analyses of the worst-case to average-case conversion of Equation \ref{bayesian-formulation} in Appendix \ref{sec:computational-feasible-details}}.

\textbf{Step 2: Pareto Frontier Sparsification:}
We compare our method, MoSH-Sparse, against greedy and random algorithms. The greedy baseline starts with the empty set, and iteratively adds the element $c = \argmax_{x \in D \backslash C} H(C \cup \{x\})$, where $H = \min_{\pmb{\lambda} \in \Lambda} F_{\pmb{\lambda}}$ for the $F_{\pmb{\lambda}}$ in Equation \ref{submodular-formulation2}, until 
some stopping point. For all experiments, additional details and figures are provided in the Appendix.

\subsection{Performance Evaluation}\label{performance-criteria}

\subsubsection{Evaluation of Step 1: Dense Pareto Frontier Sampling}
As detailed in Section~\ref{dense-sampling}, our goal for the dense point set $D$ is to achieve diversity, high coverage, and alignment with the DM's \mfs. To conduct a concrete empirical analysis, we instantiate our general framework using SHFs from \S\ref{subsec:soft_hard_utility_functions}, allowing us to assess our method's ability to handle this practical and intuitive class of expert priors. We introduce four soft-hard metrics specifically designed to measure performance with respect to the key structural properties of SHFs: their distinct soft ($A_S$) and hard ($A_H$) regions. These metrics quantify how well an algorithm achieves diversity, high coverage, and adherence to these user-defined regions.

\textbf{Soft-Hard Fill Distance}: Measures sample diversity across $A_S$ and $A_H$. This adapts the fill distance by \citet{malkomes_beyond_2021}:
FILL($C$, $D$) = $\sup_{x' \in D} \min_{x \in C} \kappa(f(x), f(x'))$, where $C$ is the set of sampled points, $D$ is a comprehensive reference set of points in the target region, and $\kappa(\cdot)$ is a distance metric (typically Euclidean). Our metric is a $\upsilon$-weighted average of such fill distances calculated separately for soft and hard regions, encouraging exploration in both. Lower values indicate better, more uniform coverage.

\textbf{Soft-Hard Positive Samples Ratio}: 
Measures adherence to \hsf constraints via the $\upsilon$-weighted ratio of sampled points within soft ($C_S$) and hard ($C_H$) regions: $\upsilon (|C_S| \division |C|) + (1-\upsilon) (|C_H| \division |C|)$.

\textbf{Soft-Hard Hypervolume}: 
Assesses sample coverage by calculating the hypervolume occupied by points within the PF regions defined by the soft bounds (e.g., relative to $r_S = (\alpha_{1,S}, \dots, \alpha_{L,S})$) and, separately, the hard bounds. The final metric is a $\upsilon$-weighted average of these two hypervolumes.

\textbf{Soft Region Distance-Weighted Score}: We seek to explicitly measure faithfulness to the high-utility regions of \hsfs, the intersection of the soft bounds, by measuring the density of points. We calculate this using the following: $\sum_{x \in C} 1 \division \kappa(r_S, f(x))$, where $\kappa$ is a measure of distance, typically Euclidean distance and $r_S$ is defined above.

\subsubsection{Evaluation of Step 2: Pareto Frontier Sparsification}
To evaluate the set of points, $C$, returned to the DM, we simulate a $\pmb{\lambda}^*$ using the heuristic $\pmb{\lambda} = \textbf{u} \division \lVert \textbf{u} \rVert_1$, where $\textbf{u}_\ell \sim \text{N}(\alpha_{\ell,S}, |\alpha_{\ell,H}-\alpha_{\ell,S}| \division 3)$ \citep{paria_flexible_2019}. We then compute the \hsf utility ratio, right term in Equation \ref{submodular-formulation2}, for the set of points $C_t$ after each iteration of greedy, random, or MoSH-Dense. The denominator is calculated using $\pmb{\lambda}^*$ with a full set of points, $D$, computed offline.

\begin{figure*}[ht!]
\begin{center}
\includegraphics[width=0.85\textwidth]{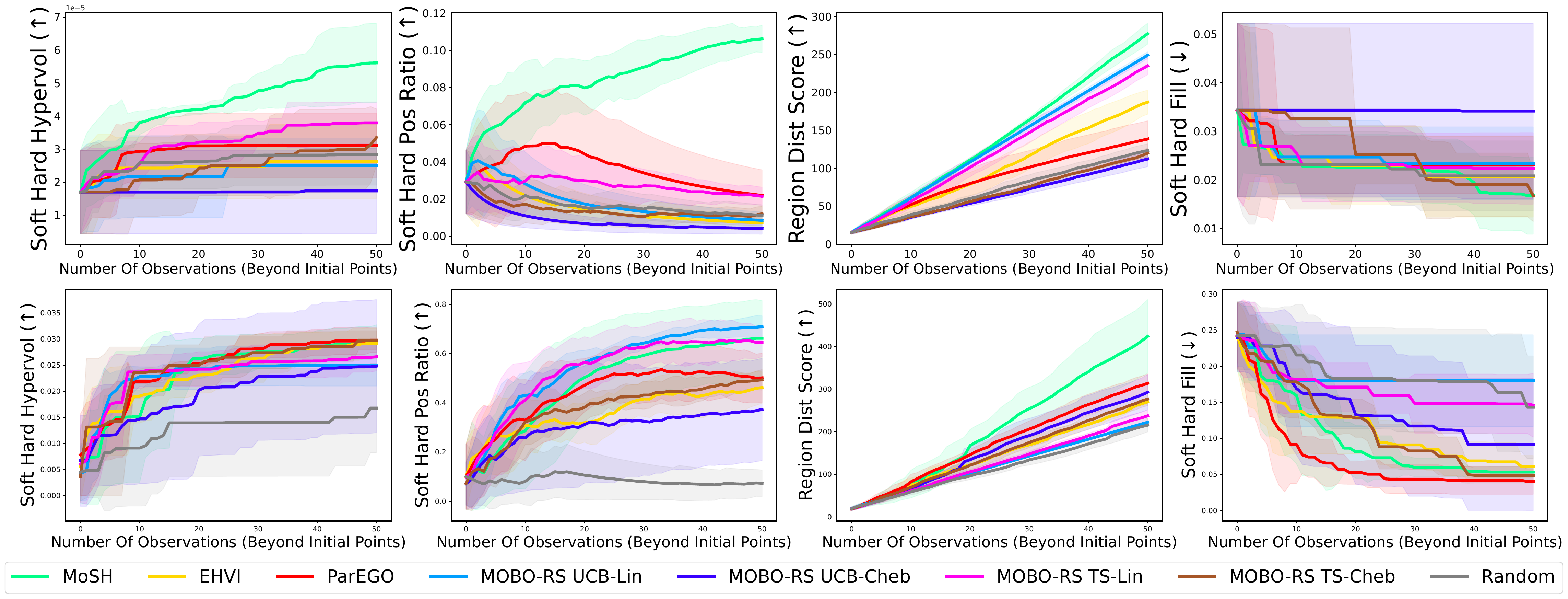}
\end{center}
\caption{Top row: Brachytherapy treatment planning. Bottom row: LLM personalization. Plots show the metrics defined in Section \ref{performance-criteria}, illustrating MoSH's consistent lead in providing a dense set of points which is (1) diverse, (2) high-coverage, and (3) modeled after the \hsfs. Mean $\pm$ std. were computed over 6 independent runs. 
}
\label{brachy_llm_soft_mega}
\end{figure*}

\subsection{Step 1 Experiments: Synthetic and Real-World Applications}
\label{subsec:step_one_exp}
We provide additional experimental details and results on a synthetic problem, Branin-Currin, an engineering design problem, Four Bar Truss, and deep learning model selection in Appendix \ref{additional-experimental-results}.

\subsubsection{Large Language Model (LLM) Personalization}

We aim to personalize an LLM to generate outputs that are both concise and informative — two competing objectives. We employ proxy tuning \citep{liu_tuning_2024, mitchell_emulator_2023, shi_decoding-time_2024}, steering a base model $M$ using expert ($M^+$) and anti-expert ($M^-$) models for each objective. Following \citet{liu_tuning_2024}, the proxy-tuned output distribution $p_{\tilde{M}}(X_t | x_{<t})$ is a function of the base model logits and weighted differences between expert/anti-expert logits for conciseness ($M^+_0, M^-_0$) and informativeness ($M^+_1, M^-_1$). The weights $\theta_0, \theta_1$ are our decision variables, controlling the trade-off. MoSH is used to find optimal $\theta_i$ values based on soft-hard bounds defined for conciseness and informativeness. As shown in Figure~\ref{brachy_llm_soft_mega}, MoSH performs consistently well across metrics, demonstrating its general applicability (more details in Appendix~\ref{appendix:llm-concise-info}).

\subsubsection{Real Clinical Case: Cervical Cancer Brachytherapy Treatment Planning}
We evaluate our method on treatment planning for a real cervical cancer brachytherapy clinical case. This problem consists of four objectives and three continuous decision variables. The objective are (1) maximize the radiation dosage level to the cancer tumor, and minimize the radiation dosage levels to the (2) bladder, (3) rectum, and (4) bowel. We converted objectives (2)-(4) into maximization objectives. The decision variables are used as inputs to a linear program formulated as an epsilon-constraint method \citep{deufel_pnav_2020}. Figure \ref{brachy_llm_soft_mega} illustrates the plots with the soft-hard performance metrics. We notice that our proposed method surpasses the baselines by a greater amount in this many-objective setting, notably the soft-hard hypervolume and soft-hard positive ratio metrics.

\subsection{Step 2 and End-to-End Experiments}
To evaluate step (2), we evaluate our baselines on the sparsification of the dense set of points from MoSH-Dense. Figure \ref{brachy_llm_step2_e2e_results} displays the \hsf utility ratio values for successive points that the DM views, across two applications. We observe that in all five experimental settings, MoSH-Sparse matches or exceeds the baselines in achieving the overall highest \hsf utility ratio. In the brachytherapy setting, MoSH-Sparse achieves the highest \hsf utility ratio at the fastest pace, showcasing the effectiveness of our approach in distilling the dense set of points into a useful smaller set.

We further holistically evaluate our entire two-step process by comparing against all baselines for step (1), using MoSH-Sparse (Figure \ref{brachy_llm_step2_e2e_results}). We show that our method consistently leads in providing DM utility in most experiments (more in Appendix \ref{sec:app-e2e-experiments}).

\begin{figure}[ht]
\begin{center}
\includegraphics[width=0.43\textwidth]{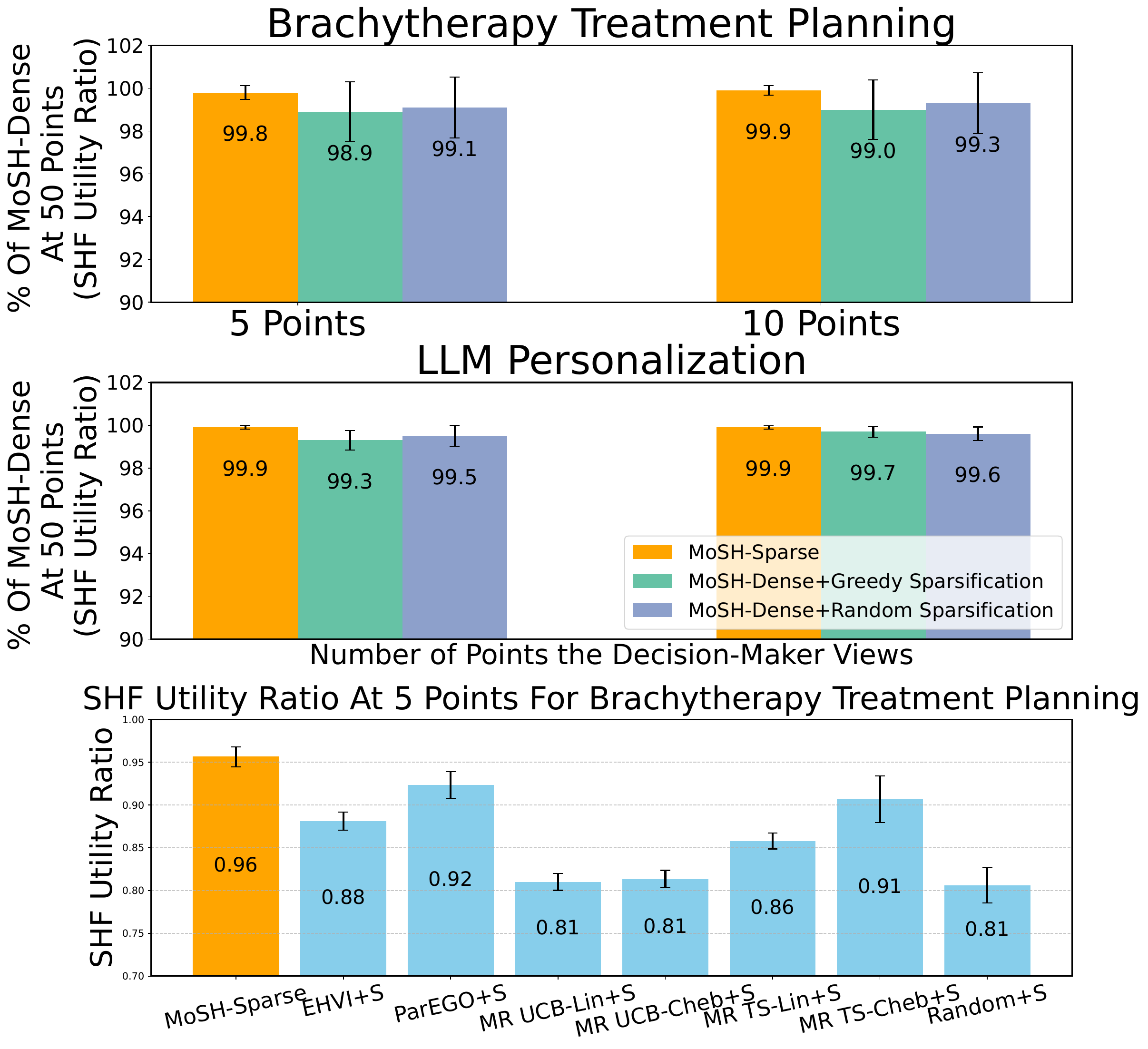}
\end{center}
\caption{Evaluation for step (2) of MoSH, illustrating the \hsf utility ratio (and std. dev) for the various methods relative to MoSH-Dense's at 50 points, compared when the DM views 5 and 10 points. (Bottom) End-to-end evaluation, using the \hsf utility ratio. Only the dense set, from step (1), changes for each bar. MR = MOBO-RS. S = SATURATE.
}
\label{brachy_llm_step2_e2e_results}
\end{figure}

\section{RELATED WORKS}
\textbf{Populating the Pareto Frontier.}
Most existing MOO works aim to approximate the entire PF, using heuristics such as the maximum hypervolume improvement \citep{campigotto_active_2014, ponweiser_multiobjective_2008, emmerich_computation_2008, picheny_multiobjective_2015, hernandez-lobato_predictive_2016, zhang_multiobjective_2009}. Others leverage RS of the objective values to attempt to recover the PF \citep{knowles_parego_2006,paria_flexible_2019, zhang_moead_2007, zhang_expensive_2010}. Additional works also place a greater emphasis on sparse, diversified PF coverage \citep{zuluaga_e-pal_2016, konakovic_lukovic_diversity-guided_2020} -- more recently, in hard-constrained regions \citep{malkomes_beyond_2021}. In contrast, we employ RS in a novel setting which aims to obtain a diverse, compact subset of the PF according to flexible user-defined priors (\mfs).

\textbf{MOO Feedback Mechanisms.}
Many feedback mechanisms, such as pairwise feedback, have been proposed -- but not all are designed for MOO \citep{zintgraf_ordered_2018, roijers_interactive_2018, astudillo_multi-attribute_2020, ip_user_2025}. \citet{trautmann_using_2017} enables the DM to specify per-objective numerical ranges for MOO. Some allow for DM to order objectives by importance \citep{abdolshah_multi-objective_2019}. \citet{ozaki_multi-objective_2023} introduced MOO improvement request feedback type. These works enable fine-grained MOO control, but our work uniquely accounts for multiple levels of preferences without needing to specify exact numerical values, which is often psychologically more difficult \citep{qian_learning_2015}.

\textbf{Utility Function Specification.}
Rather than embed the utilities directly into the optimization process, works such as \citet{huber_bayesian_2025} and \citet{tomczyk_decomposition-based_2020} propose to estimate monotonic utilities with pairwise preferences. Such assumptions of monotonic utilities are also often implicit, rather than explicit \citep{karasakal_interactive_2013, sobrie_uta-poly_2018, zintgraf_ordered_2018}, decreasing their flexibility. The formulation of objective utilities as inequality constraints, where the DMs aim to find inputs which satisfy thresholds on the objectives, is related to the topic of level set estimation \citep{gotovos_active_2013,zanette_robust_2018, malkomes_beyond_2021}. In contrast, our work is native to the MOO setting and may additionally incorporate a flexible set of constraints (e.g. \hsfs), which directly leverage the DM's priors. \textit{Several works related to MOBO may be found in Appendix \ref{app:additional-related-works}}.

\section{CONCLUSION}
We introduced a principled, two-step framework for MOO guided by general \textit{monotonic utility functions} (\mfs) in expensive black-box settings. Our approach returns a small, robust set of high-utility PO points by first performing dense PF sampling with Bayesian optimization, followed by a theoretically-grounded sparsification step using robust submodular optimization. As an instantiation of \mfs, we introduced \textit{soft-hard functions} (SHFs), an intuitive class of utility functions that models the common expert heuristic of imposing soft and hard bounds. We show that our framework has broad applicability to areas in brachytherapy, engineering design, and LLM personalization. For cervical cancer brachytherapy, our approach provides a compact set of treatment plans offering over 3\% greater utility than the next best method. Across all domains, our method consistently allows the DM to capture over 99\% of the utility from the dense set by validating just 5 points. We describe limitations in Appendix \ref{app:limitations}.

\bibliography{references, aistats_2026}

@inproceedings{balandat2020botorch,
  title = {{BoTorch: A Framework for Efficient Monte-Carlo Bayesian Optimization}},
  author = {Balandat, Maximilian and Karrer, Brian and Jiang, Daniel R. and Daulton, Samuel and Letham, Benjamin and Wilson, Andrew Gordon and Bakshy, Eytan},
  booktitle = {Advances in Neural Information Processing Systems 33},
  year = 2020,
  url = {http://arxiv.org/abs/1910.06403}
}

@article{van_der_meer_bi-objective_2020,
	title = {Bi-objective optimization of catheter positions for high-dose-rate prostate brachytherapy},
	volume = {47},
	issn = {2473-4209},
	url = {https://onlinelibrary.wiley.com/doi/abs/10.1002/mp.14505},
	doi = {10.1002/mp.14505},
	language = {en},
	number = {12},
	urldate = {2023-03-27},
	journal = {Medical Physics},
	author = {van der Meer, Marjolein C. and Bosman, Peter A.N. and Niatsetski, Yury and Alderliesten, Tanja and van Wieringen, Niek and Pieters, Bradley R. and Bel, Arjan},
	year = {2020},
	pages = {6077--6086},
}

@article{bassett_st_1987,
	title = {The {St}. {Petersburg} {Paradox} and {Bounded} {Utility}},
	volume = {19},
	issn = {0018-2702, 1527-1919},
	doi = {10.1215/00182702-19-4-517},
	language = {en},
	number = {4},
	urldate = {2025-10-02},
	journal = {History of Political Economy},
	author = {Bassett, Gilbert W.},
	month = nov,
	year = {1987},
	pages = {517--523},
}

@article{russo_learning_2014,
	title = {Learning to {Optimize} via {Posterior} {Sampling}},
	volume = {39},
	issn = {0364-765X},
	url = {https://pubsonline.informs.org/doi/abs/10.1287/moor.2014.0650},
	doi = {10.1287/moor.2014.0650},
	number = {4},
	urldate = {2026-02-07},
	journal = {Mathematics of Operations Research},
	author = {Russo, Daniel and Van Roy, Benjamin},
	month = nov,
	year = {2014},
	note = {Publisher: INFORMS},
	pages = {1221--1243},
}

@article{radulescu_multi-objective_2019,
	title = {Multi-objective multi-agent decision making: a utility-based analysis and survey},
	volume = {34},
	issn = {1573-7454},
	shorttitle = {Multi-objective multi-agent decision making},
	url = {https://doi.org/10.1007/s10458-019-09433-x},
	doi = {10.1007/s10458-019-09433-x},
	language = {en},
	number = {1},
	urldate = {2026-02-07},
	journal = {Autonomous Agents and Multi-Agent Systems},
	author = {Rădulescu, Roxana and Mannion, Patrick and Roijers, Diederik M. and Nowé, Ann},
	month = dec,
	year = {2019},
	pages = {10},
}

@article{fazio_advanced_2021,
	title = {Advanced {Resources} {Reservation} in {Mobile} {Cellular} {Networks}: {Static} vs. {Dynamic} {Approaches} under {Vehicular} {Mobility} {Model}},
	volume = {2},
	copyright = {http://creativecommons.org/licenses/by/3.0/},
	issn = {2673-4001},
	shorttitle = {Advanced {Resources} {Reservation} in {Mobile} {Cellular} {Networks}},
	url = {https://www.mdpi.com/2673-4001/2/4/20},
	doi = {10.3390/telecom2040020},
	language = {en},
	number = {4},
	urldate = {2026-02-07},
	journal = {Telecom},
	author = {Fazio, Peppino and Tropea, Mauro},
	month = sep,
	year = {2021},
	note = {Company: Multidisciplinary Digital Publishing Institute
Distributor: Multidisciplinary Digital Publishing Institute
Institution: Multidisciplinary Digital Publishing Institute
Label: Multidisciplinary Digital Publishing Institute
Publisher: publisher},
	pages = {302--327},
}

@inproceedings{samadi_optimal_2010,
	title = {Optimal {Real}-{Time} {Pricing} {Algorithm} {Based} on {Utility} {Maximization} for {Smart} {Grid}},
	url = {https://ieeexplore.ieee.org/abstract/document/5622077},
	doi = {10.1109/SMARTGRID.2010.5622077},
	urldate = {2026-02-07},
	booktitle = {2010 {First} {IEEE} {International} {Conference} on {Smart} {Grid} {Communications}},
	author = {Samadi, Pedram and Mohsenian-Rad, Amir-Hamed and Schober, Robert and Wong, Vincent W. S. and Jatskevich, Juri},
	month = oct,
	year = {2010},
	pages = {415--420},
}

@article{zhang_refined_2022,
	title = {A refined consumer behavior model for energy systems: {Application} to the pricing and energy-efficiency problems},
	volume = {308},
	issn = {0306-2619},
	shorttitle = {A refined consumer behavior model for energy systems},
	url = {https://www.sciencedirect.com/science/article/pii/S0306261921015026},
	doi = {10.1016/j.apenergy.2021.118239},
	urldate = {2026-02-07},
	journal = {Applied Energy},
	author = {Zhang, Chao and Lasaulce, Samson and Wang, Li and Saludjian, Lucas and Poor, H. Vincent},
	month = feb,
	year = {2022},
	pages = {118239},
}

@article{haghanifar_discovering_2020,
	title = {Discovering high-performance broadband and broad angle antireflection surfaces by machine learning},
	volume = {7},
	copyright = {© 2020 Optical Society of America},
	issn = {2334-2536},
	url = {https://opg.optica.org/optica/abstract.cfm?uri=optica-7-7-784},
	doi = {10.1364/OPTICA.387938},
	language = {EN},
	number = {7},
	urldate = {2026-02-07},
	journal = {Optica},
	author = {Haghanifar, Sajad and McCourt, Michael and Cheng, Bolong and Wuenschell, Jeffrey and Ohodnicki, Paul and Leu, Paul W.},
	month = jul,
	year = {2020},
	note = {Publisher: Optica Publishing Group},
	pages = {784--789},
}

@article{roijers_interactive_2018,
	title = {Interactive {Multi}-{Objective} {Reinforcement} {Learning} in {Multi}-{Armed} {Bandits} for {Any} {Utility} {Function}},
	language = {en},
	author = {Roijers, Diederik M and Zintgraf, Luisa M and Libin, Pieter and Nowé, Ann},
	year = {2018},
}

@misc{zintgraf_ordered_2018,
	title = {Ordered {Preference} {Elicitation} {Strategies} for {Supporting} {Multi}-{Objective} {Decision} {Making}},
	url = {http://arxiv.org/abs/1802.07606},
	doi = {10.48550/arXiv.1802.07606},
	urldate = {2026-02-07},
	publisher = {arXiv},
	author = {Zintgraf, Luisa M. and Roijers, Diederik M. and Linders, Sjoerd and Jonker, Catholijn M. and Nowé, Ann},
	month = feb,
	year = {2018},
	note = {arXiv:1802.07606 [cs]},
}

@article{moren_optimization_2021,
	title = {Optimization in treatment planning of high dose-rate brachytherapy — {Review} and analysis of mathematical models},
	volume = {48},
	issn = {2473-4209},
	url = {https://onlinelibrary.wiley.com/doi/abs/10.1002/mp.14762},
	doi = {10.1002/mp.14762},
	language = {en},
	number = {5},
	urldate = {2023-03-27},
	journal = {Medical Physics},
	author = {Morén, Björn and Larsson, Torbjörn and Tedgren, Åsa Carlsson},
	year = {2021},
	pages = {2057--2082},
}

@article{karasakal_interactive_2013,
	title = {An interactive partitioning approach for multiobjective decision making under a general monotone utility function},
	volume = {12},
	issn = {0219-6220},
	url = {https://www.worldscientific.com/doi/10.1142/S0219622013400051},
	doi = {10.1142/S0219622013400051},
	number = {05},
	urldate = {2025-10-02},
	journal = {International Journal of Information Technology \& Decision Making},
	author = {Karasakal, Esra and Köksalan, Murat},
	month = sep,
	year = {2013},
	note = {Publisher: World Scientific Publishing Co.},
	pages = {969--997},
}

@article{tomczyk_decomposition-based_2020,
	title = {Decomposition-{Based} {Interactive} {Evolutionary} {Algorithm} for {Multiple} {Objective} {Optimization}},
	volume = {24},
	issn = {1941-0026},
	url = {https://ieeexplore.ieee.org/document/8710313},
	doi = {10.1109/TEVC.2019.2915767},
	number = {2},
	urldate = {2025-10-02},
	journal = {IEEE Transactions on Evolutionary Computation},
	author = {Tomczyk, Michał K. and Kadziński, Miłosz},
	month = apr,
	year = {2020},
	pages = {320--334},
}

@article{sobrie_uta-poly_2018,
	title = {{UTA}-poly and {UTA}-splines: additive value functions with polynomial marginals},
	volume = {264},
	issn = {03772217},
	shorttitle = {{UTA}-poly and {UTA}-splines},
	url = {http://arxiv.org/abs/1603.02626},
	doi = {10.1016/j.ejor.2017.03.021},
	number = {2},
	urldate = {2025-10-02},
	journal = {European Journal of Operational Research},
	author = {Sobrie, Olivier and Gillis, Nicolas and Mousseau, Vincent and Pirlot, Marc},
	month = jan,
	year = {2018},
	note = {arXiv:1603.02626 [math]},
	pages = {405--418},
}

@misc{huber_bayesian_2025,
	title = {Bayesian preference elicitation for decision support in multiobjective optimization},
	url = {http://arxiv.org/abs/2507.16999},
	doi = {10.48550/arXiv.2507.16999},
	urldate = {2025-10-02},
	publisher = {arXiv},
	author = {Huber, Felix and Gonzalez, Sebastian Rojas and Astudillo, Raul},
	month = jul,
	year = {2025},
	note = {arXiv:2507.16999 [stat]},
}

@article{shaier_multi-objective_2025,
	title = {Multi-objective optimization and algorithmic evaluation for {EMS} in a {HRES} integrating {PV}, wind, and backup storage},
	volume = {15},
	copyright = {2024 The Author(s)},
	issn = {2045-2322},
	url = {https://www.nature.com/articles/s41598-024-84227-0},
	doi = {10.1038/s41598-024-84227-0},
	language = {en},
	number = {1},
	urldate = {2025-10-01},
	journal = {Scientific Reports},
	author = {Shaier, Ahmed A. and Elymany, Mahmoud M. and Enany, Mohamed A. and Elsonbaty, Nadia A.},
	month = jan,
	year = {2025},
	note = {Publisher: Nature Publishing Group},
	pages = {1147},
}

@article{muteba_mwamba_multi-objective_2025,
	title = {Multi-{Objective} {Portfolio} {Optimization}: {An} {Application} of the {Non}-{Dominated} {Sorting} {Genetic} {Algorithm} {III}},
	volume = {13},
	copyright = {http://creativecommons.org/licenses/by/3.0/},
	issn = {2227-7072},
	shorttitle = {Multi-{Objective} {Portfolio} {Optimization}},
	url = {https://www.mdpi.com/2227-7072/13/1/15},
	doi = {10.3390/ijfs13010015},
	language = {en},
	number = {1},
	urldate = {2025-10-01},
	journal = {International Journal of Financial Studies},
	author = {Muteba Mwamba, John Weirstrass and Mbucici, Leon Mishindo and Mba, Jules Clement},
	month = mar,
	year = {2025},
	note = {Publisher: Multidisciplinary Digital Publishing Institute},
	pages = {15},
}

@article{el_moutaouakil_multi-objective_2023,
	title = {Multi-{Objective} {Optimization} for {Controlling} the {Dynamics} of the {Diabetic} {Population}},
	volume = {11},
	copyright = {http://creativecommons.org/licenses/by/3.0/},
	issn = {2227-7390},
	url = {https://www.mdpi.com/2227-7390/11/13/2957},
	doi = {10.3390/math11132957},
	language = {en},
	number = {13},
	urldate = {2025-10-01},
	journal = {Mathematics},
	author = {El Moutaouakil, Karim and El Ouissari, Abdellatif and Palade, Vasile and Charroud, Anas and Olaru, Adrian and Baïzri, Hicham and Chellak, Saliha and Cheggour, Mouna},
	month = jan,
	year = {2023},
	note = {Publisher: Multidisciplinary Digital Publishing Institute},
	pages = {2957},
}

@article{soares_multi-objective_2024,
	title = {{MULTI}-{OBJECTIVE} {OPTIMIZATION} {APPLIED} {TO} {DIET} {PLANNING} {FOR} {PEOPLE} {WITH} {DIABETES}},
	volume = {44},
	issn = {0101-7438, 1678-5142},
	url = {https://www.scielo.br/j/pope/a/PSbGK4NMqmp96s6zdYrm4xF/?format=html&lang=en},
	doi = {https://doi.org/10.1590/0101-7438.2023.043.00285156},
	language = {en},
	urldate = {2025-10-01},
	journal = {Pesquisa Operacional},
	author = {Soares, Thiago F. and Escarpinati, Maurício C. and Gabriel, Paulo H. R.},
	year = {2024},
	note = {Publisher: Sociedade Brasileira de Pesquisa Operacional},
	pages = {e285156},
}

@article{mandal_robust_2019,
	title = {Robust multi‐objective blood glucose control in {Type}‐1 diabetic patient},
	volume = {13},
	issn = {1751-8849},
	url = {https://pmc.ncbi.nlm.nih.gov/articles/PMC8687400/},
	doi = {10.1049/iet-syb.2018.5093},
	number = {3},
	urldate = {2025-10-01},
	journal = {IET Systems Biology},
	author = {Mandal, Sharmistha and Sutradhar, Ashoke},
	month = jun,
	year = {2019},
	pmid = {31170693},
	pmcid = {PMC8687400},
	pages = {136--146},
}

@article{ji_multi-objective_2024,
	title = {Multi-objective {Bayesian} active learning for {MeV}-ultrafast electron diffraction},
	volume = {15},
	copyright = {2024 This is a U.S. Government work and not under copyright protection in the US; foreign copyright protection may apply},
	issn = {2041-1723},
	url = {https://www.nature.com/articles/s41467-024-48923-9},
	doi = {10.1038/s41467-024-48923-9},
	language = {en},
	number = {1},
	urldate = {2025-10-01},
	journal = {Nature Communications},
	author = {Ji, Fuhao and Edelen, Auralee and Roussel, Ryan and Shen, Xiaozhe and Miskovich, Sara and Weathersby, Stephen and Luo, Duan and Mo, Mianzhen and Kramer, Patrick and Mayes, Christopher and Othman, Mohamed A. K. and Nanni, Emilio and Wang, Xijie and Reid, Alexander and Minitti, Michael and England, Robert Joel},
	month = jun,
	year = {2024},
	note = {Publisher: Nature Publishing Group},
	pages = {4726},
}

@article{blank_pymoo_2020,
	title = {Pymoo: {Multi}-{Objective} {Optimization} in {Python}},
	volume = {8},
	issn = {2169-3536},
	shorttitle = {Pymoo},
	url = {https://ieeexplore.ieee.org/document/9078759/},
	doi = {10.1109/ACCESS.2020.2990567},
	urldate = {2025-10-01},
	journal = {IEEE Access},
	author = {Blank, Julian and Deb, Kalyanmoy},
	year = {2020},
	pages = {89497--89509},
}

@inproceedings{vesikar_reference_2018,
	address = {Bangalore, India},
	title = {Reference {Point} {Based} {NSGA}-{III} for {Preferred} {Solutions}},
	isbn = {978-1-5386-9276-9},
	url = {https://ieeexplore.ieee.org/document/8628819/},
	doi = {10.1109/SSCI.2018.8628819},
	language = {en},
	urldate = {2025-10-01},
	booktitle = {2018 {IEEE} {Symposium} {Series} on {Computational} {Intelligence} ({SSCI})},
	publisher = {IEEE},
	author = {Vesikar, Yash and Deb, Kalyanmoy and Blank, Julian},
	month = nov,
	year = {2018},
	pages = {1587--1594},
}

@misc{daulton_multi-objective_2022,
	title = {Multi-{Objective} {Bayesian} {Optimization} over {High}-{Dimensional} {Search} {Spaces}},
	url = {http://arxiv.org/abs/2109.10964},
	doi = {10.48550/arXiv.2109.10964},
	urldate = {2025-09-30},
	publisher = {arXiv},
	author = {Daulton, Samuel and Eriksson, David and Balandat, Maximilian and Bakshy, Eytan},
	month = jun,
	year = {2022},
	note = {arXiv:2109.10964 [cs]},
}

@inproceedings{auer_pareto_2016,
	title = {Pareto {Front} {Identification} from {Stochastic} {Bandit} {Feedback}},
	url = {https://proceedings.mlr.press/v51/auer16.html},
	language = {en},
	urldate = {2025-09-30},
	booktitle = {Proceedings of the 19th {International} {Conference} on {Artificial} {Intelligence} and {Statistics}},
	publisher = {PMLR},
	author = {Auer, Peter and Chiang, Chao-Kai and Ortner, Ronald and Drugan, Madalina},
	month = may,
	year = {2016},
	note = {ISSN: 1938-7228},
	pages = {939--947},
}

@misc{chen_ferero_2024,
	title = {{FERERO}: {A} {Flexible} {Framework} for {Preference}-{Guided} {Multi}-{Objective} {Learning}},
	shorttitle = {{FERERO}},
	url = {http://arxiv.org/abs/2412.01773},
	doi = {10.48550/arXiv.2412.01773},
	urldate = {2025-09-30},
	publisher = {arXiv},
	author = {Chen, Lisha and Saif, A. F. M. and Shen, Yanning and Chen, Tianyi},
	month = dec,
	year = {2024},
	note = {arXiv:2412.01773 [cs]},
}

@inproceedings{deb_reference_2006,
	address = {New York, NY, USA},
	series = {{GECCO} '06},
	title = {Reference point based multi-objective optimization using evolutionary algorithms},
	isbn = {978-1-59593-186-3},
	url = {https://dl.acm.org/doi/10.1145/1143997.1144112},
	doi = {10.1145/1143997.1144112},
	urldate = {2025-07-31},
	booktitle = {Proceedings of the 8th annual conference on {Genetic} and evolutionary computation},
	publisher = {Association for Computing Machinery},
	author = {Deb, Kalyanmoy and Sundar, J.},
	month = jul,
	year = {2006},
	pages = {635--642},
}

@inproceedings{malkomes_beyond_2021,
	title = {Beyond the {Pareto} {Efficient} {Frontier}: {Constraint} {Active} {Search} for {Multiobjective} {Experimental} {Design}},
	shorttitle = {Beyond the {Pareto} {Efficient} {Frontier}},
	url = {https://proceedings.mlr.press/v139/malkomes21a.html},
	language = {en},
	urldate = {2025-05-15},
	booktitle = {Proceedings of the 38th {International} {Conference} on {Machine} {Learning}},
	publisher = {PMLR},
	author = {Malkomes, Gustavo and Cheng, Bolong and Lee, Eric H. and Mccourt, Mike},
	month = jul,
	year = {2021},
	note = {ISSN: 2640-3498},
	pages = {7423--7434},
}

@article{deng_mnist_2012,
	title = {The {MNIST} {Database} of {Handwritten} {Digit} {Images} for {Machine} {Learning} {Research} [{Best} of the {Web}]},
	volume = {29},
	issn = {1558-0792},
	url = {https://ieeexplore.ieee.org/document/6296535},
	doi = {10.1109/MSP.2012.2211477},
	number = {6},
	urldate = {2025-05-14},
	journal = {IEEE Signal Processing Magazine},
	author = {Deng, Li},
	month = nov,
	year = {2012},
	pages = {141--142},
}

@inproceedings{cairoli_model_2019,
	title = {Model {Predictive} {Control} of glucose concentration based on {Signal} {Temporal} {Logic} specifications},
	url = {https://ieeexplore.ieee.org/abstract/document/8820492},
	doi = {10.1109/CoDIT.2019.8820492},
	urldate = {2025-05-13},
	booktitle = {2019 6th {International} {Conference} on {Control}, {Decision} and {Information} {Technologies} ({CoDIT})},
	author = {Cairoli, Francesca and Fenu, Gianfranco and Pellegrino, Felice Andrea and Salvato, Erica},
	month = apr,
	year = {2019},
	note = {ISSN: 2576-3555},
	pages = {714--719},
}

@article{grosman_zone_2010,
	title = {Zone {Model} {Predictive} {Control}: {A} {Strategy} to {Minimize} {Hyper}- and {Hypoglycemic} {Events}},
	volume = {4},
	issn = {1932-2968},
	shorttitle = {Zone {Model} {Predictive} {Control}},
	url = {https://www.ncbi.nlm.nih.gov/pmc/articles/PMC2909531/},
	doi = {10.1177/193229681000400428},
	number = {4},
	urldate = {2025-05-13},
	journal = {Journal of Diabetes Science and Technology},
	author = {Grosman, Benyamin and Dassau, Eyal and Zisser, Howard C. and Jovanoviĉ, Lois and Doyle, Francis J.},
	month = jul,
	year = {2010},
	pmid = {20663463},
	pmcid = {PMC2909531},
	pages = {961--975},
}

@misc{ip_user_2025,
	title = {User {Preference} {Meets} {Pareto}-{Optimality} in {Multi}-{Objective} {Bayesian} {Optimization}},
	url = {http://arxiv.org/abs/2502.06971},
	doi = {10.48550/arXiv.2502.06971},
	urldate = {2025-05-12},
	publisher = {arXiv},
	author = {Ip, Joshua Hang Sai and Chakrabarty, Ankush and Mesbah, Ali and Romeres, Diego},
	month = mar,
	year = {2025},
	note = {arXiv:2502.06971 [cs]},
}

@inproceedings{konakovic_lukovic_diversity-guided_2020,
	title = {Diversity-{Guided} {Multi}-{Objective} {Bayesian} {Optimization} {With} {Batch} {Evaluations}},
	volume = {33},
	url = {https://papers.nips.cc/paper/2020/hash/cd3109c63bf4323e6b987a5923becb96-Abstract.html},
	urldate = {2025-05-12},
	booktitle = {Advances in {Neural} {Information} {Processing} {Systems}},
	publisher = {Curran Associates, Inc.},
	author = {Konakovic Lukovic, Mina and Tian, Yunsheng and Matusik, Wojciech},
	year = {2020},
	pages = {17708--17720},
}

@misc{ozaki_multi-objective_2023,
	title = {Multi-{Objective} {Bayesian} {Optimization} with {Active} {Preference} {Learning}},
	url = {http://arxiv.org/abs/2311.13460},
	doi = {10.48550/arXiv.2311.13460},
	urldate = {2025-01-24},
	publisher = {arXiv},
	author = {Ozaki, Ryota and Ishikawa, Kazuki and Kanzaki, Youhei and Suzuki, Shinya and Takeno, Shion and Takeuchi, Ichiro and Karasuyama, Masayuki},
	month = nov,
	year = {2023},
	note = {arXiv:2311.13460 [cs]},
}

@article{danskin_theory_1966,
	title = {The {Theory} of  {Max}-{Min}, with {Applications}},
	volume = {14},
	issn = {0036-1399},
	url = {https://epubs.siam.org/doi/abs/10.1137/0114053},
	doi = {10.1137/0114053},
	number = {4},
	urldate = {2024-11-28},
	journal = {SIAM Journal on Applied Mathematics},
	author = {Danskin, John M.},
	month = jul,
	year = {1966},
	note = {Publisher: Society for Industrial and Applied Mathematics},
	pages = {641--664},
}

@misc{canas_learning_2012,
	title = {Learning {Probability} {Measures} with respect to {Optimal} {Transport} {Metrics}},
	url = {http://arxiv.org/abs/1209.1077},
	doi = {10.48550/arXiv.1209.1077},
	urldate = {2024-11-20},
	publisher = {arXiv},
	author = {Canas, Guillermo D. and Rosasco, Lorenzo},
	month = sep,
	year = {2012},
	note = {arXiv:1209.1077},
}

@article{zhang_multiobjective_2009,
	title = {Multiobjective optimization {Test} {Instances} for the {CEC} 2009 {Special} {Session} and {Competition}},
	language = {en},
	author = {Zhang, Qingfu and Zhou, Aimin and Zhao, Shizheng and Suganthan, Ponnuthurai Nagaratnam and Liu, Wudong and Tiwari, Santosh},
	year = {2009},
}

@inproceedings{papadimitriou_multiobjective_2001,
	address = {New York, NY, USA},
	series = {{PODS} '01},
	title = {Multiobjective query optimization},
	isbn = {978-1-58113-361-5},
	url = {https://dl.acm.org/doi/10.1145/375551.375560},
	doi = {10.1145/375551.375560},
	urldate = {2024-11-18},
	booktitle = {Proceedings of the twentieth {ACM} {SIGMOD}-{SIGACT}-{SIGART} symposium on {Principles} of database systems},
	publisher = {Association for Computing Machinery},
	author = {Papadimitriou, Christos H. and Yannakakis, Mihalis},
	month = may,
	year = {2001},
	pages = {52--59},
}

@misc{rafailov_direct_2024,
	title = {Direct {Preference} {Optimization}: {Your} {Language} {Model} is {Secretly} a {Reward} {Model}},
	shorttitle = {Direct {Preference} {Optimization}},
	url = {http://arxiv.org/abs/2305.18290},
	doi = {10.48550/arXiv.2305.18290},
	urldate = {2024-10-19},
	publisher = {arXiv},
	author = {Rafailov, Rafael and Sharma, Archit and Mitchell, Eric and Ermon, Stefano and Manning, Christopher D. and Finn, Chelsea},
	month = jul,
	year = {2024},
	note = {arXiv:2305.18290 [cs]},
}

@misc{ivison_camels_2023,
	title = {Camels in a {Changing} {Climate}: {Enhancing} {LM} {Adaptation} with {Tulu} 2},
	shorttitle = {Camels in a {Changing} {Climate}},
	url = {http://arxiv.org/abs/2311.10702},
	doi = {10.48550/arXiv.2311.10702},
	urldate = {2024-10-19},
	publisher = {arXiv},
	author = {Ivison, Hamish and Wang, Yizhong and Pyatkin, Valentina and Lambert, Nathan and Peters, Matthew and Dasigi, Pradeep and Jang, Joel and Wadden, David and Smith, Noah A. and Beltagy, Iz and Hajishirzi, Hannaneh},
	month = nov,
	year = {2023},
	note = {arXiv:2311.10702 [cs]},
}

@misc{shi_decoding-time_2024,
	title = {Decoding-{Time} {Language} {Model} {Alignment} with {Multiple} {Objectives}},
	url = {http://arxiv.org/abs/2406.18853},
	doi = {10.48550/arXiv.2406.18853},
	urldate = {2024-10-19},
	publisher = {arXiv},
	author = {Shi, Ruizhe and Chen, Yifang and Hu, Yushi and Liu, Alisa and Hajishirzi, Hannaneh and Smith, Noah A. and Du, Simon},
	month = jun,
	year = {2024},
	note = {arXiv:2406.18853 [cs]},
}

@misc{mitchell_emulator_2023,
	title = {An {Emulator} for {Fine}-{Tuning} {Large} {Language} {Models} using {Small} {Language} {Models}},
	url = {http://arxiv.org/abs/2310.12962},
	doi = {10.48550/arXiv.2310.12962},
	urldate = {2024-10-19},
	publisher = {arXiv},
	author = {Mitchell, Eric and Rafailov, Rafael and Sharma, Archit and Finn, Chelsea and Manning, Christopher D.},
	month = oct,
	year = {2023},
	note = {arXiv:2310.12962 [cs]},
}

@misc{liu_tuning_2024,
	title = {Tuning {Language} {Models} by {Proxy}},
	url = {http://arxiv.org/abs/2401.08565},
	doi = {10.48550/arXiv.2401.08565},
	urldate = {2024-10-19},
	publisher = {arXiv},
	author = {Liu, Alisa and Han, Xiaochuang and Wang, Yizhong and Tsvetkov, Yulia and Choi, Yejin and Smith, Noah A.},
	month = aug,
	year = {2024},
	note = {arXiv:2401.08565 [cs]},
}

@misc{jang_personalized_2023,
	title = {Personalized {Soups}: {Personalized} {Large} {Language} {Model} {Alignment} via {Post}-hoc {Parameter} {Merging}},
	shorttitle = {Personalized {Soups}},
	url = {http://arxiv.org/abs/2310.11564},
	language = {en},
	urldate = {2024-10-02},
	publisher = {arXiv},
	author = {Jang, Joel and Kim, Seungone and Lin, Bill Yuchen and Wang, Yizhong and Hessel, Jack and Zettlemoyer, Luke and Hajishirzi, Hannaneh and Choi, Yejin and Ammanabrolu, Prithviraj},
	month = oct,
	year = {2023},
	note = {arXiv:2310.11564 [cs]},
}

@article{krause_robust_2008,
	title = {Robust {Submodular} {Observation} {Selection}},
	volume = {9},
	issn = {1533-7928},
	url = {http://jmlr.org/papers/v9/krause08b.html},
	number = {93},
	urldate = {2024-09-30},
	journal = {Journal of Machine Learning Research},
	author = {Krause, Andreas and McMahan, H. Brendan and Guestrin, Carlos and Gupta, Anupam},
	year = {2008},
	pages = {2761--2801},
}

@article{qian_learning_2015,
	title = {Learning user preferences by adaptive pairwise comparison},
	volume = {8},
	issn = {2150-8097},
	url = {https://dl.acm.org/doi/10.14778/2809974.2809992},
	doi = {10.14778/2809974.2809992},
	number = {11},
	urldate = {2024-09-29},
	journal = {Proc. VLDB Endow.},
	author = {Qian, Li and Gao, Jinyang and Jagadish, H. V.},
	month = jul,
	year = {2015},
	pages = {1322--1333},
}

@article{viswanathan_american_2012,
	series = {Special {Issue}: {American} {Brachytherapy} {Society} {Guidelines} for {Prostate} and {Gynecology}},
	title = {American {Brachytherapy} {Society} consensus guidelines for locally advanced carcinoma of the cervix. {Part} {II}: {High}-dose-rate brachytherapy},
	volume = {11},
	issn = {1538-4721},
	shorttitle = {American {Brachytherapy} {Society} consensus guidelines for locally advanced carcinoma of the cervix. {Part} {II}},
	url = {https://www.sciencedirect.com/science/article/pii/S1538472111003515},
	doi = {10.1016/j.brachy.2011.07.002},
	number = {1},
	urldate = {2024-09-28},
	journal = {Brachytherapy},
	author = {Viswanathan, Akila N. and Beriwal, Sushil and De Los Santos, Jennifer F. and Demanes, D. Jeffrey and Gaffney, David and Hansen, Jorgen and Jones, Ellen and Kirisits, Christian and Thomadsen, Bruce and Erickson, Beth},
	month = jan,
	year = {2012},
	pages = {47--52},
}

@inproceedings{williams_gaussian_1995,
	title = {Gaussian {Processes} for {Regression}},
	volume = {8},
	url = {https://proceedings.neurips.cc/paper_files/paper/1995/hash/7cce53cf90577442771720a370c3c723-Abstract.html},
	urldate = {2024-09-27},
	booktitle = {Advances in {Neural} {Information} {Processing} {Systems}},
	publisher = {MIT Press},
	author = {Williams, Christopher and Rasmussen, Carl},
	year = {1995},
}

@misc{chugh_scalarizing_2019,
	title = {Scalarizing {Functions} in {Bayesian} {Multiobjective} {Optimization}},
	url = {http://arxiv.org/abs/1904.05760},
	language = {en},
	urldate = {2024-09-20},
	publisher = {arXiv},
	author = {Chugh, Tinkle},
	month = apr,
	year = {2019},
	note = {arXiv:1904.05760 [cs, stat]},
}

@article{wierzbicki_mathematical_1982,
	series = {Special {IIASA} {Issue}},
	title = {A mathematical basis for satisficing decision making},
	volume = {3},
	issn = {0270-0255},
	url = {https://www.sciencedirect.com/science/article/pii/0270025582900380},
	doi = {10.1016/0270-0255(82)90038-0},
	number = {5},
	urldate = {2024-09-20},
	journal = {Mathematical Modelling},
	author = {Wierzbicki, Andrzej P.},
	month = jan,
	year = {1982},
	pages = {391--405},
}

@article{yu_multi-objective_2000,
	series = {Evolutionary {Computation} in {Medicine}},
	title = {Multi-objective optimization in radiotherapy: applications to stereotactic radiosurgery and prostate brachytherapy},
	volume = {19},
	issn = {0933-3657},
	shorttitle = {Multi-objective optimization in radiotherapy},
	url = {https://www.sciencedirect.com/science/article/pii/S0933365799000494},
	doi = {10.1016/S0933-3657(99)00049-4},
	number = {1},
	urldate = {2024-09-20},
	journal = {Artificial Intelligence in Medicine},
	author = {Yu, Yan and Zhang, J. B and Cheng, Gang and Schell, M. C and Okunieff, Paul},
	month = may,
	year = {2000},
	pages = {39--51},
}

@article{luukkonen_artificial_2023,
	title = {Artificial intelligence in multi-objective drug design},
	volume = {79},
	issn = {0959-440X},
	url = {https://www.sciencedirect.com/science/article/pii/S0959440X23000118},
	doi = {10.1016/j.sbi.2023.102537},
	urldate = {2024-09-20},
	journal = {Current Opinion in Structural Biology},
	author = {Luukkonen, Sohvi and van den Maagdenberg, Helle W. and Emmerich, Michael T. M. and van Westen, Gerard J. P.},
	month = apr,
	year = {2023},
	pages = {102537},
}

@misc{xie_mars_2021,
	title = {{MARS}: {Markov} {Molecular} {Sampling} for {Multi}-objective {Drug} {Discovery}},
	shorttitle = {{MARS}},
	url = {http://arxiv.org/abs/2103.10432},
	language = {en},
	urldate = {2024-09-20},
	publisher = {arXiv},
	author = {Xie, Yutong and Shi, Chence and Zhou, Hao and Yang, Yuwei and Zhang, Weinan and Yu, Yong and Li, Lei},
	month = mar,
	year = {2021},
	note = {arXiv:2103.10432 [cs, q-bio]},
}

@article{fromer_computer-aided_2023,
	title = {Computer-aided multi-objective optimization in small molecule discovery},
	volume = {4},
	issn = {26663899},
	url = {https://linkinghub.elsevier.com/retrieve/pii/S2666389923000016},
	doi = {10.1016/j.patter.2023.100678},
	language = {en},
	number = {2},
	urldate = {2024-09-20},
	journal = {Patterns},
	author = {Fromer, Jenna C. and Coley, Connor W.},
	month = feb,
	year = {2023},
	pages = {100678},
}

@article{cheng_generalized_1999,
	title = {Generalized {Center} {Method} for {Multiobjective} {Engineering} {Optimization}},
	volume = {31},
	issn = {0305-215X},
	url = {https://doi.org/10.1080/03052159908941390},
	doi = {10.1080/03052159908941390},
	number = {5},
	urldate = {2024-09-14},
	journal = {Engineering Optimization},
	author = {CHENG, F. Y. and LI, X. S.},
	month = may,
	year = {1999},
	note = {Publisher: Taylor \& Francis
\_eprint: https://doi.org/10.1080/03052159908941390},
	pages = {641--661},
}

@article{tanabe_easy--use_2020,
	title = {An easy-to-use real-world multi-objective optimization problem suite},
	volume = {89},
	issn = {1568-4946},
	url = {https://www.sciencedirect.com/science/article/pii/S1568494620300181},
	doi = {10.1016/j.asoc.2020.106078},
	urldate = {2024-09-14},
	journal = {Applied Soft Computing},
	author = {Tanabe, Ryoji and Ishibuchi, Hisao},
	month = apr,
	year = {2020},
	pages = {106078},
}

@article{nemhauser_analysis_1978,
	title = {An analysis of approximations for maximizing submodular set functions—{I}},
	volume = {14},
	copyright = {http://www.springer.com/tdm},
	issn = {0025-5610, 1436-4646},
	url = {http://link.springer.com/10.1007/BF01588971},
	doi = {10.1007/BF01588971},
	language = {en},
	number = {1},
	urldate = {2024-08-12},
	journal = {Mathematical Programming},
	author = {Nemhauser, G. L. and Wolsey, L. A. and Fisher, M. L.},
	month = dec,
	year = {1978},
	pages = {265--294},
}

@article{diehl_when_2005,
	title = {When {Two} {Rights} {Make} a {Wrong}: {Searching} {Too} {Much} in {Ordered} {Environments}},
	volume = {42},
	issn = {0022-2437},
	shorttitle = {When {Two} {Rights} {Make} a {Wrong}},
	url = {https://doi.org/10.1509/jmkr.2005.42.3.313},
	doi = {10.1509/jmkr.2005.42.3.313},
	number = {3},
	urldate = {2024-08-12},
	journal = {Journal of Marketing Research},
	author = {Diehl, Kristin},
	month = aug,
	year = {2005},
	note = {Publisher: SAGE Publications Inc},
	pages = {313--322},
}

@inproceedings{bryan_actively_2008,
	address = {Helsinki, Finland},
	title = {Actively learning level-sets of composite functions},
	isbn = {978-1-60558-205-4},
	url = {http://portal.acm.org/citation.cfm?doid=1390156.1390167},
	doi = {10.1145/1390156.1390167},
	language = {en},
	urldate = {2024-08-11},
	booktitle = {Proceedings of the 25th international conference on {Machine} learning - {ICML} '08},
	publisher = {ACM Press},
	author = {Bryan, Brent and Schneider, Jeff},
	year = {2008},
	pages = {80--87},
}

@article{iwazaki_bayesian_2020,
	title = {Bayesian {Experimental} {Design} for {Finding} {Reliable} {Level} {Set} {Under} {Input} {Uncertainty}},
	volume = {8},
	issn = {2169-3536},
	url = {https://ieeexplore.ieee.org/document/9252941/?arnumber=9252941},
	doi = {10.1109/ACCESS.2020.3036863},
	urldate = {2024-08-11},
	journal = {IEEE Access},
	author = {Iwazaki, Shogo and Inatsu, Yu and Takeuchi, Ichiro},
	year = {2020},
	note = {Conference Name: IEEE Access},
	pages = {203982--203993},
}

@misc{zanette_robust_2018,
	title = {Robust {Super}-{Level} {Set} {Estimation} using {Gaussian} {Processes}},
	url = {http://arxiv.org/abs/1811.09977},
	language = {en},
	urldate = {2024-08-11},
	publisher = {arXiv},
	author = {Zanette, Andrea and Zhang, Junzi and Kochenderfer, Mykel J.},
	month = nov,
	year = {2018},
	note = {arXiv:1811.09977 [cs, stat]},
}

@inproceedings{bryan_active_2005,
	title = {Active {Learning} {For} {Identifying} {Function} {Threshold} {Boundaries}},
	volume = {18},
	url = {https://proceedings.neurips.cc/paper/2005/hash/8e930496927757aac0dbd2438cb3f4f6-Abstract.html},
	urldate = {2024-08-11},
	booktitle = {Advances in {Neural} {Information} {Processing} {Systems}},
	publisher = {MIT Press},
	author = {Bryan, Brent and Nichol, Robert C. and Genovese, Christopher R and Schneider, Jeff and Miller, Christopher J. and Wasserman, Larry},
	year = {2005},
}

@article{gotovos_active_2013,
	title = {Active {Learning} for {Level} {Set} {Estimation}},
	copyright = {http://rightsstatements.org/page/InC-NC/1.0/, info:eu-repo/semantics/openAccess},
	url = {http://hdl.handle.net/20.500.11850/153931},
	doi = {10.3929/ETHZ-A-009767767},
	language = {en},
	urldate = {2024-08-11},
	author = {Gotovos, Alkis},
	year = {2013},
	note = {Medium: application/pdf,Online-Ressource
Publisher: ETH Zurich},
}

@article{knowles_parego_2006,
	title = {{ParEGO}: a hybrid algorithm with on-line landscape approximation for expensive multiobjective optimization problems},
	volume = {10},
	issn = {1941-0026},
	shorttitle = {{ParEGO}},
	url = {https://ieeexplore.ieee.org/document/1583627},
	doi = {10.1109/TEVC.2005.851274},
	number = {1},
	urldate = {2024-08-11},
	journal = {IEEE Transactions on Evolutionary Computation},
	author = {Knowles, J.},
	month = feb,
	year = {2006},
	note = {Conference Name: IEEE Transactions on Evolutionary Computation},
	pages = {50--66},
}

@article{zhang_moead_2007,
	title = {{MOEA}/{D}: {A} {Multiobjective} {Evolutionary} {Algorithm} {Based} on {Decomposition}},
	volume = {11},
	issn = {1941-0026},
	shorttitle = {{MOEA}/{D}},
	url = {https://ieeexplore.ieee.org/document/4358754},
	doi = {10.1109/TEVC.2007.892759},
	number = {6},
	urldate = {2024-08-11},
	journal = {IEEE Transactions on Evolutionary Computation},
	author = {Zhang, Qingfu and Li, Hui},
	month = dec,
	year = {2007},
	note = {Conference Name: IEEE Transactions on Evolutionary Computation},
	pages = {712--731},
}

@article{zhang_expensive_2010,
	title = {Expensive {Multiobjective} {Optimization} by {MOEA}/{D} {With} {Gaussian} {Process} {Model}},
	volume = {14},
	issn = {1941-0026},
	url = {https://ieeexplore.ieee.org/document/5353656},
	doi = {10.1109/TEVC.2009.2033671},
	number = {3},
	urldate = {2024-08-11},
	journal = {IEEE Transactions on Evolutionary Computation},
	author = {Zhang, Qingfu and Liu, Wudong and Tsang, Edward and Virginas, Botond},
	month = jun,
	year = {2010},
	note = {Conference Name: IEEE Transactions on Evolutionary Computation},
	pages = {456--474},
}

@article{picheny_multiobjective_2015,
	title = {Multiobjective optimization using {Gaussian} process emulators via stepwise uncertainty reduction},
	volume = {25},
	issn = {1573-1375},
	url = {https://doi.org/10.1007/s11222-014-9477-x},
	doi = {10.1007/s11222-014-9477-x},
	language = {en},
	number = {6},
	urldate = {2024-08-11},
	journal = {Statistics and Computing},
	author = {Picheny, Victor},
	month = nov,
	year = {2015},
	pages = {1265--1280},
}

@article{emmerich_computation_2008,
	title = {The computation of the expected improvement in dominated hypervolume of {Pareto} front approximations},
	author = {Emmerich, Michael},
	month = jan,
	year = {2008},
}

@article{campigotto_active_2014,
	title = {Active {Learning} of {Pareto} {Fronts}},
	volume = {25},
	issn = {2162-2388},
	url = {https://ieeexplore.ieee.org/document/6606803},
	doi = {10.1109/TNNLS.2013.2275918},
	number = {3},
	urldate = {2024-08-11},
	journal = {IEEE Transactions on Neural Networks and Learning Systems},
	author = {Campigotto, Paolo and Passerini, Andrea and Battiti, Roberto},
	month = mar,
	year = {2014},
	note = {Conference Name: IEEE Transactions on Neural Networks and Learning Systems},
	pages = {506--519},
}

@article{zuluaga_e-pal_2016,
	title = {e-{PAL}: {An} {Active} {Learning} {Approach} to the {Multi}-{Objective} {Optimization} {Problem}},
	volume = {17},
	issn = {1533-7928},
	shorttitle = {e-{PAL}},
	url = {http://jmlr.org/papers/v17/15-047.html},
	number = {104},
	urldate = {2024-08-11},
	journal = {Journal of Machine Learning Research},
	author = {Zuluaga, Marcela and Krause, Andreas and P\{ü\}schel, Markus},
	year = {2016},
	pages = {1--32},
}

@inproceedings{ponweiser_multiobjective_2008,
	address = {Berlin, Heidelberg},
	title = {Multiobjective {Optimization} on a {Limited} {Budget} of {Evaluations} {Using} {Model}-{Assisted} \${\textbackslash}mathcal\{{S}\}\$-{Metric} {Selection}},
	isbn = {978-3-540-87700-4},
	doi = {10.1007/978-3-540-87700-4_78},
	language = {en},
	booktitle = {Parallel {Problem} {Solving} from {Nature} – {PPSN} {X}},
	publisher = {Springer},
	author = {Ponweiser, Wolfgang and Wagner, Tobias and Biermann, Dirk and Vincze, Markus},
	editor = {Rudolph, Günter and Jansen, Thomas and Beume, Nicola and Lucas, Simon and Poloni, Carlo},
	year = {2008},
	pages = {784--794},
}

@misc{hernandez-lobato_predictive_2016,
	title = {Predictive {Entropy} {Search} for {Multi}-objective {Bayesian} {Optimization}},
	url = {http://arxiv.org/abs/1511.05467},
	doi = {10.48550/arXiv.1511.05467},
	urldate = {2024-08-11},
	publisher = {arXiv},
	author = {Hernández-Lobato, Daniel and Hernández-Lobato, José Miguel and Shah, Amar and Adams, Ryan P.},
	month = feb,
	year = {2016},
	note = {arXiv:1511.05467 [stat]},
}

@inproceedings{abdolshah_multi-objective_2019,
	title = {Multi-objective {Bayesian} optimisation with preferences over objectives},
	volume = {32},
	url = {https://papers.nips.cc/paper_files/paper/2019/hash/a7b7e4b27722574c611fe91476a50238-Abstract.html},
	urldate = {2024-08-11},
	booktitle = {Advances in {Neural} {Information} {Processing} {Systems}},
	publisher = {Curran Associates, Inc.},
	author = {Abdolshah, Majid and Shilton, Alistair and Rana, Santu and Gupta, Sunil and Venkatesh, Svetha},
	year = {2019},
}

@article{cui_multi-criteria_2018,
	title = {A multi-criteria optimization approach for {HDR} prostate brachytherapy: {II}. {Benchmark} against clinical plans},
	volume = {63},
	issn = {1361-6560},
	shorttitle = {A multi-criteria optimization approach for {HDR} prostate brachytherapy},
	url = {https://iopscience.iop.org/article/10.1088/1361-6560/aae24f},
	doi = {10.1088/1361-6560/aae24f},
	language = {en},
	number = {20},
	urldate = {2024-07-14},
	journal = {Physics in Medicine \& Biology},
	author = {Cui, Songye and Després, Philippe and Beaulieu, Luc},
	month = oct,
	year = {2018},
	pages = {205005},
}

@article{roijers_survey_2013,
	title = {A {Survey} of {Multi}-{Objective} {Sequential} {Decision}-{Making}},
	volume = {48},
	issn = {1076-9757},
	url = {https://www.jair.org/index.php/jair/article/view/10836},
	doi = {10.1613/jair.3987},
	language = {en},
	urldate = {2023-09-02},
	journal = {Journal of Artificial Intelligence Research},
	author = {Roijers, D. M. and Vamplew, P. and Whiteson, S. and Dazeley, R.},
	month = oct,
	year = {2013},
	pages = {67--113},
}

@incollection{trautmann_using_2017,
	address = {Cham},
	title = {On {Using} {Decision} {Maker} {Preferences} with {ParEGO}},
	volume = {10173},
	isbn = {978-3-319-54156-3 978-3-319-54157-0},
	url = {http://link.springer.com/10.1007/978-3-319-54157-0_20},
	language = {en},
	urldate = {2023-07-22},
	booktitle = {Evolutionary {Multi}-{Criterion} {Optimization}},
	publisher = {Springer International Publishing},
	author = {Hakanen, Jussi and Knowles, Joshua D.},
	editor = {Trautmann, Heike and Rudolph, Günter and Klamroth, Kathrin and Schütze, Oliver and Wiecek, Margaret and Jin, Yaochu and Grimme, Christian},
	year = {2017},
	doi = {10.1007/978-3-319-54157-0_20},
	note = {Series Title: Lecture Notes in Computer Science},
	pages = {282--297},
}

@article{while_fast_2012,
	title = {A {Fast} {Way} of {Calculating} {Exact} {Hypervolumes}},
	volume = {16},
	issn = {1941-0026},
	doi = {10.1109/TEVC.2010.2077298},
	number = {1},
	journal = {IEEE Transactions on Evolutionary Computation},
	author = {While, Lyndon and Bradstreet, Lucas and Barone, Luigi},
	month = feb,
	year = {2012},
	note = {Conference Name: IEEE Transactions on Evolutionary Computation},
	pages = {86--95},
}

@inproceedings{astudillo_multi-attribute_2020,
	title = {Multi-attribute {Bayesian} optimization with interactive preference learning},
	url = {https://proceedings.mlr.press/v108/astudillo20a.html},
	language = {en},
	urldate = {2023-07-10},
	booktitle = {Proceedings of the {Twenty} {Third} {International} {Conference} on {Artificial} {Intelligence} and {Statistics}},
	publisher = {PMLR},
	author = {Astudillo, Raul and Frazier, Peter},
	month = jun,
	year = {2020},
	note = {ISSN: 2640-3498},
	pages = {4496--4507},
}

@misc{paria_flexible_2019,
	title = {A {Flexible} {Framework} for {Multi}-{Objective} {Bayesian} {Optimization} using {Random} {Scalarizations}},
	url = {http://arxiv.org/abs/1805.12168},
	urldate = {2023-07-08},
	publisher = {arXiv},
	author = {Paria, Biswajit and Kandasamy, Kirthevasan and Póczos, Barnabás},
	month = jun,
	year = {2019},
	note = {arXiv:1805.12168 [cs, stat]},
}

@inproceedings{suzuki_multi-objective_2020,
	title = {Multi-objective {Bayesian} {Optimization} using {Pareto}-frontier {Entropy}},
	url = {https://proceedings.mlr.press/v119/suzuki20a.html},
	language = {en},
	urldate = {2023-06-20},
	booktitle = {Proceedings of the 37th {International} {Conference} on {Machine} {Learning}},
	publisher = {PMLR},
	author = {Suzuki, Shinya and Takeno, Shion and Tamura, Tomoyuki and Shitara, Kazuki and Karasuyama, Masayuki},
	month = nov,
	year = {2020},
	note = {ISSN: 2640-3498},
	pages = {9279--9288},
}

@article{deufel_pnav_2020,
	title = {{PNaV}: {A} tool for generating a high-dose-rate brachytherapy treatment plan by navigating the {Pareto} surface guided by the visualization of multidimensional trade-offs},
	volume = {19},
	issn = {15384721},
	shorttitle = {{PNaV}},
	url = {https://linkinghub.elsevier.com/retrieve/pii/S1538472120300325},
	doi = {10.1016/j.brachy.2020.02.013},
	language = {en},
	number = {4},
	urldate = {2022-07-12},
	journal = {Brachytherapy},
	author = {Deufel, Christopher L. and Epelman, Marina A. and Pasupathy, Kalyan S. and Sir, Mustafa Y. and Wu, Victor W. and Herman, Michael G.},
	month = jul,
	year = {2020},
	pages = {518--531},
}

@article{belanger_gpu-based_2019,
	title = {A {GPU}-based multi-criteria optimization algorithm for {HDR} brachytherapy},
	volume = {64},
	issn = {1361-6560},
	url = {https://iopscience.iop.org/article/10.1088/1361-6560/ab1817},
	doi = {10.1088/1361-6560/ab1817},
	language = {en},
	number = {10},
	urldate = {2022-07-12},
	journal = {Physics in Medicine \& Biology},
	author = {Bélanger, Cédric and Cui, Songye and Ma, Yunzhi and Després, Philippe and Adam M Cunha, J and Beaulieu, Luc},
	month = may,
	year = {2019},
	pages = {105005},
}
\setlength{\itemindent}{-\leftmargin}
\makeatletter\renewcommand{\@biblabel}[1]{}\makeatother

\section*{Checklist}

\begin{enumerate}

  \item For all models and algorithms presented, check if you include:
  \begin{enumerate}
    \item A clear description of the mathematical setting, assumptions, algorithm, and/or model. Yes. Comprehensive descriptions of the setting, assumptions, and algorithms may be found in Sections \ref{background}, \ref{dense-sampling}, and \ref{sec:step-2-method}. Additional details are presented in Appendix \ref{step1-additional-details} and \ref{appendix-step-2-details}.
    \item An analysis of the properties and complexity (time, space, sample size) of any algorithm. Yes. For Step 1, Section \ref{dense-sampling} and Appendices \ref{step1-additional-details} and \ref{sec:computational-feasible-details} provide extensive analyses of the properties and motivation. For Step 2, Section \ref{sec:step-2-method} and Appendix \ref{appendix-step-2-details} provide details regarding the properties and complexity of the algorithm.
    \item (Optional) Anonymized source code, with specification of all dependencies, including external libraries. Yes, provided as part of the supplemental material.
  \end{enumerate}

  \item For any theoretical claim, check if you include:
  \begin{enumerate}
    \item Statements of the full set of assumptions of all theoretical results. Yes. Appendix \ref{app:definitions_theorems} explicitly mentions the assumptions made for the theoretical results.
    \item Complete proofs of all theoretical results. Yes. All proofs are provided in Appendix \ref{app:definitions_theorems}.
    \item Clear explanations of any assumptions. Yes. Appendix \ref{app:definitions_theorems} provides clear explanations of any assumptions.     
  \end{enumerate}

  \item For all figures and tables that present empirical results, check if you include:
  \begin{enumerate}
    \item The code, data, and instructions needed to reproduce the main experimental results (either in the supplemental material or as a URL). Yes, provided as part of the supplemental material.
    \item All the training details (e.g., data splits, hyperparameters, how they were chosen). Yes. Extensive training details for all of the experiments are provided in Appendix \ref{additional-experimental-results}. Additional details regarding number of trials and performance evaluation are listed in Appendix \ref{app:experimental-setup-details}.
    \item A clear definition of the specific measure or statistics and error bars (e.g., with respect to the random seed after running experiments multiple times). Yes. Such information is presented in each of the figures throughout this manuscript. Additional definitions are provided in Appendix \ref{app:subsubsec:baseline_methodologies}.
    \item A description of the computing infrastructure used. (e.g., type of GPUs, internal cluster, or cloud provider). Yes, these are provided in Appendix \ref{app:compute-resources}.
  \end{enumerate}

  \item If you are using existing assets (e.g., code, data, models) or curating/releasing new assets, check if you include:
  \begin{enumerate}
    \item Citations of the creator If your work uses existing assets. Yes. The models used for the LLM personalization experiment are properly cited in Appendix \ref{appendix:llm-concise-info}. Usage of the benchmarks and datasets for the other experiments are all properly cited in Appendix \ref{additional-experimental-results}.
    \item The license information of the assets, if applicable. Not Applicable
    \item New assets either in the supplemental material or as a URL, if applicable. Not Applicable
    \item Information about consent from data providers/curators. Yes. For the brachytherapy treatment planning data, consent is explicitly described in Appendix \ref{appendix:brachytherapy-details}.
    \item Discussion of sensible content if applicable, e.g., personally identifiable information or offensive content. Not Applicable
  \end{enumerate}

  \item If you used crowdsourcing or conducted research with human subjects, check if you include:
  \begin{enumerate}
    \item The full text of instructions given to participants and screenshots. Not Applicable
    \item Descriptions of potential participant risks, with links to Institutional Review Board (IRB) approvals if applicable. Not Applicable
    \item The estimated hourly wage paid to participants and the total amount spent on participant compensation. Not Applicable
  \end{enumerate}

\end{enumerate}

\clearpage
\appendix
\thispagestyle{empty}

\onecolumn
\aistatstitle{Supplementary Material}

\section{MODELING MULTI-OBJECTIVE TRADEOFFS WITH MONOTONIC UTILITY FUNCTIONS}\label{app:additional-details}

We articulate the differences between MoSH and prior methods in Table \ref{tab:descriptive-differences-table} below.

\begin{table}[htbp]
    \centering
    \caption{Descriptive Table Of Differences Comparing MoSH To Existing Methods.}
    \label{tab:descriptive-differences-table}
    \begin{tabular}{l p{3.5cm} p{2.5cm} l}
        \toprule
        \textbf{Method} & \textbf{Prior Specification} & \textbf{Cognitive Load for DM} & \textbf{Theoretical Guarantees?} \\
        \midrule
        \textbf{MoSH (Ours)} & Explicit (via MFs) & Low (High utility within just 5 points) & Yes (Coverage \& Optimality) \\
        \addlinespace %
        Pairwise Feedback \\ \citep{qian_learning_2015} \\ \citep{astudillo_multi-attribute_2020} \\ \citep{ozaki_multi-objective_2023} & Indirect (Iterative) & High (Requires many queries) & No \\
        \addlinespace
        Hypervolume-Based \\ \citep{ponweiser_multiobjective_2008} & No & Very High (Returns large sets) & Yes \\
        \addlinespace
        Random Scalarizations \\ \citep{paria_flexible_2019} & Yes (Not intuitive for DM) & High (Large sets, unintuitive prior) & Yes \\
        \bottomrule
    \end{tabular}
\end{table}

\section{STEP 1: DENSE PARETO FRONTIER SAMPLING WITH BAYESIAN OPTIMIZATION ADDITIONAL DETAILS}\label{step1-additional-details}
Here we describe the acquisition function used in line 7 of step 1, Algorithm \ref{algorithm:mosh}. For our experiments, we use the Upper Confidence Bound (UCB) heuristic. We define acq($u, \lambda_t, x$) $= s_{\lambda_t} (u_{\varphi(x)})$ where $\varphi(x) = \mu_t(x) + \sqrt{\beta_t} \sigma_t({x})$ and $\beta_t = \sqrt{0.125 \times \log(2\times t + 1)}$. For $\beta_t$, we followed the optimal suggestion in \citet{paria_flexible_2019}. We emphasize that \textit{our algorithm is agnostic to the scalarization and acquisition functions}.

\section{STEP 2: PARETO FRONTIER SPARSIFICATION ADDITIONAL DETAILS}\label{appendix-step-2-details}
\begin{theorem}~\citep{nemhauser_analysis_1978} In the case of any normalized, nomotonic submodular function $F$, the set $A_G$ obtained by the greedy algorithm achieves at least a constant fraction $(1-\dfrac{1}{e})$ of the objective value obtained by the optimal solution, that is,

\begin{center}
    $F(A_G) \geq (1-\dfrac{1}{e}) \max_{|A| \leq k} F(A)$
\end{center}
\end{theorem}

Algorithm \ref{algorithm:gpc} describes the greedy submodular partial cover (GPC) algorithm in detail, used as a subroutine in Algorithm \ref{algorithm:saturate}.

\begin{algorithm}[H]
    \caption{Greedy Submodular Partial Cover (GPC) Algorithm \citep{krause_robust_2008}}
    \label{algorithm:gpc}
    \begin{algorithmic}[1]
    \Procedure{GPC}{$\Bar{F}_q, q$}
        \State $C \gets \emptyset$
        
        \While{$\Bar{F}_q(C) < q$}
            \State \textbf{foreach} $c \in D \setminus C$ \textbf{do} $\delta_c = \Bar{F}_q(C \cup \{c\}) - \Bar{F}_q(C)$
            
            \State $C \gets C \cup \{\operatorname*{argmax}_{c} \delta_c\}$
        \EndWhile %
        
        \State \Return $C$
    \EndProcedure
    \end{algorithmic}
\end{algorithm}

\section{REQUIRED PROOFS AND DEFINITIONS}
\label{app:definitions_theorems}

\begin{definition}[Pareto dominant]
    A solution $x_1 \in X$ is Pareto dominated by another point $x_2 \in X$ if and only if $f_\ell(x_1) \leq f_\ell(x_2)$ $\forall \ell \in [L]$ and $\exists \ell \in [L]$ s.t. $f_\ell(x_1) < f_\ell(x_2)$~\citep{paria_flexible_2019}.
\end{definition}

\begin{definition}[Submodular]
    $F$ is submodular if and only if for all $A \subseteq B \subseteq V$ and $s \in V \text{\textbackslash} B$ it holds that $F(A \cup \{s\})-F(A) \geq F(B \cup \{s\}) - F(B)$
\end{definition}

\begin{theorem}
    Consider finite sets $\Omega$, $\Lambda$ and function $f: \Omega \times \Lambda \rightarrow \mathbb{R}$. Fix $\lambda \in \Lambda$. Then, given $C\subset\Omega$, the set function
    \begin{equation*}
        F(C) := \frac{\max_{c\in C} f(c, \lambda)}{\max_{c\in\Omega} f(c, \lambda)}
    \end{equation*}
    is submodular.
\end{theorem}
\begin{proof}
    Let $X \subseteq Y \subseteq \Omega$ and $x\in\Omega\backslash Y$ (note: it clearly follows that $x\notin X$). Consider $x^{*} := \max_{c\in\Omega} f(c, \lambda)$. We obtain the following mutually exclusive and exhaustive cases, in which we use the fact that for any sets $A, B$ such that $A\subseteq B$, $\max_{a\in A} f(a) \leq \max_{b\in B} f(b)$ for any function $f$.
    \begin{itemize}
        \item $x^{*} \in X \subseteq Y$. By definition, $F(X\cup \{x\}) - F(X) = F(Y\cup\{x\}) - F(Y) = 0$.
        \item $x^{*} \in Y\backslash X$. By definition, $F(Y\cup\{x\}) - F(Y) = 0$. Clearly $X \subseteq X\cup\{x\}$, thus $F(X\cup\{x\}) - F(X) \geq 0 = F(Y\cup\{x\}) - F(Y)$.
        \item $x^{*} \in \Omega\backslash Y$ and $x^{*} \neq x$. 
        \begin{itemize}
            \item If $F(x, \lambda) \geq \max_{c\in Y} f(c, \lambda)$, then $F(X \cup \{x\}) = F(Y \cup \{x\}) = F(x, \lambda)$.
            \item If $F(x, \lambda) \leq \max_{c\in X} f(c, \lambda)$, $F(X \cup\{x\}) = F(X)$ and $F(Y\cup\{x\}) = F(Y)$.
            \item Finally, if $\max_{c\in X} f(c, \lambda) \leq F(x, \lambda) \leq \max_{c\in Y} f(c, \lambda)$, then $F(Y\cup\{x\}) = F(Y)$ and $F(X\cup\{x\}) \geq F(X)$.
        \end{itemize}
        Combining with the fact since $X\subseteq Y$, $F(X) \leq F(Y)$, in all of the above sub-cases, we indeed have $F(X\cup\{x\}) - F(X) \geq F(Y\cup\{x\}) - F(Y)$.
        \item $x^{*} = x$. As $F(X\cup\{x\}) = F(Y\cup\{x\})$, and $F(X) \leq F(Y)$, we automatically get $F(X\cup\{x\}) - F(X) \geq F(Y\cup\{x\}) - F(Y)$.
    \end{itemize}
    Therefore, we satisfy the definition of submodularity.
\end{proof}

\begin{definition}[Instantaneous \mf Regret]
    $r(x_t, \pmb{\lambda}_t) = 1 - \dfrac{s_{\pmb{\lambda}_t}(u_f(x_t))}{\max_{x \in X}s_{\pmb{\lambda}_t}(u_f(x_t))}$
\end{definition}

\begin{definition}[Cumulative \mf Regret]
    $R_C (T) = \sum_{t=1}^{T} r(x_t, \pmb{\lambda}_t)$
\end{definition}

\begin{definition}[Bayes \mf Regret and Utility Ratio]
    $R_B (T) = \mathbb{E}_{\lambda \sim p(\lambda)}[1 - \dfrac{\max_{x \in D_T} s_{\pmb{\lambda}}(u_f(x))}{\max_{x \in X}s_{\pmb{\lambda}}(u_f(x))}]$ where $D_T = \{x_t\}_{t=1}^T$. Additionally, $U_B (T) = \mathbb{E}_{\lambda \sim p(\lambda)}[\dfrac{\max_{x \in D^T} s_{\pmb{\lambda}}(u_f(x))}{\max_{x \in X}s_{\pmb{\lambda}}(u_f(x))}]$, where $U_B(T)$ is the Bayes \mf Utility Ratio after $T$ iterations.
\end{definition}

\begin{definition}[Expected Bayes \mf Regret and Expected Cumulative \mf Regret]
    Similar to \citep{paria_flexible_2019}, $\mathbb{E}R_B(T)$ is the expected Bayes \mf Regret, with the expectation being taken over $f$, noise $\epsilon$, and other sources of randomness. Likewise, $\mathbb{E}R_C(T)$ is the expected Cumulative \mf Regret, with the expectation being taken over $f$, noise $\epsilon$, and $\pmb{\lambda}_t$.
\end{definition}

\begin{definition}[Maximum Information Gain]
    We leverage this definition from \citep{paria_flexible_2019}. The maximum information gain after T observations measures the notion of information gained about random process $f$ after observing some set of points $A$, and is defined as:
    \begin{equation}
        \gamma_T = \max_{A \subset X: |A| = T} I(y_A; f)
    \end{equation}
\end{definition}

\begin{definition}[Lipschitz Condition]\label{app:def:lipschitz_input}
    We assume the following is $M_{\pmb{\lambda}}$-Lipschitz in the $\ell_1$-norm for all $\pmb{\lambda} \in \Lambda$, 
    \begin{equation}
        \Big\vert\dfrac{s_{\pmb{\lambda}}(y_1)}{\max_{y \in \text{Im}(u_f)} s_{\pmb{\lambda}}(y)} - \dfrac{s_{\pmb{\lambda}}(y_2)}{\max_{y \in \text{Im}(u_f)} s_{\pmb{\lambda}}(y)}\Big\vert \leq M_{\pmb{\lambda}} \|y_1 - y_2\|_1
    \end{equation}
    where $y \in \mathbb{R}^L$ corresponds to $u_f(x)$.
\end{definition}

\begin{definition}[Lipschitz Condition]\label{app:def:lipschitz_lambda}
    We assume the following is J-Lipschitz in $\pmb{\lambda}$ for all $y \in \mathbb{R}^L$.
    \begin{equation}
        \Big\vert\dfrac{s_{\pmb{\lambda}_1}(y)}{\max_{y' \in \text{Im}(u_f)} s_{\pmb{\lambda}_1}(y')} - \dfrac{s_{\pmb{\lambda}_2}(y)}{\max_{y' \in \text{Im}(u_f)} s_{\pmb{\lambda}_2}(y')} \Big\vert \leq J \|\pmb{\lambda}_1 - \pmb{\lambda}_2\|_1
    \end{equation}
\end{definition}

\begin{definition}[Regret Bounds]
    We follow similar notation as in \citep{paria_flexible_2019}. Assume that $\forall \ell \in [L], t \in [T], x \in \mathcal{X}$, each objective $f_\ell(x)$ follows a Gaussian distribution with marginal variances upper bounded by 1, and the observation noise $\epsilon_{t\ell} \sim \mathcal{N}(0, \sigma_\ell^2)$ is drawn independently of everything else. We assume upper bounds $M_\lambda \leq M$, $\sigma_\ell^2 \leq \sigma^2$, $\gamma_{T\ell} \leq \gamma_T$, where $\gamma_{T\ell}$ is the maximum information gain for the $\ell$th objective. We assume $\mathcal{X} \subseteq [0,1]^d$. Furthermore, let $x^*_t = \argmax_{x \in X}s_{\pmb{\lambda_t}}(u_f(x))$. We denote by $U_t(\pmb{\lambda}, x) = s_{\pmb{\lambda}} (u_{\varphi(x)})$ where $\varphi(x) = \mu_t(x) + \sqrt{\beta_t} \sigma_t({x})$. Finally, the history until T-1 is denoted as $\mathcal{H}_t$, i.e. $\{(x_t, y_t, \pmb{\lambda}_t)\}_{t=1}^{T-1}.$
\end{definition}

\begin{theorem}
    The expected cumulative \mf regret for dense sampling after T observations can be upper bounded for both Upper Confidence Bound and Thompson Sampling as,
    \begin{equation*}
        \mathbb{E} R_C(T) = O(M[\dfrac{L^2 T d \gamma_T ln T}{ln (1 + \sigma^{-2})}]^{1/2})
    \end{equation*}
\end{theorem}    

\begin{proof}
    \begin{align*}
        \mathbb{E}R_C(T) = \mathbb{E}\left[\sum_{t=1}^T \left(1 - \frac{s_{\lambda_t}(u_f(x_t))}{\max_{x\in X} s_{\lambda_t}(u_f(x))}\right)\right] \\
        = \mathbb{E}\left[\sum_{t=1}^T \left(\dfrac{\max_{x\in X} s_{\lambda_t}(u_f(x))}{\max_{x\in X} s_{\lambda_t}(u_f(x))} - \frac{s_{\lambda_t}(u_f(x_t))}{\max_{x\in X} s_{\lambda_t}(u_f(x))}\right)\right]
    \end{align*}
    
    Using Lemma 5 from \citet{paria_flexible_2019}, we have the following decomposition for UCB:
    \begin{align*}
    \mathbb{E}R_C(T) &\leq \underbrace{\mathbb{E}\left[\sum_{t=1}^T \dfrac{U_t(\lambda_t,x_t)}{\max_{x\in X} s_{\lambda_t}(u_f(x))} - \dfrac{s_{\lambda_t}(u_f(x_t))}{\max_{x\in X} s_{\lambda_t}(u_f(x))}\right]}_{B1} + \\
    &\underbrace{\quad \mathbb{E}\left[\sum_{t=1}^T \dfrac{s_{\lambda_t}(u_f([x_t^*]_t))}{\max_{x\in X} s_{\lambda_t}(u_f(x))} - \dfrac{U_t(\lambda_t,[x_t^*]_t)}{\max_{x\in X} s_{\lambda_t}(u_f(x))}\right]}_{B2} + \\
    &\underbrace{\quad \mathbb{E}\left[\sum_{t=1}^T \dfrac{s_{\lambda_t}(u_f(x_t^*))}{\max_{x\in X} s_{\lambda_t}(u_f(x))} - \dfrac{s_{\lambda_t}(u_f([x_t^*]_t))}{\max_{x\in X} s_{\lambda_t}(u_f(x))}\right]}_{B3}
    \end{align*}

    By using Definition 12, we extend Lemma 3 from \citet{paria_flexible_2019} and obtain the following:

    \begin{lemma}
    \begin{align*}
        &\mathbb{E}\left[\sum_{t=1}^T \frac{U_t(\lambda_t,x_t)}{\max_{x\in X} s_{\lambda_t}(u_f(x))} - \frac{s_{\lambda_t}(u_f(x_t))}{\max_{x\in X} s_{\lambda_t}(u_f(x))}\right] \\
        &\quad \leq \mathbb{E}\left[\left(L \beta_T \sum_{t=1}^T M^2_{\pmb{\lambda}_t}\right)^{1/2} \left(\sum_{\ell=1}^{L} \frac{\gamma_{T\ell}}{\ln(1+\sigma_\ell^{-2})}\right)^{1/2}\right] \\
        &\quad \quad +\frac{\pi^2}{6}\frac{L\mathbb{E}[M_{\pmb{\lambda}}]}{|X|}
    \end{align*}
\end{lemma}

    By using Definition 12, we extend Lemma 2 from \citep{paria_flexible_2019} and obtain the following:

    \begin{lemma}
        $\mathbb{E}\left[\sum_{t=1}^T \dfrac{s_{\lambda_t}(u_f(x^*_t))}{\max_{x\in X} s_{\lambda_t}(u_f(x))} - \dfrac{U_t(\lambda_t,x^*_t)}{\max_{x\in X} s_{\lambda_t}(u_f(x))} - \right] \leq \dfrac{\pi^2}{6} \mathbb{E}[M_{\pmb{\lambda}}] L$
    \end{lemma}

    By using Definition 12, we extend Equation (17) from \citep{paria_flexible_2019} and obtain the following:

    \begin{lemma}
        $\mathbb{E}\left[|\dfrac{s_{\lambda}(u_f(x))}{\max_{x\in X} s_{\lambda_t}(u_f(x))} - \dfrac{s_{\lambda}(u_f([x]_t))}{\max_{x\in X} s_{\lambda_t}(u_f(x))}|\right] \leq L \mathbb{E} [M_{\pmb{\lambda}}] \dfrac{1}{t^2}$
    \end{lemma}
    
    Finally, we use Lemma 16 to bound B1, Lemma 17 to bound B2, Lemma 18 to bound B3, and obtain:
    \begin{equation}
    \mathbb{E}R_C(T) \leq C_1 L \mathbb{E}[M_{\pmb{\lambda}}] + C_2 \Bar{M}_{\pmb{\lambda}} \left(LT(d\ln T + d\ln d)\sum_{\ell=1}^L \dfrac{\gamma_{T\ell}}{\ln(1+\sigma_\ell^{-2})}\right)^{1/2}
    \end{equation}
    
    which converges to 0 as $T \to \infty$.
\end{proof}

\begin{theorem}
    The expected Bayes \mf regret can be upper bounded as: 
    \begin{equation*}
        \mathbb{E}R_B(T) \leq \frac{1}{T}\mathbb{E}R_C(T) + o(1)
    \end{equation*}
    As a result, showing that the \mf utility ratio converges to 1 as $T \to \infty$ during dense sampling.
\end{theorem}
\begin{proof}
    
    We assume that $\Lambda$ is a bounded subset of a normed linear space. Following the approach in \citet{paria_flexible_2019}, we begin by relating the sampling distribution to the empirical distribution. Let $\hat{p}$ denote the empirical distribution corresponding to the samples $\{\lambda_t\}_{t=1}^T$. Consider the Wasserstein (Earth Mover's) distance between the sampling distribution $p(\lambda)$ and $\hat{p}$:
    
    \begin{equation}
    W_1(p,\hat{p}) = \inf_q\{\mathbb{E}_q\|X - Y\|_1 : q(X)=p, q(Y)=\hat{p}\}
    \end{equation}
    
    where $q$ is a joint distribution on random variables $X,Y$ with marginals $p$ and $\hat{p}$ respectively.
    
    We can then use Definition 13 to bound the following in the \mf setting:
    \begin{align*}
    \frac{1}{T}\sum_{t=1}^T (1- \dfrac{s_{\lambda_t}(u_f(x_t))}{\max_{x \in X} s_{\lambda_t} (u_f(x))}) - \mathbb{E}\left[1- \dfrac{\max_{x \in D} s_{\lambda_t}(u_f(x_t))}{\max_{x \in X} s_{\lambda_t} (u_f(x))}\right] \\
    \geq \frac{1}{T}\sum_{t=1}^T (1- \dfrac{\max_{x \in D} s_{\lambda_t}(u_f(x))}{\max_{x \in X} s_{\lambda_t} (u_f(x))}) - \mathbb{E}\left[1- \dfrac{\max_{x \in D} s_{\lambda_t}(u_f(x))}{\max_{x \in X} s_{\lambda_t} (u_f(x))} \right] \\
    \geq \mathbb{E}_{q(Z,Y)}\left[(1- \dfrac{\max_{x \in D} s_{Y}(u_f(x))}{\max_{x \in X} s_{Y} (u_f(x))}) - (1- \dfrac{\max_{x \in D} s_{Z}(u_f(x))}{\max_{x \in X} s_{Z} (u_f(x))}) \right] \\
    \geq -\mathbb{E}_{q(Z,Y)}\{J\|Z - Y\|_1\}
    \end{align*}
    
    Taking the infimum with respect to $q$ and expectation with respect to the history $\mathcal{H}_t$:
    
    \begin{equation}
    \mathbb{E}\left[\frac{1}{T}\sum_{t=1}^T (1- \dfrac{s_{\lambda_t}(u_f(x_t))}{\max_{x \in X} s_{\lambda_t} (u_f(x))})\right] - \mathbb{E}\left[1- \dfrac{\max_{x \in D} s_{\lambda_t}(u_f(x_t))}{\max_{x \in X} s_{\lambda_t} (u_f(x))}\right] \geq -J\mathbb{E}W_1(p,\hat{p})
    \end{equation}
    
    Using the fact that $\mathbb{E}[\max_{x\in\mathcal{X}} s_\lambda(u_f(x))] = \mathbb{E}[\max_{x\in\mathcal{X}} s_{\lambda_t}(u_f(x))]$, we obtain:
    
    \begin{equation}
    \mathbb{E}R_B(T) \leq \frac{1}{T}\mathbb{E}R_C(T) + J\mathbb{E}W_1(p,\hat{p})
    \end{equation}
    
    By Theorem 15 and results from \citet{paria_flexible_2019}, the first term converges to zero at rate $O^*(T^{-1/2})^4$ and the second term converges to zero at rate $O^*(T^{-1/D})$ for $D \geq 2$ \citep{canas_learning_2012}, where $D$ is the dimension of $\Lambda$. As a result, $\mathbb{E}R_B(T) \to 0$ as $T \to \infty$. Since $R_B(T)$ is the inverse of $U_B(T)$, from Definition 9, $\mathbb{E}U_B(T) \to 1$ as $T \to \infty$.
\end{proof}

\begin{lemma}
    For a fixed $\pmb{\lambda}\in\Lambda$, the augmented Chebyshev scalarization function $s_{\pmb{\lambda}}(y) = -\max_{\ell\in[L]}\{\pmb{\lambda}_{\ell}\lvert y_{\ell} - z^{*}_{\ell}\rvert\} - \gamma\sum_{\ell=1}^L \lvert y_{\ell} - z^{*}_{\ell}\rvert$, as described in Section~\ref{app:subsubsec:baseline_methodologies}, satisfies the assumption in Definition~\ref{app:def:lipschitz_input}.
\end{lemma}
\begin{proof}
    Let $\pmb{\lambda}\in\Lambda$. Recall $\Lambda := \Delta^L$, thus is bounded. Furthermore, $\text{Im}(u_f)$ is bounded. First, we demonstrate that $s_{\pmb{\lambda}}(y)$ is Lipschitz w.r.t. $y$. For $\ell\in[L]$,
    \begin{align*}
        \frac{\partial s_{\pmb{\lambda}}(y)}{\partial y_{\ell}} &= 
        \begin{cases}
            -\lambda_{\ell^{*}}, & y_{\ell^{*}} > z^{*}_{\ell^{*}}, \ell = \ell^{*} \\
            \lambda_{\ell^{*}}, & y_{\ell^{*}} < z^{*}_{\ell^{*}}, \ell = \ell^{*} \\
            0, & \ell \neq \ell^{*}
        \end{cases} \Bigg\vert_{\ell^{*}:=\argmax_{\ell}\{\pmb{\lambda}_{\ell}\lvert y_{\ell} - z^{*}_{\ell}\rvert\}} - \gamma
        \begin{cases}
            1 & y_{\ell} > z^{*}_{\ell} \\
            -1 & y_{\ell} < z^{*}_{\ell}
        \end{cases} \\
        \implies \Big\lVert \frac{\partial s_{\pmb{\lambda}}(y)}{\partial y} \Big\rVert &\leq \lambda_{\ell^{*}}\Bigg\vert_{\ell^{*}:=\argmax_{\ell}\{\pmb{\lambda}_{\ell}\lvert y_{\ell} - z^{*}_{\ell}\rvert\}} + L
    \end{align*}
    As the partial derivative is bounded w.r.t. $y$, by Mean Value Theorem (MVT) the scalarization function is Lipschitz w.r.t. $y$. Thus, there exists some constant $C_{\pmb{\lambda}}$, such that for all $y_1, y_2 \in \text{Im}(u_f)$,
    \begin{equation*}
        \vert s_{\pmb{\lambda}}(y_1) - s_{\pmb{\lambda}}(y_2)\vert \leq C_{\pmb{\lambda}} \Vert y_1 - y_2 \Vert
    \end{equation*}
    Note that for fixed $\pmb{\lambda}$, $\max_{y\in \text{Im}(u_f)} s_{\pmb{\lambda}}(y)$ is a constant. Using that fact, and the equation above, Definition~\ref{app:def:lipschitz_input} follows.
\end{proof}

\begin{lemma}
    The augmented Chebyshev scalarization function $s_{\pmb{\lambda}}(y) = -\max_{\ell\in[L]}\{\pmb{\lambda}_{\ell}\lvert y_{\ell} - z^{*}_{\ell}\rvert\} - \gamma\sum_{\ell=1}^L \lvert y_{\ell} - z^{*}_{\ell}\rvert$, as described in Section~\ref{app:subsubsec:baseline_methodologies}, satisfies the assumption in Definition~\ref{app:def:lipschitz_lambda} when $y\in S$ such that $S$ bounded and $S\subseteq \text{Im}(u_f)$.
\end{lemma}
\begin{proof}
    We assume the following: $y \in S \backslash \{-\infty\}$ and $\exists \ell \in [L]$ s.t. $y_\ell = M$, where $M$ is the upper bound of the \mf. %
    First, we apply Danskin's Theorem \citep{danskin_theory_1966} to $\max_{y \in \text{Im}(u_f)}s_{\pmb{\lambda}}(y)$. If $s_{\pmb{\lambda}}(y)$ is convex in $\pmb{\lambda}$ for all $y \in \text{Im}(u_f)$, and at a given $\lambda_o \in \Lambda$, $\exists$ a unique maximizer $y_o \in \text{Im}(u_f)$ then $s_{\pmb{\lambda}}(y)$ is differentiable w.r.t. $\pmb{\lambda}$ at $\pmb{\lambda}_o$ with derivative $\dfrac{\partial s_{\pmb{\lambda}}(y_o)}{\partial \pmb{\lambda}} \Bigg\vert_{\pmb{\lambda}=\pmb{\lambda_o}}$.

    As a result, if $y_o$ is the unique maximizer at $\pmb{\lambda_o}$, then
    \begin{align*}
    \dfrac{\partial s_{\pmb{\lambda}}(y_o)}{\partial \pmb{\lambda}} \Bigg\vert_{\pmb{\lambda}=\pmb{\lambda_o}} = (\lvert y_{o,\ell^*} - z_{\ell^*}^* \rvert) \Bigg\vert_{\ell^*=\argmax_{\ell} [\lambda_\ell \lvert y'_\ell - z_\ell^* \rvert]}
    \end{align*}
    
    Then, we demonstrate that $s_{\pmb{\lambda}}(y')/\max_{y \in \text{Im}(u_f)}s_{\pmb{\lambda}}(y)$ is Lipschitz w.r.t. $\pmb{\lambda}$. To conserve space, we denote $\max_{y \in \text{Im}(u_f)}s_{\pmb{\lambda}}(y)$ with $A$.
    \begin{flalign*}
        \frac{\partial}{\partial \pmb{\lambda}} \left[ \dfrac{s_{\pmb{\lambda}}(y')}{A} \right] \Bigg\vert_{\pmb{\lambda}=\pmb{\lambda_o}} &= \frac{ \left[ (A)(\lvert y'_{\ell^*}-z^*_{\ell^*}\rvert) \Bigg\vert_{\ell^*=\argmax_{\ell} [\lambda_\ell \lvert y'_\ell - z_\ell^* \rvert]} \right] }{\left[ A \right]^2} - \\ \frac{\left[ (s_{\pmb{\lambda}}(y'))(\lvert y_{o,\ell^*} - z_{\ell^*}^* \rvert) \Bigg\vert_{\ell^*=\argmax_{\ell} [\lambda_\ell \lvert y'_\ell - z_\ell^* \rvert]} \right]}{{\left[ A \right]^2}}
        \\
        &= 
        \underbrace{\frac{ \left[ (\lvert y'_{\ell^*}-z^*_{\ell^*}\rvert) \Bigg\vert_{\ell^*=\argmax_{\ell} [\lambda_\ell \lvert y'_\ell - z_\ell^* \rvert]} \right] }{\left[ A \right]}}_{C1} - \\ \underbrace{\frac{\left[ (s_{\pmb{\lambda}}(y'))(\lvert y_{o,\ell^*} - z_{\ell^*}^* \rvert) \Bigg\vert_{\ell^*=\argmax_{\ell} [\lambda_\ell \lvert y'_\ell - z_\ell^* \rvert]} \right]}{{\left[ A \right]^2}}}_{C2}
    \end{flalign*}

    The numerator for term C2 is bounded since $s_{\pmb{\lambda}}$ is in a bounded space. The numerator for C1 is upper bounded since $y$, our \mf, is an upper-bounded function. 
    
    The denominator values for C1 and C2 are both lower-bounded: since $\pmb{\lambda}$ lies on the probability simplex, there must be a $\lambda_{\ell^{\dagger}} \geq 1/L$ for some $\ell^{\dagger} \in L$, by the pigeonhole principle. Since we are maximizing over $y \in \text{Im}(u_f)$, there must be a $y^{\dagger}$ at the upper bound of $\text{Im}(u_f)$, which we denote as $M$. As a result, $\max_{y \in \text{Im}(u_f)}s_{\pmb{\lambda}}(y) \geq M/L$. Since the denominators of both C1 and C2 are lower-bounded, the overall partial derivative is bounded w.r.t. $\pmb{\lambda}$.

    Since the overall partial derivative is bounded w.r.t. $\pmb{\lambda}$, by MVT $s_{\pmb{\lambda}}(y')/\max_{y \in \text{Im}(u_f)}s_{\pmb{\lambda}}(y)$ is Lipschitz w.r.t. $\pmb{\lambda}$. Thus, there exists some constant $J$ such that for all $\pmb{\lambda}_1, \pmb{\lambda}_2 \in \Lambda$
    \begin{align*}
        \Big\vert\dfrac{s_{\pmb{\lambda}_1}(y)}{\max_{y' \in \text{Im}(u_f)} s_{\pmb{\lambda}_1}(y')} - \dfrac{s_{\pmb{\lambda}_2}(y)}{\max_{y' \in \text{Im}(u_f)} s_{\pmb{\lambda}_2}(y')} \Big\vert \leq J \|\pmb{\lambda}_1 - \pmb{\lambda}_2\|_1
    \end{align*}
\end{proof}

\section{ADDITIONAL EXPERIMENTAL RESULTS}\label{additional-experimental-results}

\subsection{Concrete \mf: Soft-Hard Functions}\label{sec:appendix-shfs}

Effectively leveraging expert knowledge is a central challenge in multi-objective optimization, particularly when function evaluations are costly. This hurdle is especially acute in high-stakes domains like healthcare, where practitioners must balance treatment efficacy against harmful side-effects. Consider brachytherapy for cancer treatment~\citep{deufel_pnav_2020}. Clinicians must design a radiation plan that maximizes dose to the tumor while minimizing exposure to nearby healthy organs. Each potential plan represents a complex trade-off, and devising or evaluating plans is time-consuming, often occurring under clinical pressure \citep{belanger_gpu-based_2019, cui_multi-criteria_2018, van_der_meer_bi-objective_2020}. Crucially, clinicians and other domain experts approach such tasks with practical knowledge, often conceptualizing their objectives not just as values to be purely minimized or maximized, but in terms of both \textit{desired targets} (which we term "soft bounds" or aims) and \textit{strict, non-negotiable boundaries} (which we term "hard bounds" or limits). For instance, a clinician might \textit{aim} for >95\% tumor coverage (a soft bound, indicating a desirable region) but require \textit{at least} 90\% (a hard bound, defining a feasibility constraint), while \textit{aiming} for <513 centigrays (cGy) bladder dose (soft bound) with a strict \textit{limit} of 601 cGy (hard bound) \citep{viswanathan_american_2012}. This natural way of specifying preferences—defining both acceptable ranges and ideal targets for each objective—is prevalent yet not fully leveraged by existing MOO frameworks for specifying tradeoff regions \citep{paria_flexible_2019, abdolshah_multi-objective_2019, zuluaga_e-pal_2016, while_fast_2012, suzuki_multi-objective_2020}. Examples of real-world applications where such priors already exist include brachytherapy, diabetes, and vitals monitoring \citep{grosman_zone_2010, cairoli_model_2019, cui_multi-criteria_2018}. Our explicit usage of \mfs into the optimization process aims to fill this gap by operationalizing these prevalent soft and hard bounds using \textit{soft-hard functions} (SHFs).

Let $\varphi$ denote a generic objective function (e.g., some $f_\ell$) with its associated soft bound $\alpha_S$ and hard bound $\alpha_H$. We define its \hsf as follows:

\begin{equation}
    u_{\varphi}(x) = \begin{cases}
        1 + 2 \beta \times (\tilde{\alpha}_\tau - \tilde{\alpha}_S) & \varphi(x) \geq \alpha_{\tau} \\
        1 + 2 \beta \times (\tilde{\varphi}(x) - \tilde{\alpha}_S) & \alpha_{S} < \varphi(x) < \alpha_{\tau} \\
        \text{1} & \varphi(x) = \alpha_{S} \\
        \text{2 } \times \tilde{\varphi}(x) & \alpha_{H} < \varphi(x) < \alpha_{S} \\
        0 & \varphi(x) = \alpha_{H} \\
        -\infty & \varphi(x) < \alpha_{H} \\
    \end{cases}
\end{equation}

where $\tilde{\varphi}(x)$ and $\tilde{\alpha}$ are the soft-hard bound normalized values \footnote{Normalization, for value $z$, is performed according to the soft and hard bounds, $\alpha_{S}$ and $\alpha_{H}$, respectively, using: $\tilde{z} = ((z-\alpha_{H}) \division (\alpha_{S}-\alpha_{H})) * 0.5$}, $\alpha_{\tau}$, the saturation point, determines where the utility values begin to saturate (as previously described, to prevent exploding utility values) \footnote{In our experiments, we determine $\alpha_{\tau}$ to be $\alpha_{H}$+$\zeta(\alpha_{S}-\alpha_{H})$, for $\zeta > 1.0$. We set $\zeta=2.0$.}, and $\beta$ $\in [0, 1]$ determines the fraction of the original rate of utility, in $[\alpha_{H},\alpha_{S}]$, obtained within $[\alpha_{S}, \alpha_{\tau}]$. 

\subsection{Other Monotonic Utility Functions}\label{app:other-monotonic-functions}

We conducted several experiments to illustrate our method’s flexibility in adopting alternative monotonic functions. We provide brief motivation from other disciplines including additive manufacturing, energy systems and consumer behavior (business-related), and video streaming quality and evaluate on Branin-Currin.

\noindent \textbf{Application Domain(s): Additive Manufacturing (Engineering \& Science)}

\noindent \textbf{Motivation:}
\citet{haghanifar_discovering_2020} studied a numerical simulation to design an additive manufacturing strategy for minimizing the reflection of light at different angles. \citet{malkomes_beyond_2021} proposed to observe a diverse set of samples satisfying a set of hard boundary conditions.

\noindent \textbf{Monotonic Utility Function:}
Hard-bounds only function. In this setting, the soft bound is effectively vacuous and the utility function is upper bounded at a user-specified value since we require upper-bounded utility functions.

\vspace{1em}

\noindent \textbf{Application Domain(s): Energy Systems and Consumer Behavior (Business Studies)}

\noindent \textbf{Motivation:}
In the optimization of real-time pricing for energy systems, monotonic utility functions are often used to model consumer satisfaction, to minimize electricity costs, etc. \citep{zhang_refined_2022, samadi_optimal_2010}. \citet{zhang_refined_2022} and \cite{samadi_optimal_2010} proposed using sigmoid utility functions and functions which resemble soft-hard bounds (to denote preferred levels of utility).

\noindent \textbf{Monotonic Utility Function:}
Sigmoid utility functions and soft-bounds only functions may satisfy the assumptions we have in our paper. We implement the sigmoid utility function to be bounded between 0 and 1 while also allowing for the decision-maker to specify where the inflection/midpoint of the sigmoid function should be (intuitively representing the region of highest rate of increase in utility). The sigmoid function also smoothly models diminishing returns which is beneficial for the above domains.

\vspace{1em}

\noindent \textbf{Additional Potential Application Domains: Satisfaction Levels in Video Streaming Quality, Forest Management}

\noindent \textbf{Motivation:}
In the domain of video streaming services, \citet{fazio_advanced_2021} utilizes multi-step utility functions to model user satisfaction, where each discrete step corresponds to a specific video encoding layer and its associated utility. Similarly, in forest management, step utility functions are employed to balance the trade-off between maximizing economic profit from harvesting and minimizing ecological impact \citep{radulescu_multi-objective_2019}.

\noindent \textbf{Other Potential Monotonic Utility Functions:} Piecewise step functions, polynomial splines

We implemented the above utility function and observed the following end-to-end results (over 6 trials) for Branin-Currin in Figure \ref{fig:alternative_mfs_results}.

\begin{figure}[h]
\begin{center}
\includegraphics[width=0.7\textwidth]{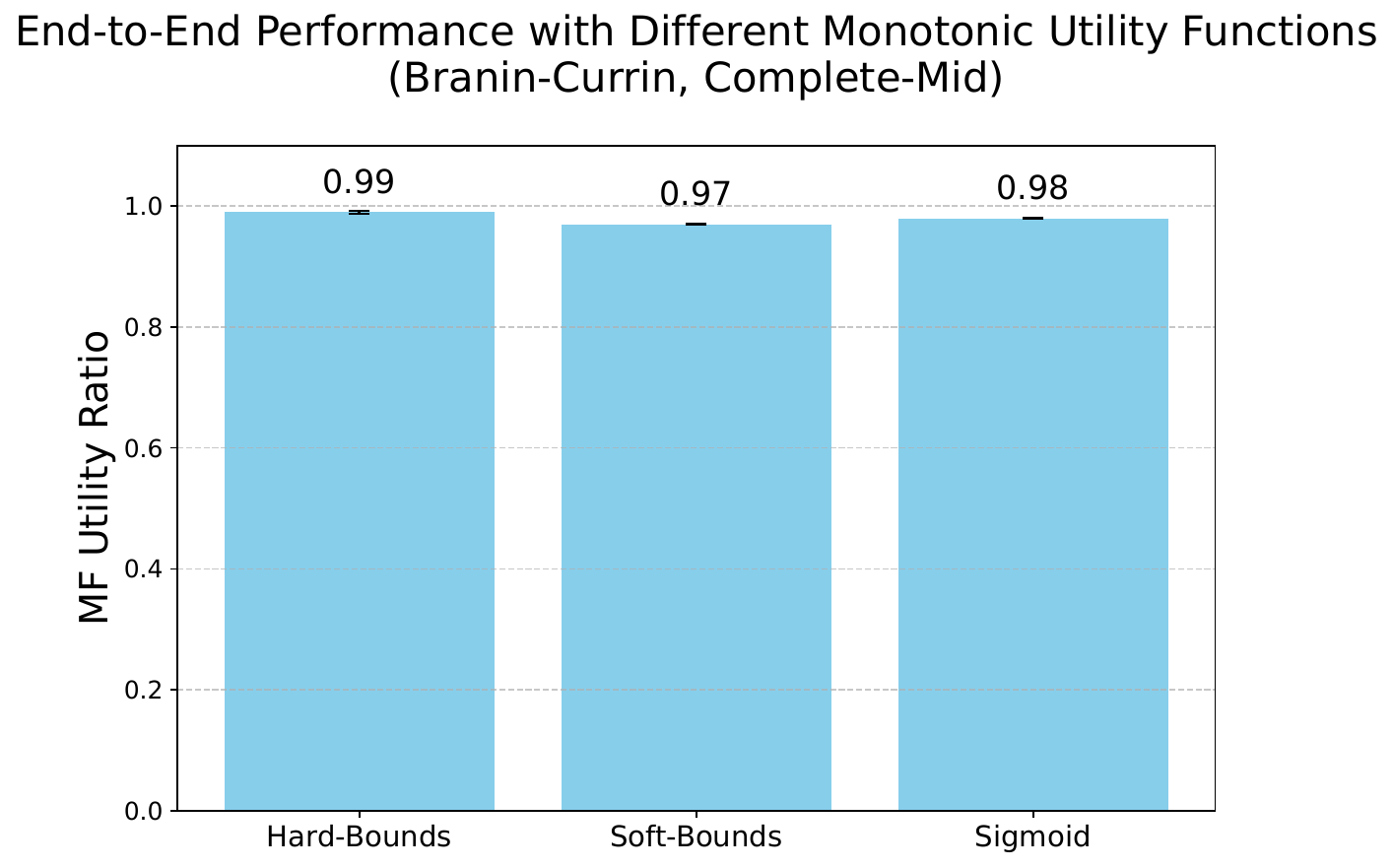}
\end{center}
\caption{End-to-End MF Utility Ratio at 5 Points for different forms of Monotonic Utility Functions.}
\label{fig:alternative_mfs_results}
\end{figure}

From a qualitative standpoint, the final subset from each of the above utility functions often exhibited higher density around the high-utility regions. The high MF utility ratio values quantitatively indicate that as well.

\subsection{Importance of Two-Step Approach}\label{app:two-step-benefits}

Our central challenge is solving the max-min objective (Equation \ref{submodular-formulation}) to find a small set of points robust to unknown user weights. We decompose this into two steps because current literature lacks a principled single-step approach for this specific goal.

By splitting the problem, we maintain strong theoretical guarantees for both phases:

Dense Sampling (Step 1): We prove convergence to the high-utility region (Appendix \ref{app:definitions_theorems}) using Bayesian Optimization, effectively discretizing the continuous search space.

Sparsification (Step 2): We leverage submodularity to provide optimality guarantees on the final, smaller set (Section \ref{sec:step-2-method}).

A monolithic approach would sacrifice these guarantees. Relying solely on Step 1 risks overwhelming the DM with redundant points, while Step 2 requires a high-quality discrete candidate set to function. Our decomposition ensures we achieve both dense coverage (Step 1) and optimal sparsity (Step 2) with formal backing for each stage.

To demonstrate this from an empirical perspective, we implemented and ran several one-step approaches: (1) a simulated annealing approach where we directly optimize for Equation \ref{submodular-formulation} of our paper, rather than split it into two separate steps, and (2) 
$\epsilon$-PAL \citep{zuluaga_e-pal_2016}, which is a one-step approach that has a component that acts as a form of regularization on the sparsity of the solutions. We display end-to-end results for Branin-Currin in Figure \ref{fig:one_step_results}.

The one-step simulated annealing approach typically takes many more function evaluations to converge compared to MoSH, especially since MoSH is designed for expensive function evaluations (a critical assumption in our setting, due to the time-critical nature of applications such as Brachytherapy cancer treatment). As a result, a budget of 50 function evaluations for simulated annealing is not nearly enough for it to obtain points well-distributed in the region defined by the soft-hard bounds. This is especially evident for the real-world Brachytherapy application, which is of much higher dimensionality and would require significantly more function evaluations to converge for simulated annealing.

$\epsilon$-PAL (one-step), while designed for expensive black-box functions, does not take into consideration any expert-defined priors (e.g. our soft-hard bounds). As a result, 
$\epsilon$-PAL spends a lot of time obtaining points along the entire Pareto frontier, as opposed to the region the decision-maker truly cares about. As a result, it is difficult for the decision-maker to make high-utility decisions in their preferred tradeoff regions.

\begin{figure}[h]
\begin{center}
\includegraphics[width=0.5\textwidth]{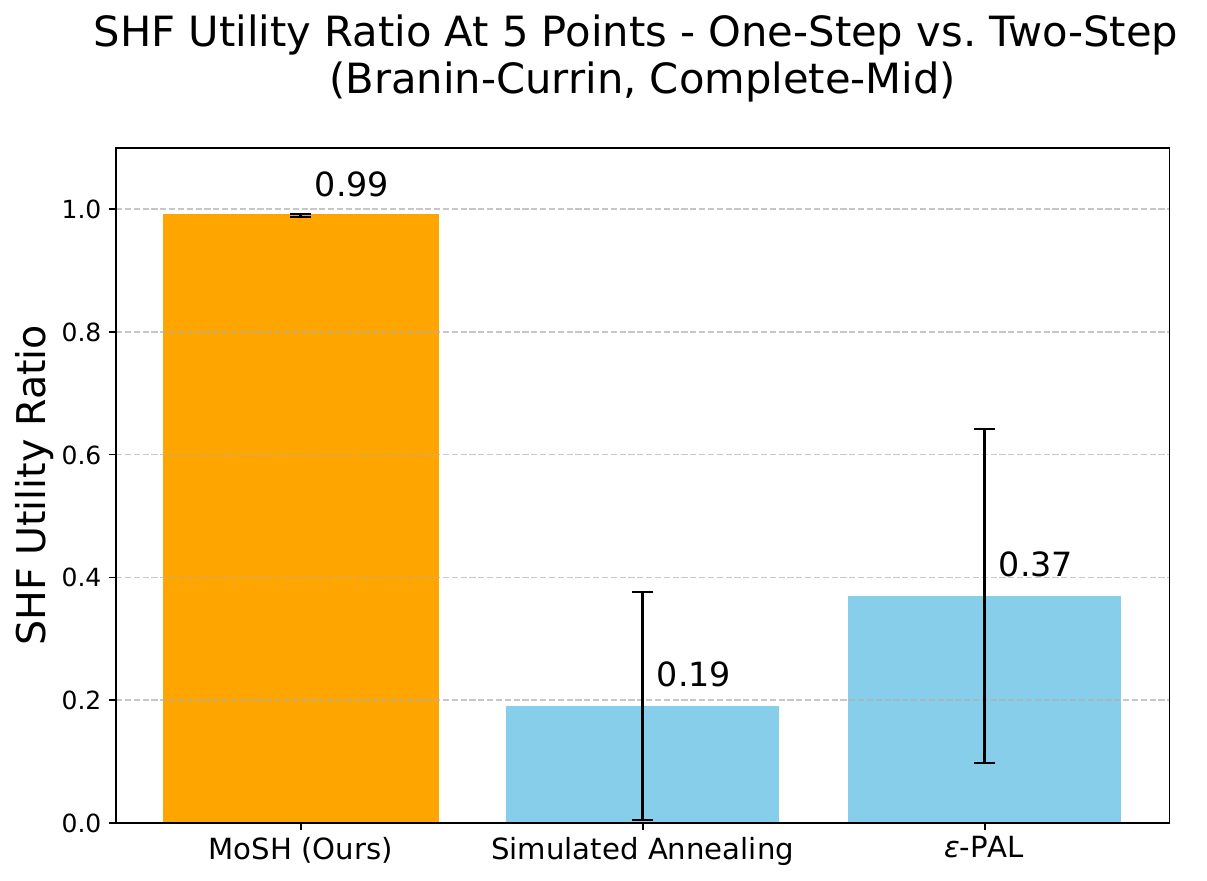}
\end{center}
\caption{End-to-End SHF Utility Ratio at 5 points results comparing our two-step approach (MoSH) against two other one-step approaches.}
\label{fig:one_step_results}
\end{figure}

\subsection{Alternative Acquisition Functions}\label{app:alternative-acq-functions}

To illustrate the performance of MoSH across acquisition functions, we conducted several ablations comparing UCB to Thompson Sampling on the Branin-Currin, Four Bar Truss, and Brachytherapy applications. We present the end-to-end results in Figure \ref{fig:acq_func_comparisons} (over 6 trials):

\begin{figure}[h]
\begin{center}
\includegraphics[width=0.5\textwidth]{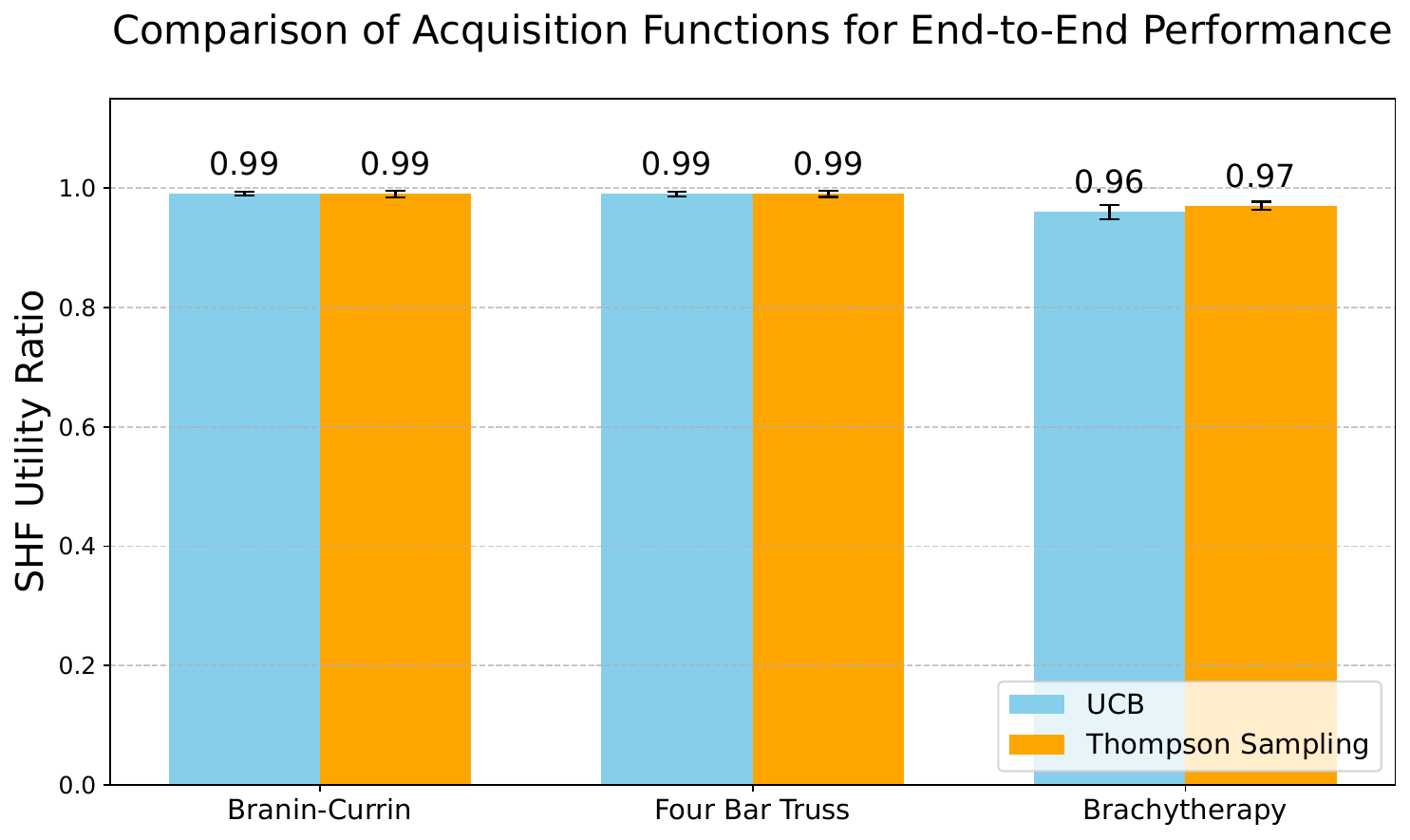}
\end{center}
\caption{End-to-End SHF Utility Ratio results at 5 points comparing the Upper Confidence Bound (UCB) and Thompson Sampling (TS) acquisition functions.}
\label{fig:acq_func_comparisons}
\end{figure}

We did not observe statistically significant differences in the end-to-end results between UCB and Thompson Sampling. Both methods converged similarly to the high-utility regions defined by the Soft-Hard bounds. Given the empirical similarity in performance, we believe it could be helpful to have heuristics based on secondary practical considerations rather than raw performance. Below are several high-level ones (among others):

UCB: When explicit control over the exploration-exploitation trade-off is desired or when deterministic reproducibility is desired \citep{russo_learning_2014}.

Thompson Sampling: To simplify the design and implementation process, especially when UCB algorithms are computationally challenging \citep{russo_learning_2014}.

\subsection{Computational Feasibility}\label{sec:computational-feasible-details}
Equation \ref{bayesian-formulation} of the dense sampling step uses a conversion from the worst-case to an average-case minimization problem. To observe the difficulty of the worst-case minimization problem and how the results may be affected, we performed several experiments which directly solve the following:

\begin{equation}\label{eq:comp-infeasible}
    \max_{D \subseteq X, |D| \leq k_D} \min_{\pmb{\lambda} \in \Lambda}\Bigg[\dfrac{\max_{x \in D} s_{\pmb{\lambda}}(u_f(x))}{\max_{x \in X} s_{\pmb{\lambda}}(u_f(x))}\Bigg]
\end{equation}

To conduct the experiments, we instantiate the \mfs with the soft-hard functions (\hsfs) from Section \ref{subsec:soft_hard_utility_functions} and Appendix \ref{sec:appendix-shfs}. When using \hsfs, we refer to the dense sampling step of our method as MoSH-Dense. 

We solve Equation \ref{eq:comp-infeasible} using a greedy algorithm over discretized input space $X$ and weight space $\Lambda$ and compare the results with both MoSH-Dense and its discretized variations.

\textbf{Discretization.}
To create the finite approximation of the continuous input space, we discretize the domain using a uniform grid. Given input space bounds $[\mathbf{l}, \mathbf{u}] \subset \mathbb{R}^d$ where $\mathbf{l} = [l_1, \ldots, l_d]$ and $\mathbf{u} = [u_1, \ldots, u_d]$ are the lower and upper bounds respectively, we construct a discretized grid as follows:

For each dimension $i \in [d]$, we create an evenly-spaced sequence:
\begin{equation}
    \mathcal{X}_i = \{l_i, l_i + \delta, l_i + 2\delta, \ldots, u_i\}
\end{equation}
where $\delta$ is the step size parameter controlling the granularity of the discretization.

The complete discretized input space $\mathcal{X}_\text{disc}$ is then constructed as the Cartesian product:
\begin{equation}
    \mathcal{X}_\text{disc} = \mathcal{X}_1 \times \mathcal{X}_2 \times \cdots \times \mathcal{X}_d
\end{equation}

This results in a finite grid of points where each point $\mathbf{x} \in \mathcal{X}_\text{disc}$ represents a candidate solution in the original continuous space. The total number of points in the discretized space is $\prod_{i=1}^d \left\lceil\frac{u_i - l_i}{\delta}\right\rceil$. The same discretization method is done for $\Lambda$, except, from this grid, we select only those points that lie on the probability simplex, i.e., whose components sum to 1. 

\textbf{Experiments and Discussions.}
We term the baseline algorithms as \textit{discrete-greedy-}$\delta$ and \textit{discrete-average-}$\delta$ and primarily experiment with the Branin-Currin objective function (described in Section \ref{sec:appendix-branin-currin}). We show numerical results in Figure \ref{discrete_greedy_step1_results}. In general, computationally, \textit{discrete-greedy-}$\delta$ and \textit{discrete-average-}$\delta$ take longer as $\delta \rightarrow 0$. Specifically, \textit{discrete-greedy-0.05} and \textit{discrete-greedy-0.10} take, on average, 20x and 3x longer (in seconds) than MoSH-Dense, respectively, to select the next sample $x_t$ at iteration $t$. This computational disparity is further exacerbated as the size of the input space dimensionality increases. However, as expected, as $\delta \rightarrow 0$, the metrics improve and move closer to those of MoSH-Dense. Finally, we note that, as described in Section \ref{sec:step-2-method}, the greedy algorithm may perform arbitrarily bad when solving Equation \ref{eq:comp-infeasible}. We leave for future work exploration of other algorithms which are designed for our setting described in Section \ref{dense-sampling}; notably, access to some noisy and expensive black-box function. As these results show, our formulation for MoSH-Dense leads to superior results while also being more computationally efficient and simple.

\begin{figure}[h]
\begin{center}
\includegraphics[width=0.7\textwidth]{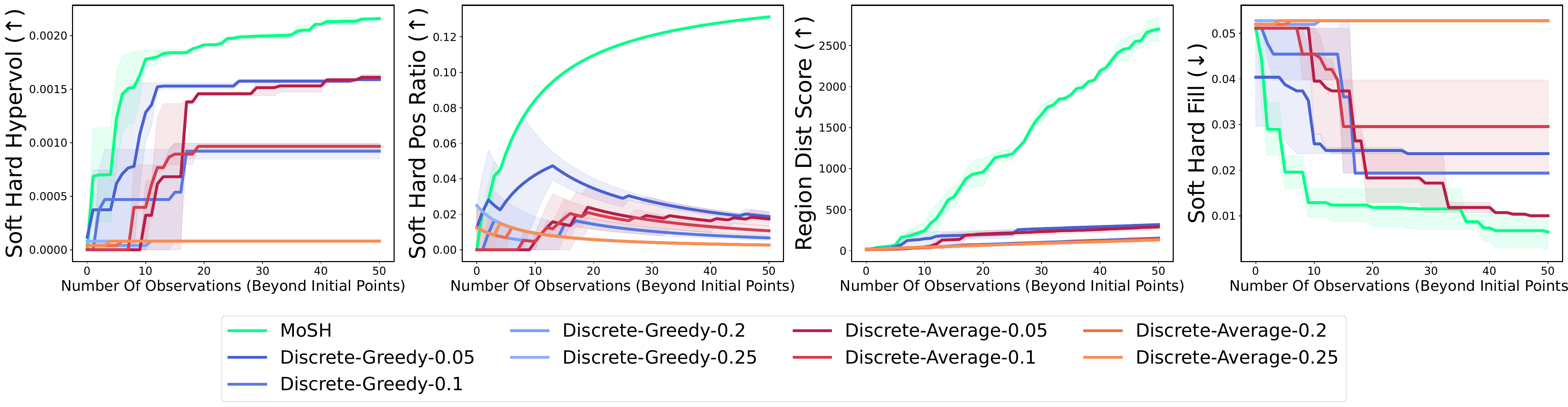}
\end{center}
\caption{Step (1) results with discrete greedy algorithms which explicitly solve for the max-min problem for MoSH-Dense (Equation \ref{eq:comp-infeasible}). The plots compare our proposed method, MoSH-Dense, against several variations of \textit{discrete-greedy-}$\delta$ and \textit{discrete-average-}$\delta$ using the metrics defined in Section \ref{performance-criteria}.}
\label{discrete_greedy_step1_results}
\end{figure}

\subsection{Synthetic Two-Objective Function: Branin-Currin}\label{sec:appendix-branin-currin}
We leverage the Branin-Currin synthetic two-objective optimization problem provided in the BoTorch framework \citep{balandat2020botorch}, which has a mapping of $[0,1]^2 \rightarrow \mathbb{R}^2$. To demonstrate our method's flexibility in accommodating various configurations, we sample from the following variations
: (1) \textit{complete-mid}, where the hard bounds cover the \textit{complete} Pareto frontier and the soft bounds are in the \textit{middle} of the hard region (2) \textit{complete-top}, (3) \textit{complete-bot}, (4) \textit{top-mid}, and (5) \textit{bot-mid}. 
Figures \ref{branin_currin_soft_metrics_complete_mid}, \ref{branin_currin_soft_metrics_complete_top}, \ref{branin_currin_soft_metrics_complete_bot}, \ref{branin_currin_soft_metrics_top_mid}, \ref{branin_currin_soft_metrics_bot_mid} illustrate the plots, for step 1, with the performance metrics defined in Section \ref{performance-criteria} (computed over 6 independent trials). Figure \ref{branin_currin_soft_metrics_pbemo_results} additionally compares against preference-based evolutionary multi-objective algorithm baselines (experimental details in Appendix \ref{app:experimental-setup-details}). We note that EHVI, although superior in some metrics, is extremely computationally demanding for higher dimensions. Overall, we observe that our algorithm generally matches or surpasses other baselines in all four metrics, with sampling a much higher density near the soft region. The figures for the other configurations display a similar pattern.

\textbf{Configurations (normalized to [0,1])} : $\{\alpha_{0,S}, \alpha_{0,H}, \alpha_{1,S}, \alpha_{1,H} \}$
\begin{enumerate}
    \item \textit{Complete-Mid}: \{0.988, 0.943, 0.856, 0.618\}
    \item \textit{Complete-Top}: \{0.969, 0.943, 0.935, 0.618\}
    \item \textit{Complete-Bot}: \{0.998, 0.943, 0.697, 0.618\}
    \item \textit{Top-Mid}: \{0.969, 0.940, 0.915, 0.856\}
    \item \textit{Bot-Mid}: \{0.996, 0.975, 0.737, 0.658\}
\end{enumerate}

\begin{figure}[h]
\begin{center}
\includegraphics[width=0.8\textwidth]{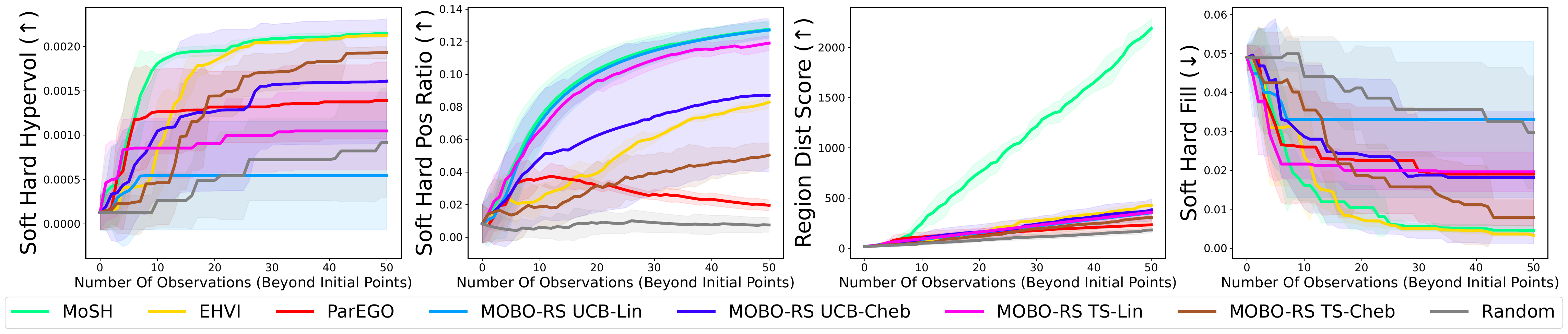}
\end{center}
\caption{\textit{Complete-Mid} configuration for the Branin-Currin synthetic two-objective function. Results are plotted using the metrics defined in Section \ref{performance-criteria}.}
\label{branin_currin_soft_metrics_complete_mid}
\end{figure}

\begin{figure}[h]
\begin{center}
\includegraphics[width=0.8\textwidth]{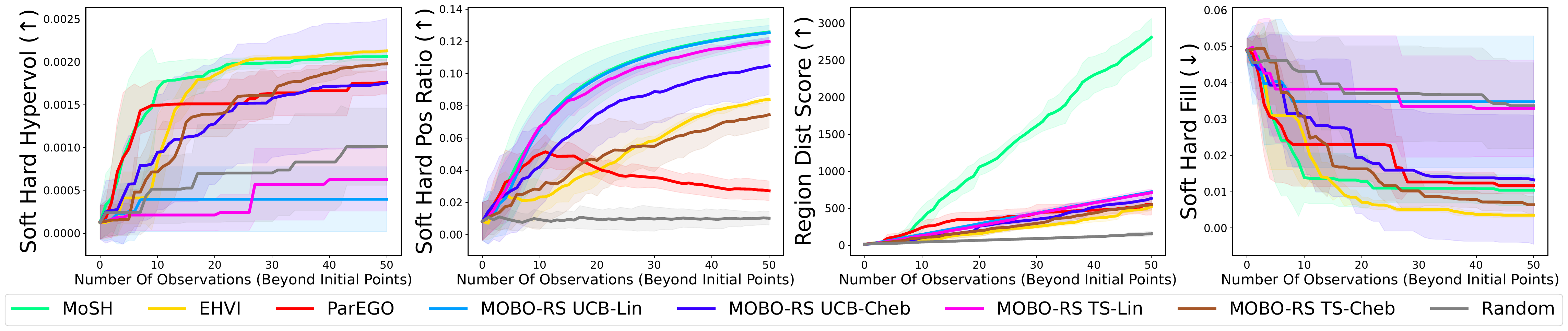}
\end{center}
\caption{\textit{Complete-Top} configuration for the Branin-Currin synthetic two-objective function. Results are plotted using the metrics defined in Section \ref{performance-criteria}.}
\label{branin_currin_soft_metrics_complete_top}
\end{figure}

\begin{figure}[h]
\begin{center}
\includegraphics[width=0.8\textwidth]{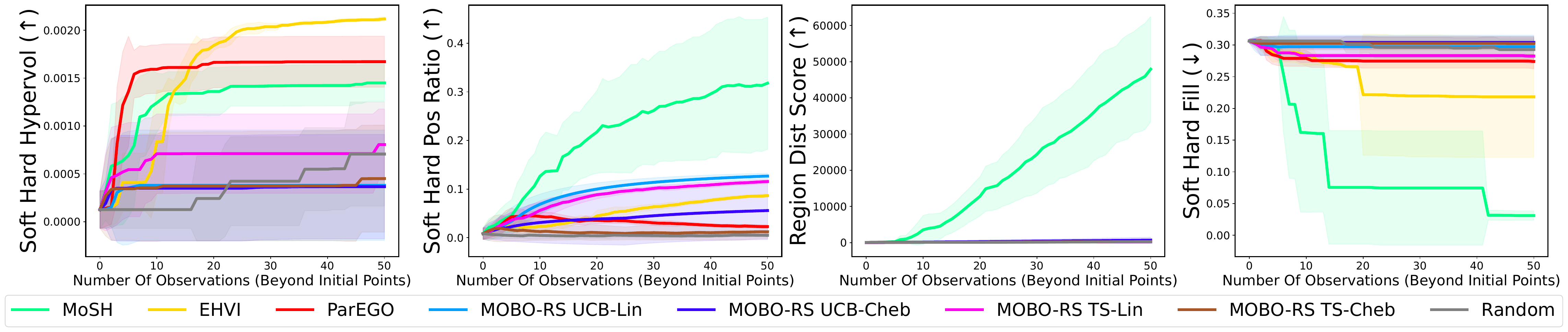}
\end{center}
\caption{\textit{Complete-Bot} configuration for the Branin-Currin synthetic two-objective function. Results are plotted using the metrics defined in Section \ref{performance-criteria}.}
\label{branin_currin_soft_metrics_complete_bot}
\end{figure}

\begin{figure}[h]
\begin{center}
\includegraphics[width=0.8\textwidth]{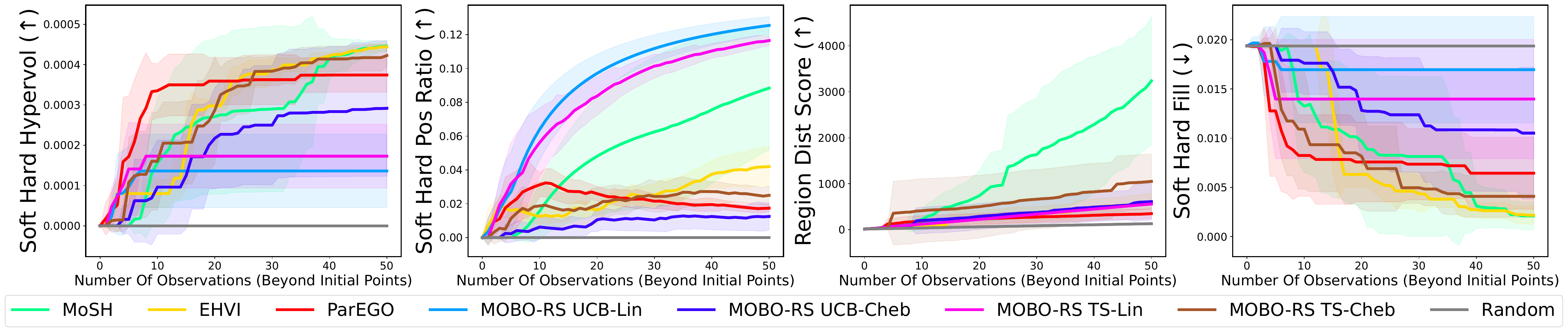}
\end{center}
\caption{\textit{Top-Mid} configuration for the Branin-Currin synthetic two-objective function. Results are plotted using the metrics defined in Section \ref{performance-criteria}.}
\label{branin_currin_soft_metrics_top_mid}
\end{figure}

\begin{figure}[h]
\begin{center}
\includegraphics[width=0.8\textwidth]{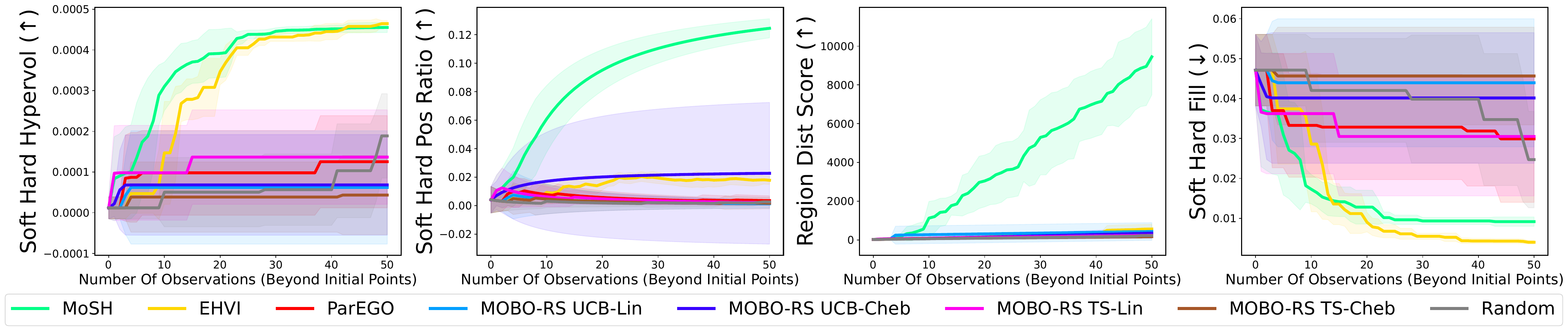}
\end{center}
\caption{\textit{Bot-Mid} configuration for the Branin-Currin synthetic two-objective function. Results are plotted using the metrics defined in Section \ref{performance-criteria}.}
\label{branin_currin_soft_metrics_bot_mid}
\end{figure}

\begin{figure}[h]
\begin{center}
\includegraphics[width=0.8\textwidth]{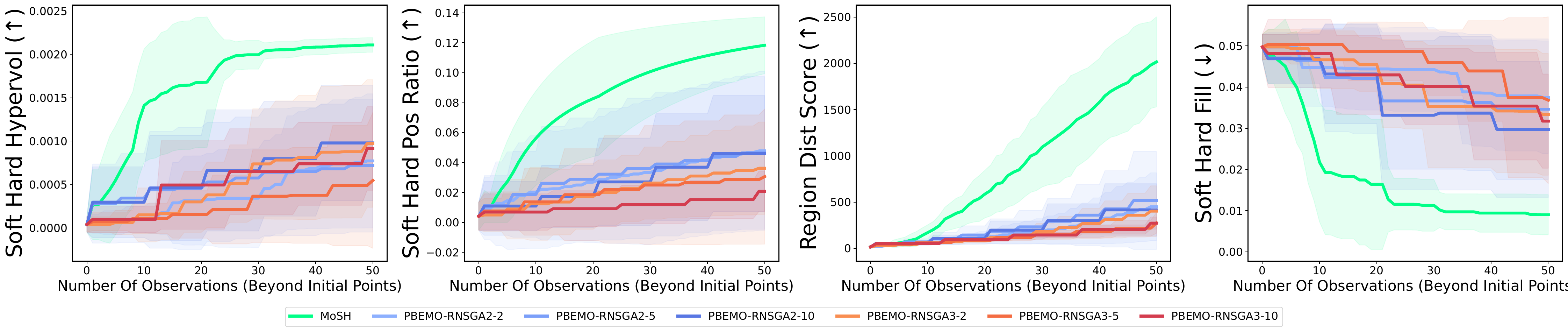}
\end{center}
\caption{\textit{Complete-Mid} configuration for the Branin-Currin synthetic two-objective function with preference-based evolutionary multi-objective (PBEMO) algorithm baselines. Results are plotted using the metrics defined in Section \ref{performance-criteria}.}
\label{branin_currin_soft_metrics_pbemo_results}
\end{figure}

\subsection{Engineering Design Problem: Four Bar Truss}\label{subsec:appendix-four-bar-truss}

We also evaluate on a MOO engineering design problem, Four Bar Truss, from REPROBLEM \citep{tanabe_easy--use_2020}, which consists of two objectives and four continuous decision variables \citep{cheng_generalized_1999}. The objectives of the problem are to minimize the structural volume and the joint displacement of the four bar truss. The four decision variables determine the length of the four bars, respectively. To demonstrate our method's flexibility, we sample from the following variations: (1) \textit{narrow-mid}, (2) \textit{narrow-bot}, (3) \textit{narrow-top}, (4) \textit{bot-mid}, and (5) \textit{top-mid} (over 6 independent trials). Step 1 results are shown in Figures \ref{four_bar_truss_soft_metrics_narrow_mid}, \ref{four_bar_truss_soft_metrics_narrow_bot}, \ref{four_bar_truss_soft_metrics_narrow_top}, \ref{four_bar_truss_soft_metrics_bot_mid}, and \ref{four_bar_truss_soft_metrics_top_mid}. Figure \ref{four_bar_truss_soft_metrics_pbemo_results} additionally compares against preference-based evolutionary multi-objective algorithm baselines (experimental details in Appendix \ref{app:experimental-setup-details}). Based on our experiments, our algorithm consistently matches or surpasses other baselines in all metrics, while sampling at a clearly higher density near the soft region. 

\textbf{Configurations (normalized to [0,1])} : $\{\alpha_{0,S}, \alpha_{0,H}, \alpha_{1,S}, \alpha_{1,H} \}$
\begin{enumerate}
    \item \textit{Narrow-Mid}: \{0.62, 0.45, 0.72, 0.55\}
    \item \textit{Narrow-Bot}: \{0.70, 0.45, 0.65, 0.55\}
    \item \textit{Narrow-Top}: \{0.55, 0.45, 0.78, 0.55\}
    \item \textit{Bot-Mid}: \{0.86, 0.70, 0.48, 0.25\}
    \item \textit{Top-Mid}: \{0.43, 0.20, 0.85, 0.70\}
\end{enumerate}

\begin{figure}[h]
\begin{center}
\includegraphics[width=0.8\textwidth]{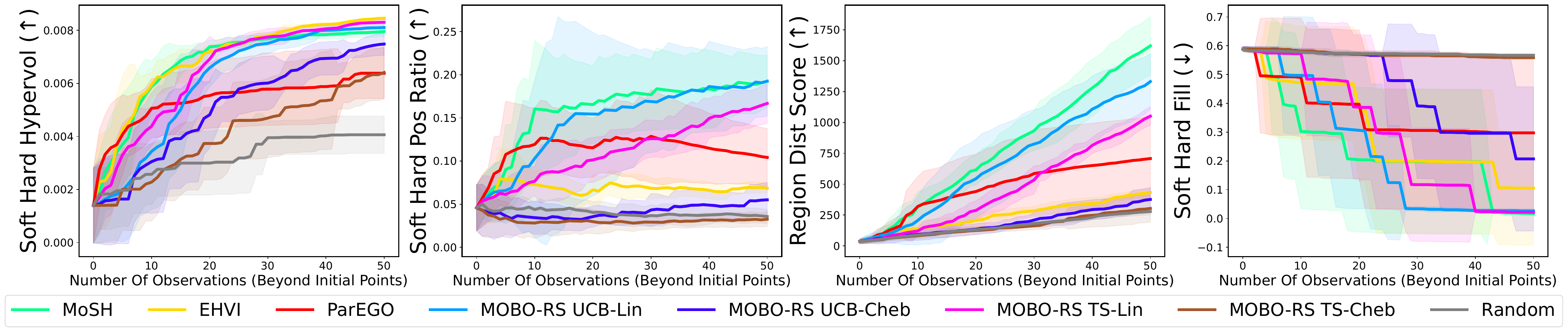}
\end{center}
\caption{\textit{Narrow-Mid} configuration for the Four Bar Truss engineering design two-objective function. Results are plotted using the metrics defined in Section \ref{performance-criteria}.}
\label{four_bar_truss_soft_metrics_narrow_mid}
\end{figure}

\begin{figure}[h]
\begin{center}
\includegraphics[width=0.8\textwidth]{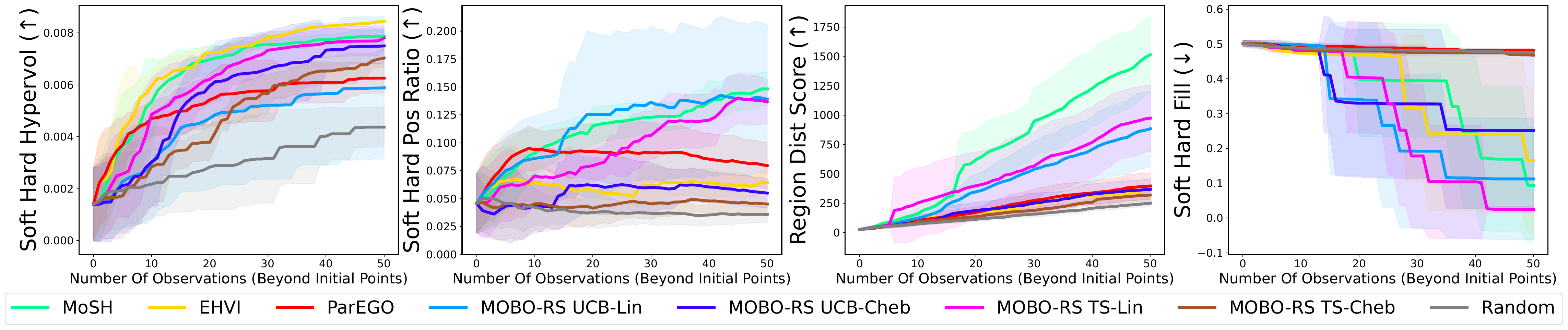}
\end{center}
\caption{\textit{Narrow-Bot} configuration for the Four Bar Truss engineering design two-objective function. Results are plotted using the metrics defined in Section \ref{performance-criteria}.}
\label{four_bar_truss_soft_metrics_narrow_bot}
\end{figure}

\begin{figure}[h]
\begin{center}
\includegraphics[width=0.8\textwidth]{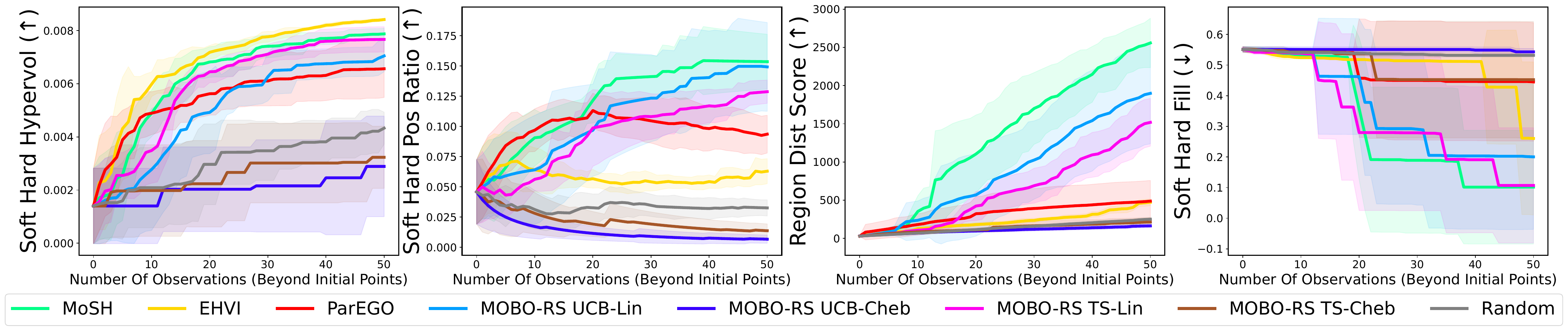}
\end{center}
\caption{\textit{Narrow-Top} configuration for the Four Bar Truss engineering design two-objective function. Results are plotted using the metrics defined in Section \ref{performance-criteria}.}
\label{four_bar_truss_soft_metrics_narrow_top}
\end{figure}

\begin{figure}[h]
\begin{center}
\includegraphics[width=0.8\textwidth]{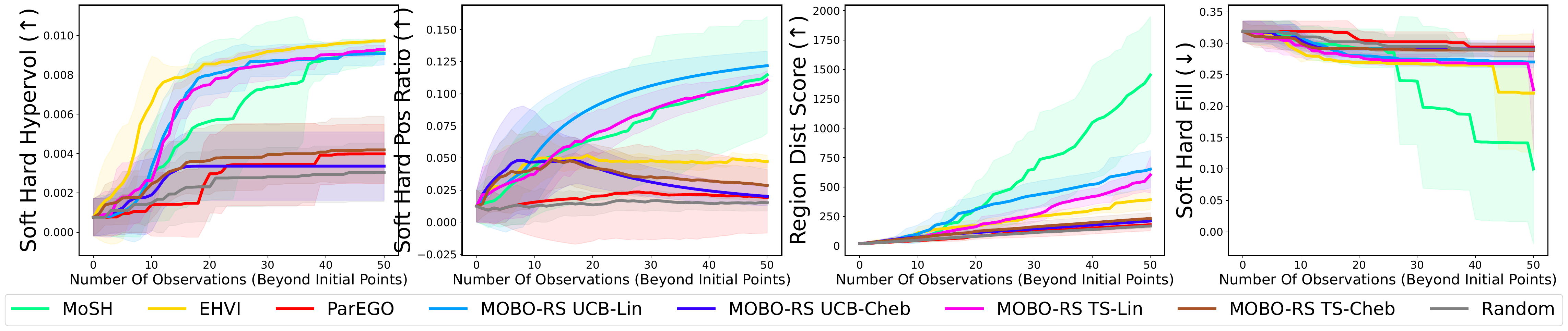}
\end{center}
\caption{\textit{Bot-Mid} configuration for the Four Bar Truss engineering design two-objective function. Results are plotted using the metrics defined in Section \ref{performance-criteria}.}
\label{four_bar_truss_soft_metrics_bot_mid}
\end{figure}

\begin{figure}[h]
\begin{center}
\includegraphics[width=0.8\textwidth]{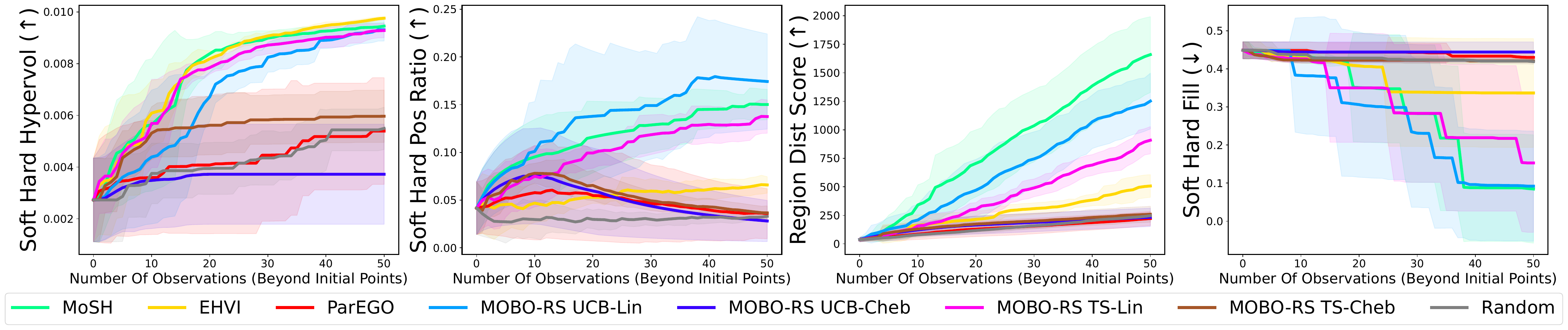}
\end{center}
\caption{\textit{Top-Mid} configuration for the Four Bar Truss engineering design two-objective function. Results are plotted using the metrics defined in Section \ref{performance-criteria}.}
\label{four_bar_truss_soft_metrics_top_mid}
\end{figure}

\begin{figure}[h]
\begin{center}
\includegraphics[width=0.8\textwidth]{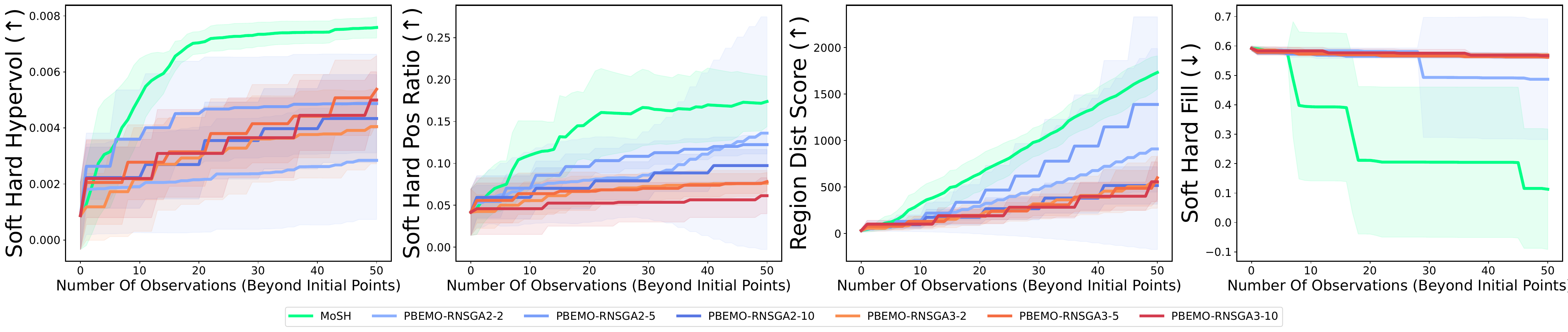}
\end{center}
\caption{\textit{Narrow-Mid} configuration for the Four Bar Truss engineering design two-objective function with preference-based evolutionary multi-objective (PBEMO) algorithm baselines. Results are plotted using the metrics defined in Section \ref{performance-criteria}.}
\label{four_bar_truss_soft_metrics_pbemo_results}
\end{figure}

\subsection{Real Clinical Case: Cervical Cancer Brachytherapy Treatment Planning}\label{appendix:brachytherapy-details}

We had consent for and leveraged a real clinical case which had previously been treated in the clinic. As a result, we had access to the fully anonymized patient CT data (in the form of DICOM files) and the final treatment plan which had been administered to the patient. For our experiments, we obtained a single treatment plan via a linear program formulated as an epsilon-constraint method \citep{deufel_pnav_2020}. The three input parameters to that linear program, which implicitly control the weights of the different objectives (radiation dosage to the bladder, rectum, bowel, and cancer tumor), were employed as the decision variables.

With MoSH, we are able to obtain an informative and concise subset of the treatment plans with metrics which align with the clinicians' priors . As shown in Figure \ref{overall_pipeline}, we are even able to obtain treatment plans which Pareto dominates those plans found from manual clinician planning, which typically takes significantly more time ($\sim$30min) \citep{moren_optimization_2021}. Furthermore, the soft and hard bound values we employed for the experiment are exactly the values that the clinician we worked with usually thinks about when performing treatment planning, so our setting closely mimics the real decision-making setting.

The configuration used was: \{$\alpha_{0,S}=$0.95, $\alpha_{0,H}=$0.90, $\alpha_{1,S}=$513, $\alpha_{1,H}=$601, $\alpha_{2,S}=$352, $\alpha_{2,H}=$464, $\alpha_{3,S}=$411, $\alpha_{3,S}=$464\}, where the objectives correspond to $PTV_{V700}$, $Bladder_{D2cc}$, $Rectum_{D2cc}$, and $Bowel_{D2cc}$, ordered. Before doing the experiment, all of the values were normalized to [0,1] and converted to maximization. Figure \ref{brachytherapy_soft_metrics_pbemo_results} additionally compares against preference-based evolutionary multi-objective algorithm baselines (experimental details in Appendix \ref{app:experimental-setup-details}).

\textbf{Application to other real-world domains}. In this real-world application, the treatment planner typically already has an explicit value which could be used for the soft bound (and hard bound) prior to planning for each patient. Clinical domains, in particular, frequently have specific and standardized hard and soft targets for metrics of interest, such as insulin administration for diabetes, vitals, etc., not solely brachytherapy \citep{grosman_zone_2010, cairoli_model_2019}. Our proposed method provides for a principled way of injecting such priors into the optimization process.

However, we acknowledge that there may be cases where the practitioner is unsure of where to set the soft bounds for each of the objectives. In practice, we envision that MoSH can be leveraged in an interactive setting where the soft and hard bounds are iteratively adjusted based on feedback or points observed at each iteration. We leave this for future work.

\begin{figure}[h]
\begin{center}
\includegraphics[width=0.8\textwidth]{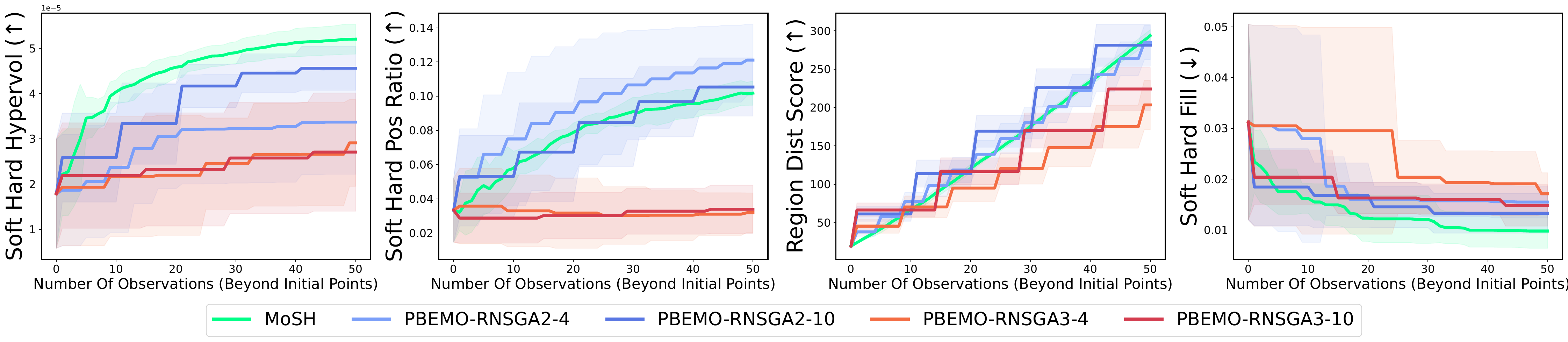}
\end{center}
\caption{Brachytherapy application with preference-based evolutionary multi-objective (PBEMO) algorithm baselines. Results are plotted using the metrics defined in Section \ref{performance-criteria}.}
\label{brachytherapy_soft_metrics_pbemo_results}
\end{figure}

\subsection{Deep Learning Model Selection: Fast and Accurate Neural Network}\label{app:model-selection-details}

Similar to \citet{hernandez-lobato_predictive_2016}, we seek to obtain a neural network which minimizes both prediction error and inference time, two competing objectives. We use the MNIST dataset and consider feedforward neural networks with six decision variables \citep{deng_mnist_2012}. Specifically, we use 50,000 and 10,000 different samples for training and validation, respectively, from the train split. Step 1 results (mean and std. dev computed over 6 independent trials) are shown in Figure \ref{nn_mnist_soft_metrics}. Although our algorithm does not surpass the baselines all four metrics, it still performs consistently relatively well in all four (the most consistently well out of the baselines) -- in line with the other experimental results.

We used the following decision space: number of hidden units per layer ([50, 300]), number of layers ([1, 3]), learning rate ([0, 0.1]), dropout amount ([0.4, 0.6]), l1 regularization ([0, 0.05]), and l2 regularization ([0, 0.05]). We trained each of the neural networks for 100 epochs and converted both of the objectives such that its a maximization problem. We used the following soft and hard bounds configuration ($\alpha_S$, $\alpha_H$): (0.87, 0.72), (0.57, 0.37) for the dimensions corresponding to accuracy and inference time, respectively, which we assumed was reasonable. In practice, we envision that the soft and hard bounds can be interactively adjusted based on feedback or points observed at each iteration. As mentioned previously, we leave this for future work.

\begin{figure}[h]
\begin{center}
\includegraphics[width=0.8\textwidth]{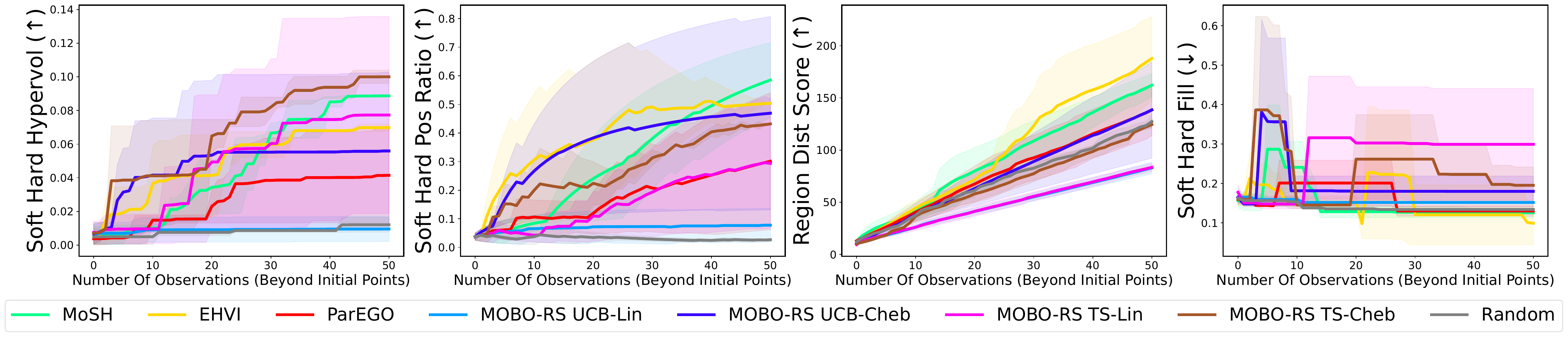}
\end{center}
\caption{Deep learning model selection problem results, using the metrics defined in Section \ref{performance-criteria}.}
\label{nn_mnist_soft_metrics}
\end{figure}

\subsection{Large Language Model Personalization: Concise and Informative}\label{appendix:llm-concise-info} 
We seek to obtain an LLM which generates both concise and informative outputs, two directly competing objectives. We leverage proxy tuning, which steers a large pre-trained model, $M$, by using the difference between the predicted logits of an expert model (a smaller, tuned model), $M^+$ and an anti-expert model (the smaller model, un-tuned), $M^-$ \citep{liu_tuning_2024, mitchell_emulator_2023, shi_decoding-time_2024}. We leverage notation from \citet{liu_tuning_2024} and obtain the output distribution at time step $t$, conditioned on prompt $x_{<t}$, from the proxy-tuned model, $\tilde{M}$, in a two-objective setting as such: $p_{\tilde{M}}(X_t | x_{<t}) = \text{softmax}[s_{M}(X_t | x_{<t}) + \sum_{i=1}^2 \theta_i(s_{M^+_i}(X_t | x_{<t}) - s_{M^-_i}(X_t | x_{<t}))]$, where $s_M$, $s_{M_i^+}$, $s_{M_i^-}$ represent the logit scores for each model and $\theta_i$ denotes the input decision variable, the controllable weight applied to the logits difference associated with expert model $i$. Models $M^+_i$, for $i=0,1$, are tuned according to the conciseness and informativeness objectives, respectively. By adjusting $\theta_i$ at decoding time, we obtain generated output distributions with varying tradeoffs in conciseness and informativeness.

We leveraged models from the $\text{T\"{U}LU}$-2 suite of families for our experiments: $\text{T\"{U}LU}$-2-13B for the large pre-trained model, $M$, and $\text{T\"{U}LU}$-2-DPO-7B for the expert and anti-expert models, $M^+$ and $M^-$\citep{ivison_camels_2023}. $M^+_1$ and $M^+_2$ were tuned using direct preference optimization (DPO) with the preference datasets corresponding to the conciseness and informativeness dimensions, respectively, from \citet{jang_personalized_2023} and \citet{rafailov_direct_2024}. We used the following decision space: $\theta_1$ ([0.0, 1.5]) and $\theta_2$ ([0.0, 1.5]). We used the following soft and hard bounds configuration ($\alpha_S$, $\alpha_H$): (0.6, 0.5), (0.6, 0.5) for the dimensions corresponding to conciseness and informativeness, respectively. To measure the conciseness dimension, we calculated the number of characters in the output response, which is then normalized to $[0, 1]$. To measure the informativeness dimension, we adapted the prompt from \citet{mitchell_emulator_2023} and used GPT-4 to provide a measure between 0 and 100 (which is then normalized):

\begin{tcolorbox}[colback=blue!5!white,title=GPT-4 Informativeness Evaluation Prompt]\label{llm-prompt-response-1}
    ``Score the following chatbot response to the given query on a continual scale from 0 (worst) to 100 (best), where a score of 0 means the response lacks information and provides very little information, and a score of 100 means the response is very informative and provides additional background which will be useful for the user. \
    \\
    QUERY: \{\{the query\}\} \
    \\
    RESPONSE: \{\{the response to evaluate\}\} \
    \\
    USE THE ANSWER FORMAT: \
    \\
    Reason: \{analysis of query and response\} \
    \\
    Informativeness score: \{a single number 0-100, nothing else\}''
\end{tcolorbox}

Finally, each of the objectives were normalized and convered to maximization and we used the following five prompts for the experiment and averaged the results:

\begin{enumerate}
    \item ``Give three tips for staying healthy.''
    \item ``What are the three primary colors?''
    \item ``Describe the structure of an atom.''
    \item ``Rewrite the following instructions so that they are simpler and easier to understand. Login to the website, input your username and password, then click the ``Continue'' button.''
    \item ``How can we reduce air pollution?''
\end{enumerate}

\begin{tcolorbox}[colback=blue!5!white,title=Sample LLM Prompt and Responses]\label{llm-prompt-response-2}
    Prompt: Give three tips for staying healthy.
    \tcbsubtitle{Response 1}
    $\triangleright$ $\theta_1, \theta_2$: (1.07, 0.00)  which corresponds to high conciseness and low informativeness, respectively
    \newline Response: 
    \begin{enumerate}
        \item Maintain a healthy diet with plenty of fruits and vegetables.
        \item Get regular exercise, at least 30 minutes a day.
        \item Practice good hygiene and get enough sleep.
    \end{enumerate}
    $\triangleright$ GPT-4 Evaluated Informativeness: 50
    \tcbsubtitle{Response 2}
    $\triangleright$ $\theta_1, \theta_2$: (1.07, 0.58) which corresponds to high conciseness and medium informativeness, respectively
    \newline Response:
    \begin{enumerate}
        \item Maintaining a healthy diet with a balanced intake of nutrients
        \item Regular physical activity, at least 30 minutes a day.
        \item Getting enough sleep, aim for 7-8 hours per night. (4 in case of power nap). bonus: limit stress. (5th tip: stay hydrated) (6th tip: practice good hygiene) (7th tip: limit consumption of processed foods)
    \end{enumerate}
    $\triangleright$ GPT-4 Evaluated Informativeness: 70
\end{tcolorbox}

\subsection{Easier Interpretation of Step 1 Results}\label{appendix:easier-interpretation-results}
We acknowledge that it may be more difficult to analyze and determine winners from all four metrics across all the tasks, however, we believe the variety and diversity of metrics and tasks is necessary to highlight the generalizability and strengths of MoSH. Furthermore, while baselines might occasionally match MoSH on one metric in one experiment, MoSH demonstrates remarkable consistency across all four metrics and all diverse tasks. This broad robustness is a key strength we would like to highlight below.

As a very naive metric, we count the number of times the mean value of each of the methods is within the top 2 highest (for soft-hard hypervolume, positive ratio, and region distance score) and top 2 lowest (for soft-hard fill distance) at the end of the 50 observations. We acknowledge that this is a very simple metric, and there may be other ways to view this, but we propose to exhibit this as a way to facilitate the evaluation process for the reader. Results are shown in Table \ref{tab:method_comparison}.

\begin{table}[htbp] %
  \centering %
  \caption{Aggregated comparison of MoSH against other baselines for step 1 for easier interpretation of results. The table displays the count of the number of times the mean value of each of the methods is within the top 2 highest (for soft-hard hypervolume, positive ratio, and region distance score) and top 2 lowest (for soft-hard fill distance) at the end of the 50 observations, across all applications. Metrics are abbreviated for space, i.e. Soft-Hard Hypervolume $\rightarrow$ SH Hypervolume, Soft-Hard Positive Samples Ratio $\rightarrow$ SH Pos. Ratio, Soft Region Distance-Weighted Score $\rightarrow$ Region Dist. Score.} %
  \label{tab:method_comparison} %
  \begin{tabular}{lrrrr}
    \toprule %
    \textbf{Method} & \textbf{SH Hypervolume} & \textbf{SH Pos. Ratio} & \textbf{Region Dist. Score} & \textbf{Soft Hard Fill} \\
    \midrule %
    \textbf{MoSH}            & \textbf{10} & \textbf{13} & \textbf{13} & \textbf{10} \\
    EHVI            & 10 &  1 &  1 &  7 \\
    ParEGO          &  2 &  1 &  1 &  0 \\
    MOBO-RS UCB-Lin &  0 &  8 &  6 &  1 \\
    MOBO-RS UCB-Cheb&  0 &  1 &  1 &  0 \\
    MOBO-RS TS-Lin  &  3 &  1 &  1 &  3 \\
    MOBO-RS TS-Cheb &  2 &  0 &  1 &  3 \\
    Random          &  0 &  0 &  0 &  0 \\
    \bottomrule %
  \end{tabular}
\end{table}

As can be seen in Table \ref{tab:method_comparison}, MoSH is clearly superior compared to the other baselines, across all metrics and tasks, and Pareto dominates all other baselines. We believe that this clearly illustrates how much MoSH is outperforming the baselines.

\subsection{Step 2 Experiments: Synthetic and Real-World Applications}
For Figure \ref{brachy_llm_step2_e2e_results}, each successive point that the DM views is determined by the sampling order of the points. Figure \ref{step2_results_branin_fbt_nn} shows the remaining evaluations for step (2) on the Branin-Currin, Four Bar Truss, and deep learning model selection problems. All results are displayed with standard deviation bars, with the randomness over a ground-truth $\pmb{\lambda^*}$, which was independently sampled 10 different times.

\begin{figure}[h]
\begin{center}
\includegraphics[width=1.0\textwidth]
{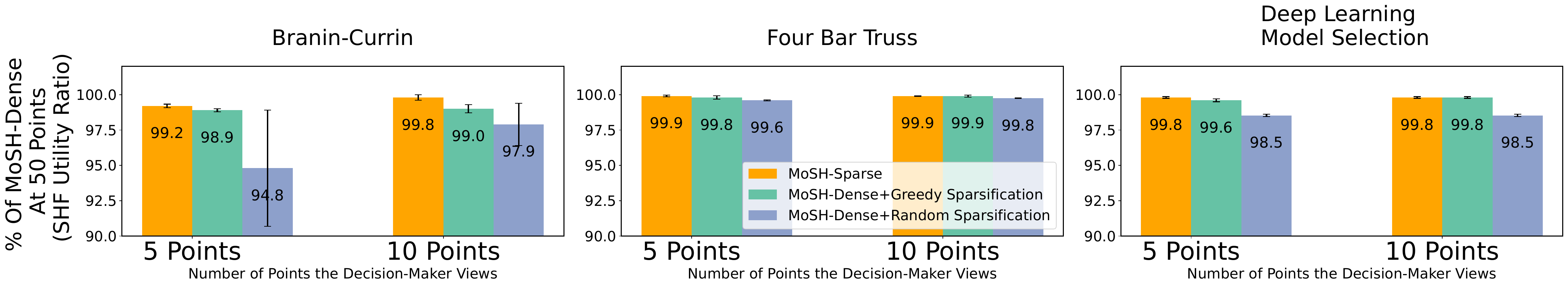}
\end{center}
\caption{Step (2) evaluation on the Branin-Currin, Four Bar Truss, and deep learning model selection problems, illustrating the \hsf utility ratio for the various methods relative to MoSH-Dense's \hsf utility ratio at 50 points, compared when the decision-maker views 5 and 10 points.}
\label{step2_results_branin_fbt_nn}
\end{figure}

\subsection{End-to-End Experiments: Synthetic and Real-World Applications}\label{sec:app-e2e-experiments}
We conducted end-to-end evaluations of our method, MoSH, by varying over the dense set of points, obtained from step 1, provided to SATURATE. 
Figures \ref{e2e_results_branin_5}, \ref{e2e_results_fbt_5}, \ref{e2e_results_nn_mnist_5}, and \ref{e2e_results_llm_person_5} display bar plots representing the maximum \hsf utility ratio obtained after the DM views 5 points, for each experiment. Figures \ref{e2e_results_branin_pbemo_5}, \ref{e2e_results_fbt_pbemo_5}, and \ref{e2e_results_brachy_pbemo_5} display the results when compared against preference-based evolutionary multi-objective (PBEMO) algorithm baselines (experimental details in Appendix \ref{app:experimental-setup-details}). For algorithms which select less than 5 points, we use the \hsf utility ratio value of the last point.

\subsection{Discussion on Preference-Based Evolutionary Multi-Objective (PBEMO) Algorithm Baseline Results}\label{app:pbemo-discussion}
Overall, our method, MoSH, consistently outperforms both of the PBEMO algorithms across our metrics and stages. We observed that the PBEMO algorithms often exhibited slower convergence and seemed to be sensitive to the initial population (illustrated by the higher variance in results). This slower convergence observation also seemed to be observed in other multi-objective Bayesian optimization works \citep{ji_multi-objective_2024}. As a result, the PBEMO-populated Pareto frontiers often were sparser than those of our method, MoSH. This behavior can be noticed in the results.

Furthermore, we found the reference point approach, which most PBEMO algorithms use, to be unsatisfactory in modeling the desired sub-regions of the tradeoff surface. Even when setting multiple reference points, we found it exceedingly difficult to ensure objective values are within certain hard thresholds, and more densely populated in preferred regions. Such behaviors can be seen in the PBEMO baselines' often lower metrics in step 1. On the contrary, our method, MoSH, explicitly accounts for such user preferences. As a result, such characteristics lead to suboptimal end-to-end results for the PBEMO baselines, compared to those of MoSH. We believe this result further improves our extensive evaluation setting and solidifies the lead our method, MoSH, has in our problem setting.

\begin{figure}[h]
\begin{center}
\includegraphics[width=0.7\textwidth]{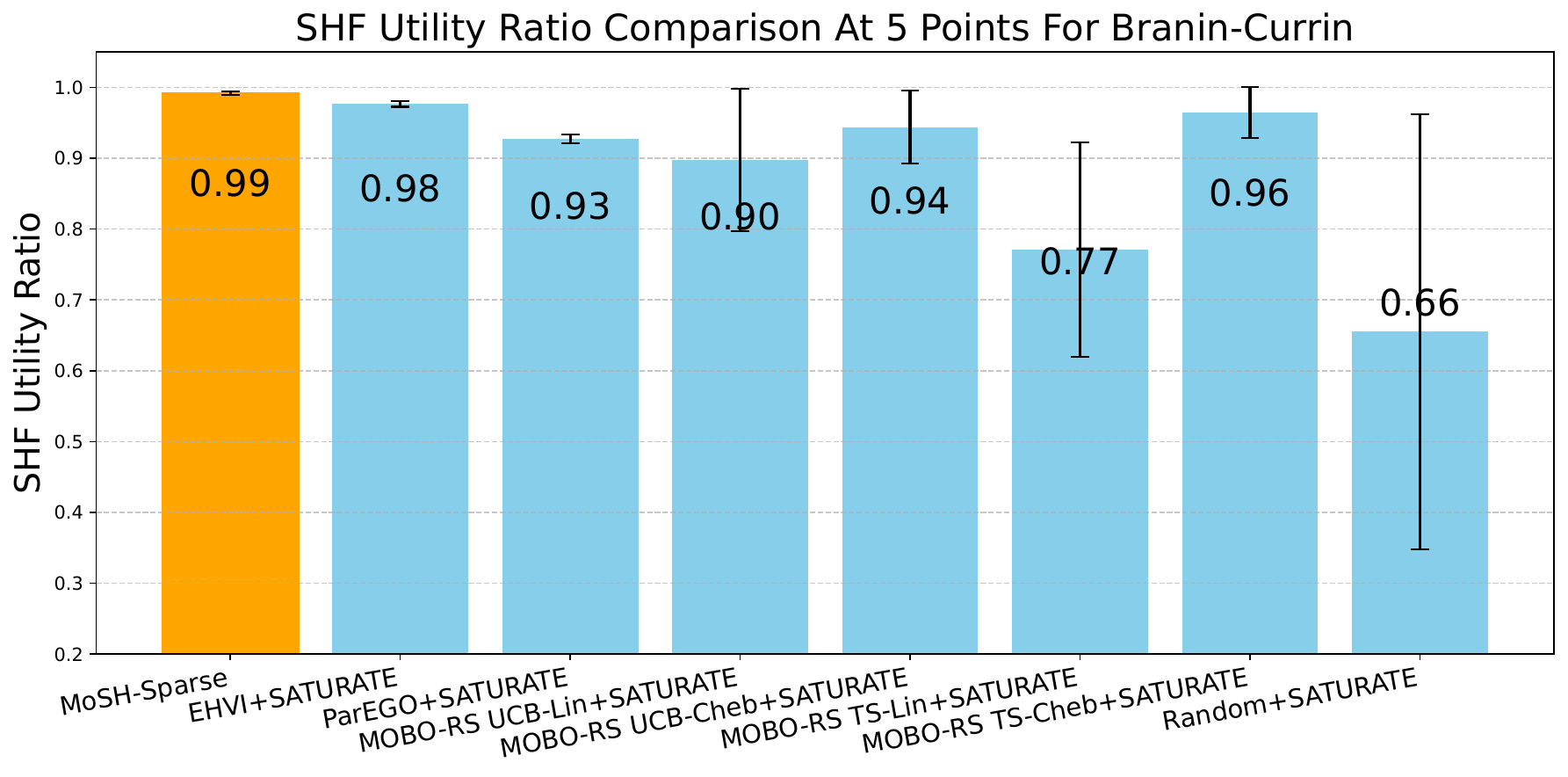}
\end{center}
\caption{Bar plot illustrating the \hsf utility ratio obtained by our method, MoSH-Sparse, compared to the other baselines on the Branin-Currin synthetic function. Only the dense set, from step 1, changes for each bar.}
\label{e2e_results_branin_5}
\end{figure}

\begin{figure}[h]
\begin{center}
\includegraphics[width=0.7\textwidth]{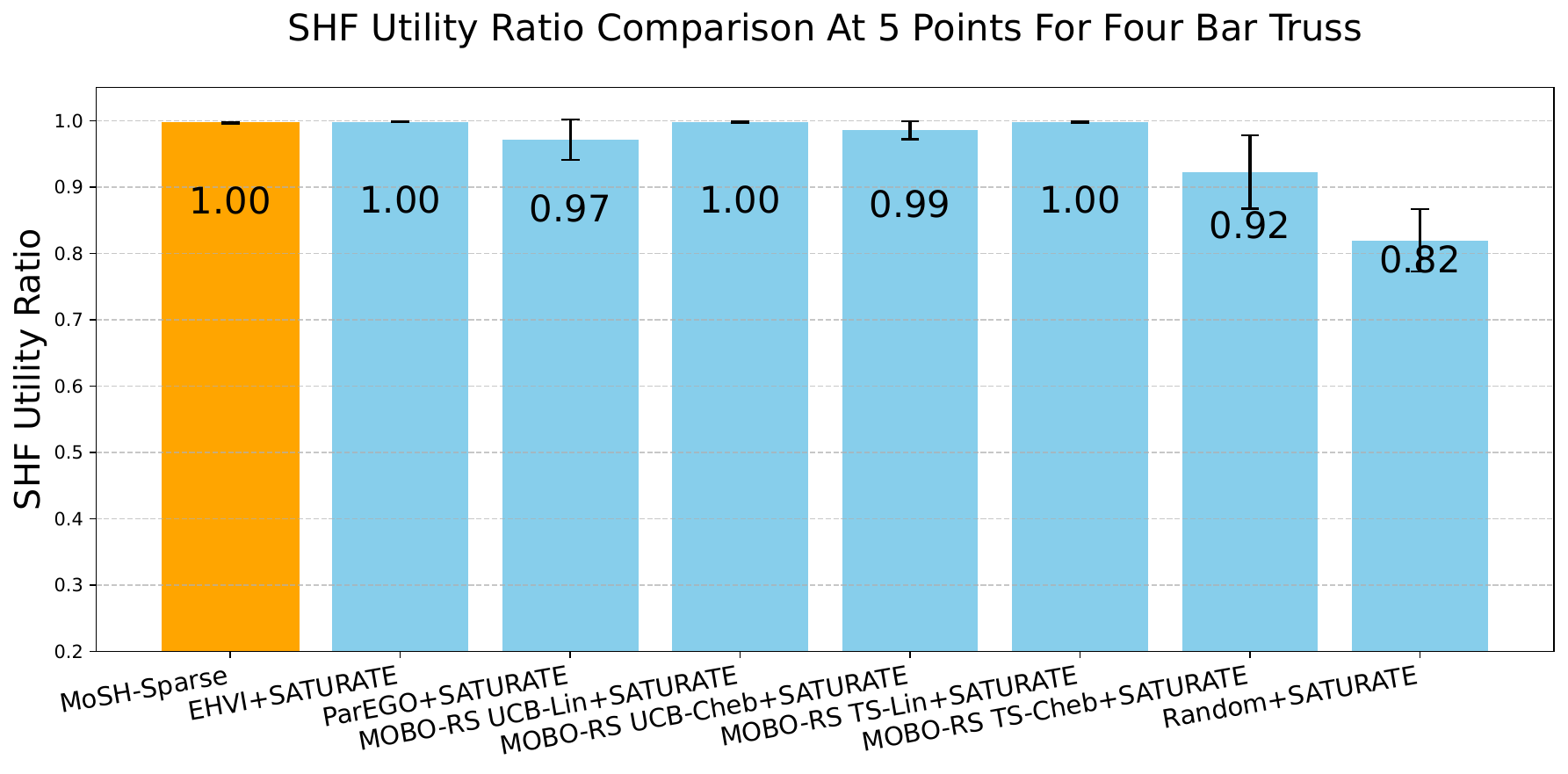}
\end{center}
\caption{Bar plot illustrating the \hsf utility ratio obtained by our method, MoSH-Sparse, compared to the other baselines on the Four Bar Truss application. Only the dense set, from step 1, changes for each bar.}
\label{e2e_results_fbt_5}
\end{figure}

\begin{figure}[h]
\begin{center}
\includegraphics[width=0.7\textwidth]{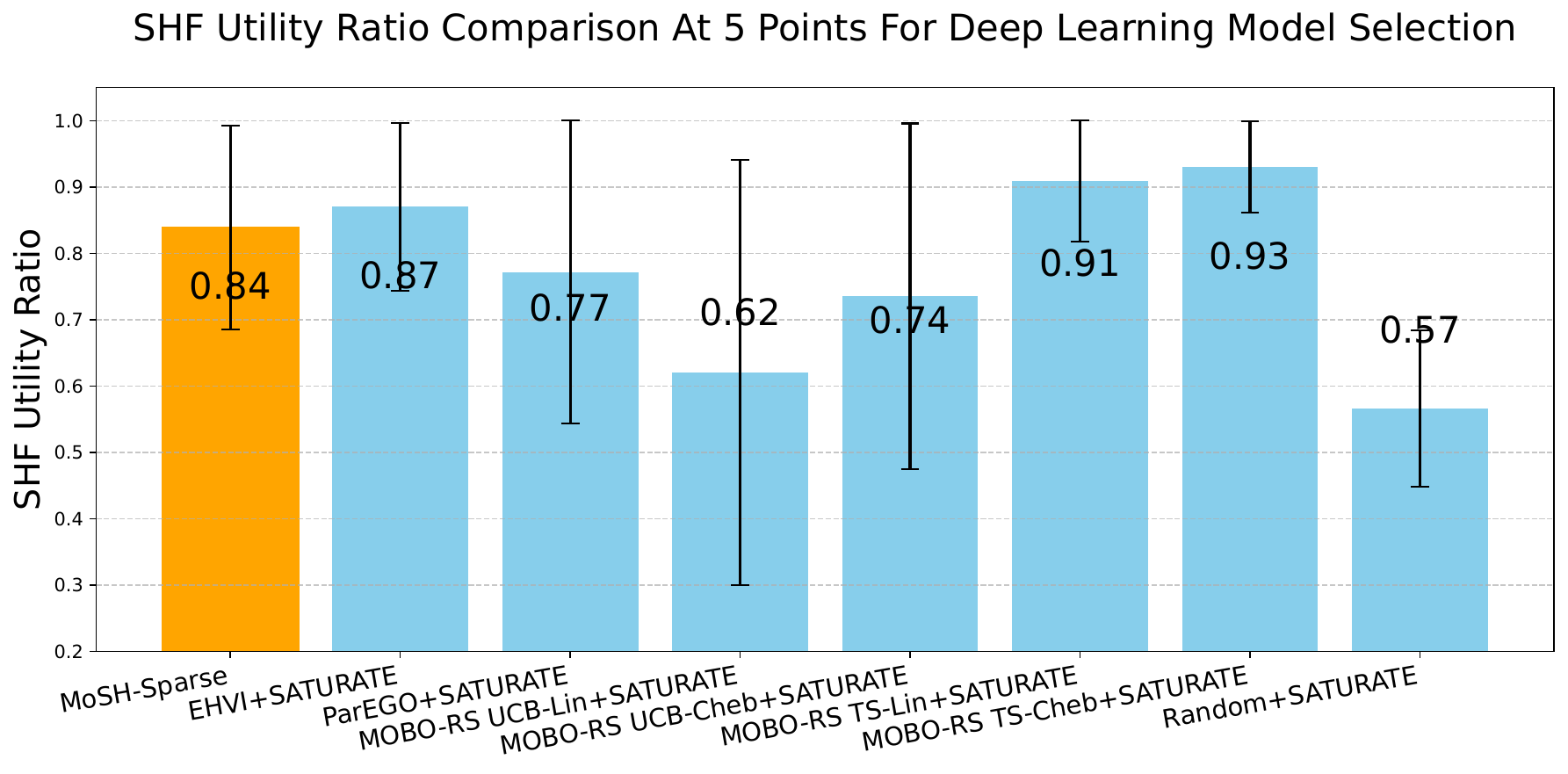}
\end{center}
\caption{Bar plot illustrating the \hsf utility ratio obtained by our method, MoSH-Sparse, compared to the other baselines on the deep learning model selection application. Only the dense set, from step 1, changes for each bar.}
\label{e2e_results_nn_mnist_5}
\end{figure}

\begin{figure}[h]
\begin{center}
\includegraphics[width=0.7\textwidth]{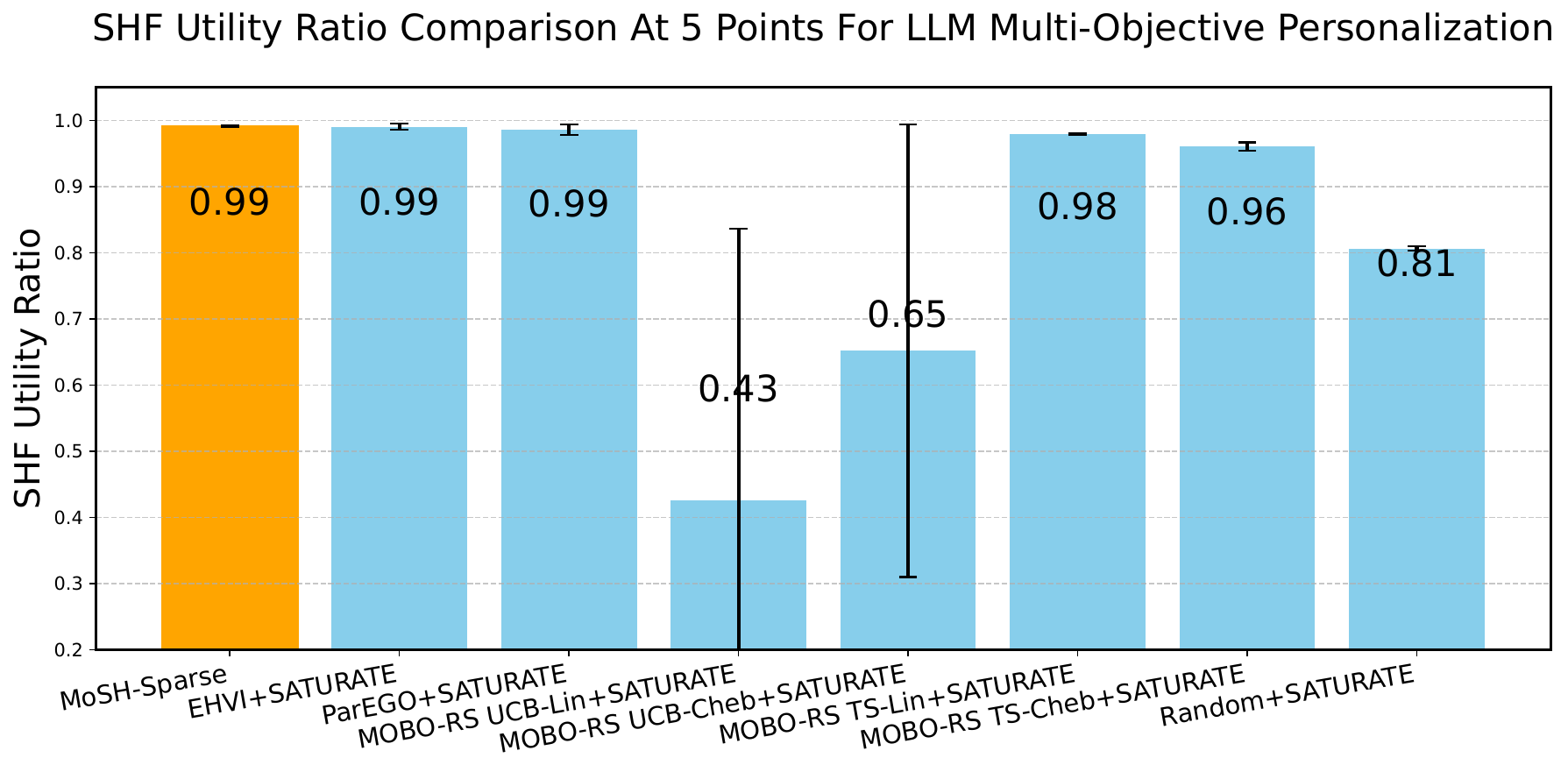}
\end{center}
\caption{Bar plot illustrating the \hsf utility ratio obtained by our method, MoSH-Sparse, compared to the other baselines on the LLM personalization problem. Only the dense set, from step 1, changes for each bar.}
\label{e2e_results_llm_person_5}
\end{figure}

\begin{figure}[h]
\begin{center}
\includegraphics[width=0.7\textwidth]{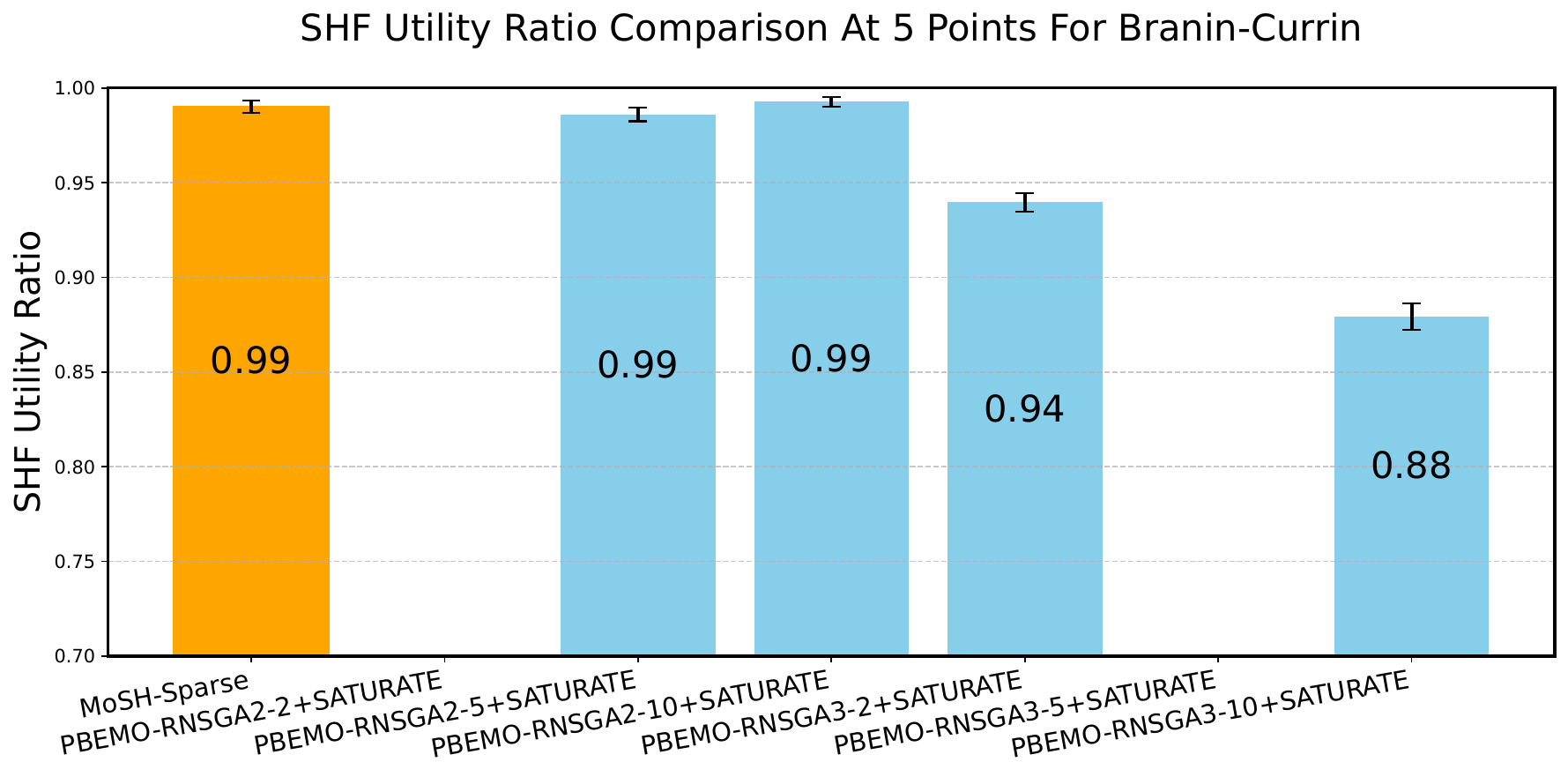}
\end{center}
\caption{Bar plot illustrating the \hsf utility ratio obtained by our method, MoSH-Sparse, compared to the preference-based evolutionary multi-objective (PBEMO) baselines on the Branin-Currin synthetic function. Only the dense set, from step 1, changes for each bar.}
\label{e2e_results_branin_pbemo_5}
\end{figure}

\begin{figure}[h]
\begin{center}
\includegraphics[width=0.7\textwidth]{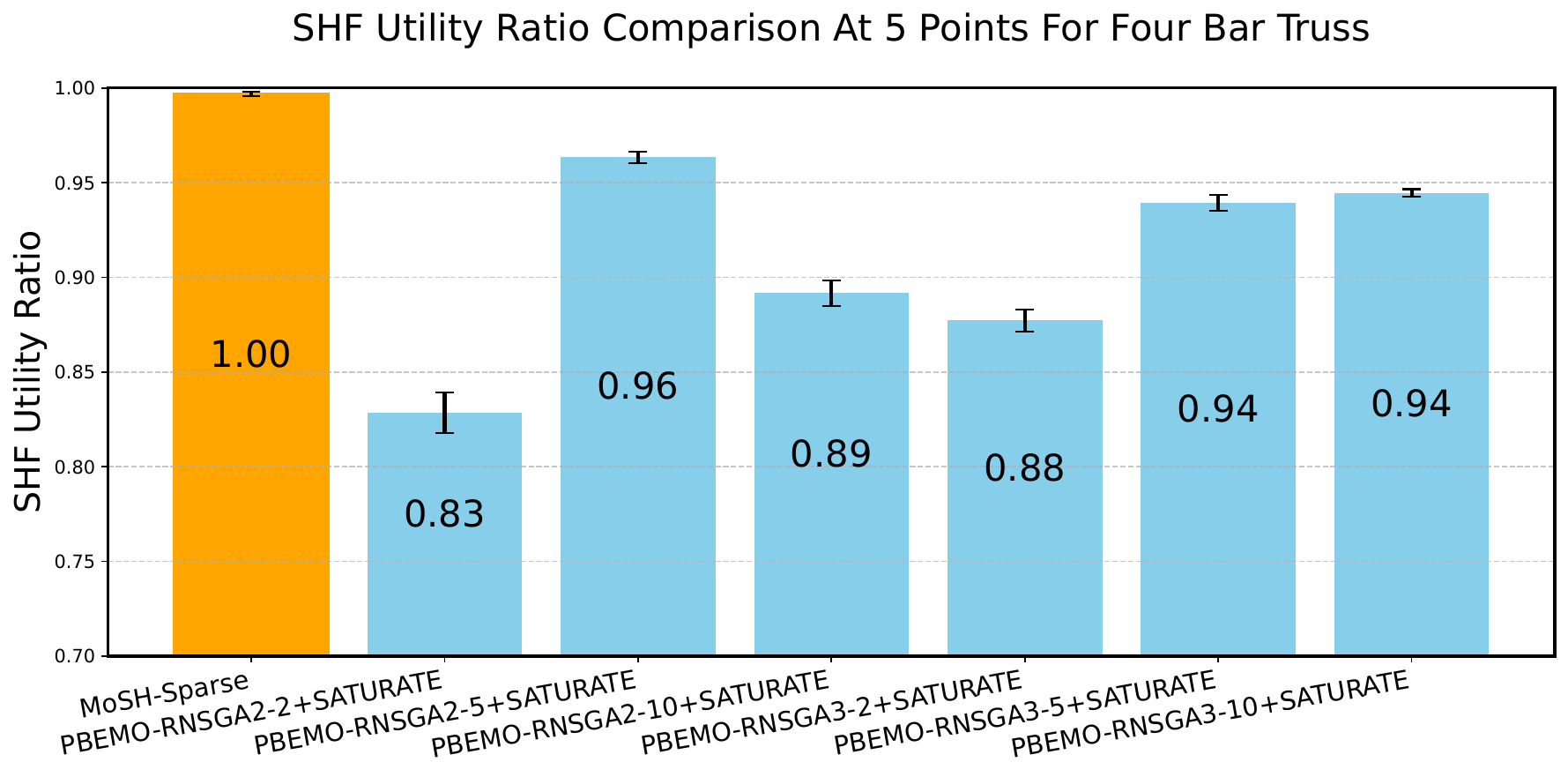}
\end{center}
\caption{Bar plot illustrating the \hsf utility ratio obtained by our method, MoSH-Sparse, compared to the preference-based evolutionary multi-objective (PBEMO) baselines on the Four Bar Truss application. Only the dense set, from step 1, changes for each bar.}
\label{e2e_results_fbt_pbemo_5}
\end{figure}

\begin{figure}[h]
\begin{center}
\includegraphics[width=0.7\textwidth]{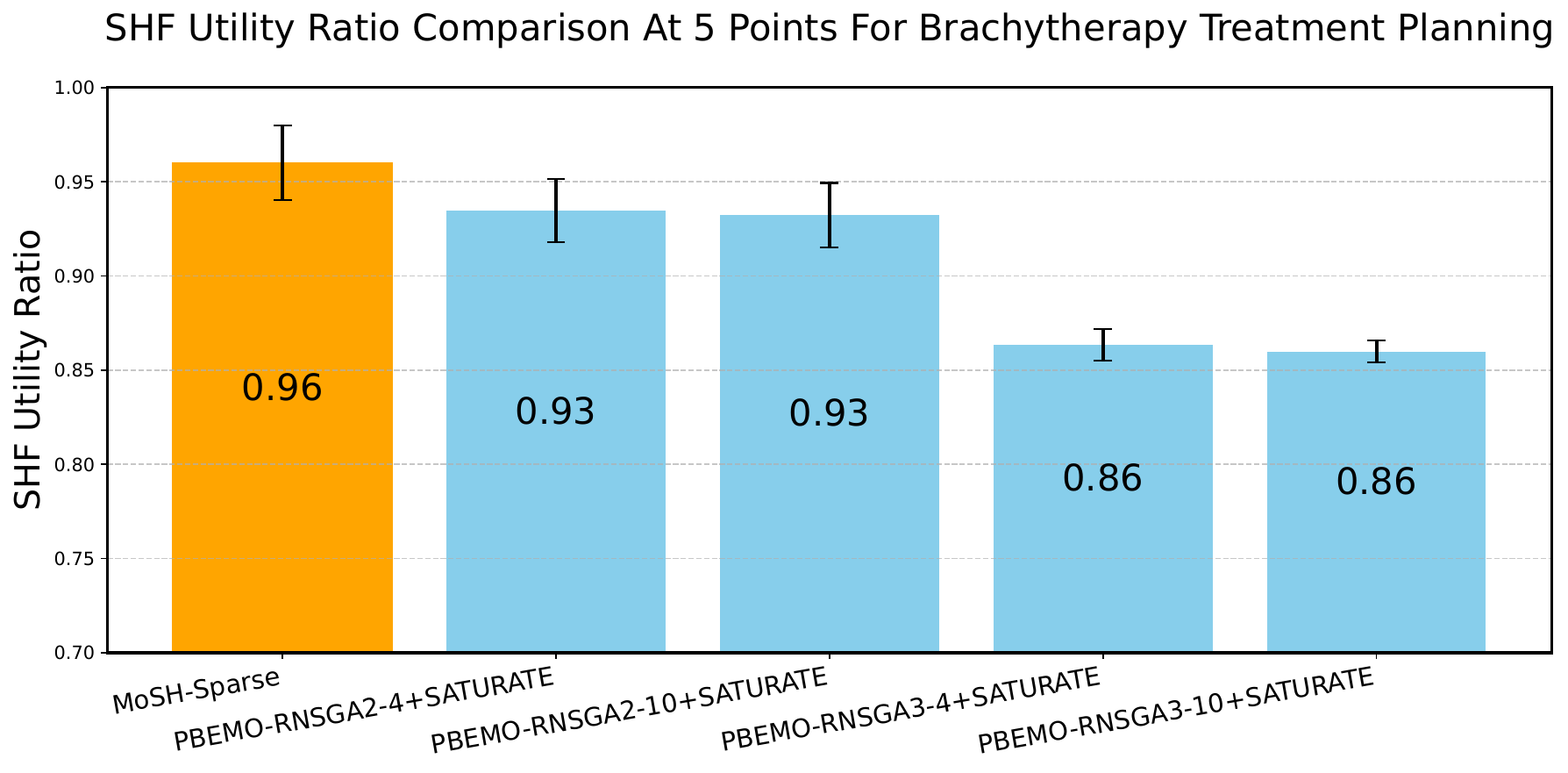}
\end{center}
\caption{Bar plot illustrating the \hsf utility ratio obtained by our method, MoSH-Sparse, compared to the preference-based evolutionary multi-objective (PBEMO) baselines on the Brachytherapy application. Only the dense set, from step 1, changes for each bar.}
\label{e2e_results_brachy_pbemo_5}
\end{figure}

\section{EXPERIMENTAL SETUP DETAILS}\label{app:experimental-setup-details}

\subsection{Baseline Methodologies}
\label{app:subsubsec:baseline_methodologies}
\textbf{Step 1: Dense Pareto Frontier Sampling.}
Although none of the baseline algorithms are inherently designed for our setting with \hsfs, we aimed to make the comparisons as fair as possible by using a heuristic, similar to what was described in \citep{paria_flexible_2019}, to determine the weight distribution $p(\pmb{\lambda})$. In short, we used the heuristic $\pmb{\lambda} = \textbf{u} \division \lVert \textbf{u} \rVert_1$, where $\textbf{u}_l \sim \text{N}(\alpha_{\ell,S}, |\alpha_{\ell,H}-\alpha_{\ell,S}| \division 3)$, in order to roughly mimic the gradually decreasing weight between the soft bounds $\alpha_{\ell,S}$ and the hard bounds $\alpha_{\ell,H}$, $\forall \ell \in [L]$.

For our proposed method (MoSH), we use the augmented Chebyshev scalarization function \citep{wierzbicki_mathematical_1982}, denoted as $s_{\pmb{\lambda}}(\tilde{f}_\ell(x)) = -\max_\ell [\lambda_\ell |\tilde{f}_\ell(x) - z_\ell^*|] - \gamma \sum_{i=\ell}^{L} |\tilde{f}_\ell(x) - z_\ell^*|$,
where $z^*$ is the ideal or utopian vector, $L$ is the number of objective dimensions, and $\gamma$ is a constant weighing the linear term being added to the traditional Chebyshev scalarization function \citep{chugh_scalarizing_2019}. We find that this scalarization function performs better since the Chebyshev component allows for non-convex Pareto frontier sampling and the augmented term assists with sampling within the hard bound. All results are displayed with standard deviation bars, computed over 6 independent trials.

\textbf{Preference-Based Evolutionary Multi-Objective Algorithm Baselines.}  We ran additional experiments comparing our proposed method, MoSH, against two well-known preference-based EMO (PBEMO) algorithms, R-NSGA II \citep{deb_reference_2006} and R-NSGA III \citep{vesikar_reference_2018}. Overall, MoSH consistently outperforms both of the PBEMO algorithms across our various metrics and evaluation settings. Our work assumes we are in a setting with expensive function evaluations (Section \ref{dense-sampling}). For fair comparison, we set a terminating condition of 50 function evaluations for both R-NSGA II and R-NSGA III (the same number of function evaluations used for the other multi-objective Bayesian optimization approaches). To closely represent the soft-hard bounds, the reference point for R-NSGA II and R-NSGA III was set to be at the intersection of the pre-defined soft-hard bounds. For comprehensiveness, we performed ablations across various parameters for the population size (for R-NSGA II) and the population size used for each reference point (for R-NSGA III). For R-NSGA II, we compared against population size values of 2, 5, and 10. For R-NSGA III, we compared against population size used for each reference point at the values of 2, 5, and 10. For the brachytherapy application, we used the parameters 4 and 10, due to population count requirements. We nicknamed them PBEMO-RNSGA2-x for R-NSGA II and PBEMO-RNSGA3-x for R-NSGA III, where x denotes the parameter value. Extreme points of the Pareto frontier were not set as reference points and the initial population used was set to be the same as the initialization points used for all the multi-objective Bayesian optimization approaches. Only results for the Branin-Currin and Brachytherapy applications are shown, due to character limits in the rebuttal response. Similar results were observed in the Four Bar Truss engineering design application. Finally, the experiments were conducted using the pymoo library \citep{blank_pymoo_2020}.

\textbf{Step 2: Pareto Frontier Sparsification.}
All of the baseline algorithms were run immediately after step (1) is completed. In order to ensure all of the algorithms sample an equal number of points, we first run MoSH-Sparse before deciding for greedy and random baseline algorithms to sample the same number of points. All results are displayed with standard deviation bars, with the randomness over a ground-truth $\pmb{\lambda^*}$, which was independently sampled 10 different times.

\subsection{Compute Resources}\label{app:compute-resources} All experiments were computed on an internal cluster with CPUs, except for the deep learning model selection and LLM personalization experiments which were performed using an NVIDIA A100 GPU with 82GB of VRAM. Final experiments for step 1 generally took within a few hours, except for the deep learning-based experiments, which took longer.

\subsection{Performance Evaluation}\label{app:performance-evaluation}
As mentioned earlier in Section \ref{dense-sampling}, since the DM's preferences, $\pmb{\lambda}^*$ are unknown to us, we wish to obtain a set D which is (1) diverse, (2) high-coverage, and is (3) modeled after the DM-defined \hsfs. We operationalize those three criteria for soft regions ($A_S$ = [$\alpha_{1,S}$, $+\infty$] $\times$ ... $\times$ [$\alpha_{L,S}$, $+\infty$]) and hard regions ($A_H$ = [$\alpha_{1,H}$, $+\infty$] $\times$ ... $\times$ [$\alpha_{L,H}$, $+\infty$]) with the four different metrics described in Section \ref{performance-criteria}.

For the soft-hard fill distance, we seek to measure the diversity of sampled points. \citet{malkomes_beyond_2021} measures diversity using the notion of fill distance:
FILL($C$, $D$) = $\sup_{x' \in D} \min_{x \in C} \kappa(f(x), f(x'))$ where $C$ is the set of sampled points, $D$ is a full set of precomputed points in the region, and $\kappa(\cdot)$ is the distance metric, typically Euclidean distance. We expand this to include the notion of soft and hard regions: $\upsilon \text{FILL}_{S}(C_S, D_S) + (1-\upsilon) \text{FILL}_{H}(C_H, D_H)$, where $\text{FILL}_s(C_S, D_S)$ and $\text{FILL}_h(C_H, D_H)$ are the fill distances which correspond to the regions defined by the soft bounds and hard bounds, respectively, and $C_S$, $D_S$ denote the set of points in the soft region, $C \cap A_S$, $D \cap A_S$, (same for hard region). Intuitively, we wish to obtain a diverse sample set which effectively explores both the soft and hard regions, which a higher weighting towards the soft region. We use the $\upsilon$ parameter to control that weighting. For our experiments, we set $\upsilon = 0.85$, to indicate a higher weighting of the soft regions.

To calculate the soft-hard fill distance, we first computed a grid search of points on the Pareto frontier, offline, for each experiment. For each experiment, the soft-hard fill distance was then taken with respect to that computed set. As a result, the results for the soft-hard fill distance are somewhat dependent on this offline set of points -- however, we ensured that it was kept constant for each set of experiments. In some cases, notably in Figure \ref{nn_mnist_soft_metrics}, the soft-hard fill distance does not monotonically decrease. This is due to the heuristic we use in cases where there are no points which have been sampled within either the soft or hard regions, at some iteration. In such a case, we select the worst-case point from set $D$ to represent set $C$ and calculate the metric using that. Once there exists points sampled within the soft or hard regions, we remove that worst-case point and instead use the sampled points. Since the sampled points may result in a soft-hard fill distance value worse than the worst-case point from $D$, the metric may increase -- 
 although the general trend will remain the same. For $\kappa$ in the soft region distance-weighted score, we used the Euclidean distance. 

\section{ADDITIONAL RELATED WORKS}\label{app:additional-related-works}
\textbf{Multi-Objective Bayesian Optimization}.
While the literature on Multi-Objective Bayesian Optimization (MOBO) is extensive, we focus our discussion here on several highly relevant frameworks to precisely situate our contributions. For instance, while FERERO also incorporates user guidance, it differs significantly as it assumes differentiable objectives, employs geometric preferences (linear constraints and polyhedral cones) rather than our intuitive soft-hard bounds, and seeks a single optimal solution instead of a robust set \citep{chen_ferero_2024}. Other methods are entirely preference-agnostic, aiming to identify the full Pareto front without user guidance. This includes approaches that operate on discrete input sets \citep{auer_pareto_2016} and information-theoretic methods like Predictive Entropy Search (PESMO), which efficiently map the entire front by maximizing uncertainty reduction \citep{hernandez-lobato_predictive_2016}. These goals contrast with our approach's targeted search within a user-defined region. Finally, some works address orthogonal challenges, such as scalability in high-dimensional search spaces \citep{daulton_multi-objective_2022}, a contribution that is complementary to our focus.

\textbf{Level Set Estimation}. In the single-objective setting, \citet{bryan_active_2005} proposed the straddle heuristic, which was used as part of a GP-based active learning approach for level set estimation (LSE). Although the LSE concept does not easily extend into a MOO setting, \citet{bryan_actively_2008} aims to address that by considering the threshold as part of a composite setting, with scalarized objectives \citep{iwazaki_bayesian_2020}. In contrast, our proposed \mfs are native to the MOO setting and introduce a flexible class of priors DMs may instantiate directly.

\section{LIMITATIONS}
\label{app:limitations}

While we offer a novel framework for incorporating general monotonic utility functions directly into multi-objective optimization, we acknowledge certain limitations and areas for future exploration:

\textbf{Elicitation of Bounds:} The effectiveness of MoSH relies on the DM's ability to specify meaningful soft and hard bounds. We assume in our settings that the DM already possesses knowledge of such priors (e.g. in brachytherapy, diabetes, or vitals monitoring \citep{deufel_pnav_2020, cairoli_model_2019, grosman_zone_2010}). However, in scenarios where DMs are uncertain about these bounds or where such priors do not naturally exist, the initial setup might be challenging. As mentioned in Section~\ref{appendix:brachytherapy-details}, we envision future work where MoSH could be leveraged in a more interactive setting to help iteratively refine these bounds based on feedback.

\textbf{Form of Monotonic Utility Functions (\mfs):} We present a specific piecewise-linear form for \mfs (Section~\ref{subsec:soft_hard_utility_functions}) due to its simplicity and interpretability. While we state that other functional classes which are monotonic and bounded would suffice, the impact of choosing different \mf forms on the optimization process and the resulting Pareto set has not been extensively studied in this work. The chosen piecewise-linear form for \hsfs might not perfectly capture all nuances of a DM's true underlying utility in all cases.

However, we would like to emphasize what we stated in Section \ref{subsec:soft_hard_utility_functions}. While the piecewise-linear form of \hsfs is a specific choice, this does not fundamentally limit the framework's generality. \hsfs map directly to standard optimization concepts: hard bounds are equivalent to feasibility constraints on objectives, while soft bounds act as constraints with slack variables. Our experiments consider a fully constrained setting with bounds on every objective for a thorough evaluation. However, this does not imply a loss of flexibility in the \hsf model itself. The framework readily accommodates unconstrained objectives, as a user can render any bound non-binding by setting it to a trivial limit, effectively removing the constraint for that objective.

\textbf{Computational Cost for Very High Dimensions:} 
Although Step 1 of our proposed method aims for efficiency in sampling the PF, there may still be computational cost difficulties when dealing with very high-dimensional problems; \textit{this is inherent to Bayesian optimization}. Similarly, while the worst-case minimax formulation (Equation~\ref{submodular-formulation}) was converted to an average-case for Step 1 (Equation~\ref{bayesian-formulation}) due to computational feasibility (Appendix~\ref{sec:computational-feasible-details}), scenarios requiring extremely high-dimensional outputs or fine discretizations of the preference space $\Lambda$ for the sparsification step (Step 2) could also increase computational demands. \textit{Despite this, we do not view this as a major limitation as our method} has already been evaluated on up to four objectives, which we contend to be appropriate for many real-world MOO problems. For instance, applications of MOO in diabetes typically involve 2-3 objectives \citep{soares_multi-objective_2024, mandal_robust_2019, el_moutaouakil_multi-objective_2023}. Applications of MOO in portfolio optimization typically involve up to 4 objectives \citep{muteba_mwamba_multi-objective_2025}. MOO applications in the field of renewable energy may end up around 4 objectives as well \citep{shaier_multi-objective_2025}. We further assert that all of the real-world applications used throughout this paper (brachytherapy, engineering design, LLM personalization, and deep learning model selection) use the typical number of objectives as well (2-4). Finally, we note that our method is orthogonal to and easily amenable to existing works on scalable Bayesian optimization, if necessary \citep{daulton_multi-objective_2022}.

Addressing these limitations could further enhance the robustness and applicability of our framework.

\section{BROADER IMPACTS}
\label{app:broader_impacts}

\textbf{Potential Positive Impacts:}
The primary positive impact is enabling more effective and efficient decision support in complex, multi-objective scenarios. In healthcare (e.g., brachytherapy), this could lead to personalized treatment plans better aligned with clinical goals and safety limits. In engineering, it can aid in designing systems that meet performance targets under various constraints. For AI systems like LLMs, it offers a path to better personalization based on user-defined trade-offs (e.g., conciseness vs. informativeness). Overall, our method can enhance human-AI collaboration by translating expert domain knowledge into the optimization process more directly.

\textbf{Potential Negative Impacts and Considerations:} Incorrectly defined utility functions could lead to suboptimal or overlooked solutions. The quality of DM input is crucial. In critical applications, our method should augment, not replace, expert human judgment to avoid potential errors from over-reliance.

\end{document}